%% file: main.tex
\documentclass{article}
\input{macros}

\PassOptionsToPackage{numbers}{natbib}

\usepackage[final]{neurips_2021}

\usepackage[utf8]{inputenc} % allow utf-8 input
\usepackage[T1]{fontenc}    % use 8-bit T1 fonts
\usepackage{hyperref}       % hyperlinks
\usepackage{url}            % simple URL typesetting
\usepackage{booktabs}       % professional-quality tables
\usepackage{amsfonts}       % blackboard math symbols
\usepackage{nicefrac}       % compact symbols for 1/2, etc.
\usepackage{microtype}      % microtypography
\usepackage{xcolor}         % colors
\usepackage{subcaption}

\title{On Blame Attribution for Accountable Multi-Agent Sequential Decision Making}

\author{%
    Stelios Triantafyllou\\
	MPI-SWS\\
	\texttt{strianta@mpi-sws.org} \\
	\And
    Adish Singla\\
	MPI-SWS\\
	\texttt{adishs@mpi-sws.org} \\
	\And
	Goran Radanovic\\
	MPI-SWS\\
	\texttt{gradanovic@mpi-sws.org} \\
	\\
}

\begin{document}

\maketitle

%%%%%%%%%%%%%%%%%%%%%%%%%%%%%%%%%%%%%
%%%%%%%%%%%%%%%%%%%%%%%%%%%%%%%%%%%%%
\newtoggle{longversion}
\settoggle{longversion}{true}
% \settoggle{longversion}{false}
\newtoggle{submission}
% \settoggle{submission}{true}
\settoggle{submission}{false}
%%%%%%%%%%%%%%%%%%%%%%%%%%%%%%%%%%%%% ABSTRACT
\input{0_abstract}

%%%%%%%%%%%%%%%%%%%%%%%%%%%%%%%%%%%%% INTRODUCTION
\input{1_introduction}

%%%%%%%%%%%%%%%%%%%%%%%%%%%%%%%%%%%%% RELATED WORK
\input{1.1_related_work}

%%%%%%%%%%%%%%%%%%%%%%%%%%%%%%%%%%%%% PROBLEM FORMULATION
\input{2_preliminaries}

%%%%%%%%%%%%%%%%%%%%%%%%%%%%%%%%%%%%% GAME-THEORETIC APPROACHES TO BLAME ASSIGNMENT
\input{3_full_information_setting}

%%%%%%%%%%%%%%%%%%%%%%%%%%%%%%%%%%%%% UNCERTAINTY IN EXECUTION BEHAVIOR
\input{4_uncertainty_in_execution_behavior}

%%%%%%%%%%%%%%%%%%%%%%%%%%%%%%%%%%%%% EXPERIMENTS
\input{5_experiments}

%%%%%%%%%%%%%%%%%%%%%%%%%%%%%%%%%%%%% CONCLUSION
\input{6_conclusion}

%%%%%%%%%%%%%%%%%%%%%%%%%%%%%%%%%%%%% ACKNOWLEDGMENTS
% \input{8_ackonwledgements}

%%%%%%%%%%%%%%%%%%%%%%%%%%%%%%%%%%%%%
%\newpage
%\bibliographystyle{abbrv}
\bibliographystyle{unsrt}
\bibliography{main}

%%%%%%%%%%%%%%%%%%%%%%%%%%%%%%%%%%%%% APPENDIX
\iftoggle{longversion}{
\clearpage
\onecolumn
\appendix 
{\allowdisplaybreaks
\input{9.0_appendix_main}

}
}
{}
%%%%%%%%%%%%%%%%%%%%%%%%%%%%%%%%%%%%%
%%%%%%%%%%%%%%%%%%%%%%%%%%%%%%%%%%%%%
\end{document}

%%%%%%%%%%%%%%%%%%%%%%%%%%%%%%%%%%%%%
%%%%%%%%%%%%%%%%%%%%%%%%%%%%%%%%%%%%%

%% file: macros.tex
\usepackage{hyperref}

\usepackage{savesym,multirow,url,xspace,graphicx,hyperref,enumitem,color}
\usepackage[titletoc]{appendix}
\usepackage{subcaption}
\usepackage{dsfont}
\usepackage{multicol}

\usepackage{etoolbox}

\usepackage{tikz}
\usetikzlibrary{automata, positioning, arrows}
\tikzset{
	->, % makes the edges directed
	every state/.style={thick, fill=gray!10}, % sets the properties for each ’state’ node
	initial text=$ $, % sets the text that appears on the start arrow
}
\usetikzlibrary{arrows.meta}
\usepackage{subcaption}
\usepackage{pgfplots}
\pgfplotsset{width=10cm,compat=1.9}
\usepackage{diagbox}
\usepackage{caption}

\makeatletter
\newif\if@restonecol
\makeatother

\usepackage[lined, boxed, ruled, commentsnumbered]{algorithm2e}
\usepackage{amsmath,amssymb,xfrac,amsthm}
\usepackage{mathtools}
\setitemize{noitemsep,topsep=1pt,parsep=1pt,partopsep=1pt, leftmargin=12pt}
\newtheorem{lemma}{Lemma}
\newtheorem{theorem}{Theorem}
\newtheorem*{theorem*}{Theorem}

\newtheorem{corollary}{Corollary}
\newtheorem{proposition}{Proposition}

\newtheorem{observation}{Observation}
\newtheorem{claim}{Claim}
\makeatletter

\newcommand{\Rmnum}[1]{\expandafter\@slowromancap\romannumeral #1@}
\makeatother
\usepackage{caption}

\usepackage{enumitem}
\usepackage{wasysym}

\DeclareMathOperator*{\argmin}{arg\,min}
\DeclareMathOperator*{\argmax}{arg\,max}

\usepackage{bm}
\usepackage{bbm}

\newcommand{\norm}[1]{\left\lVert#1\right\rVert}

\newcommand{\expct}[1]{\mathbb{E}\left[#1\right]}

\newcommand{\ind}[1]{\mathds{1}\left[#1\right]}

\newcommand{\return}{{J}}

\newcommand{\optpi}{\pi^{*}}

\newcommand{\optpir}{\pi^{*|\pi}}
\newcommand{\optpirb}{\pi^{*|\bpi}}

\newcommand{\polspace}{\Pi}

\newcommand{\inef}{{\Delta}}

\newcommand{\blamefunc}{{\Psi}}

\newcommand{\blame}{{\beta}}
\newcommand{\blamei}{{\beta_i}}
\newcommand{\blamej}{{\beta_j}}
\newcommand{\blamek}{{\beta_k}}
\newcommand{\weight}{{w}}
\newcommand{\bpi}{{\pi^{b}}}
\newcommand{\bpie}{{\widehat{\pi}^{b}}}

\newcommand{\prop}{{\mathcal R}}

\newcommand{\agentone}{{$ A_1$}}
\newcommand{\agenttwo}{{$ A_2$}}

%% file: 0_abstract.tex
% !TEX root =  main.tex
%%%%%%%%%%%%%%%%%%%%%%%%%%%%%%%%%%%%%
%%%%%%%%%%%%%%%%%%%%%%%%%%%%%%%%%%%%%
\begin{abstract}
Blame attribution is one of the key aspects of accountable decision making, as it provides means to quantify the responsibility of an agent for a decision making outcome. In this paper, we study blame attribution in the context of cooperative multi-agent sequential decision making. As a particular setting of interest, we focus on cooperative decision making formalized by Multi-Agent Markov Decision Processes (MMDPs), and we analyze different blame attribution methods derived from or inspired by existing concepts in cooperative game theory. We formalize desirable properties of blame attribution in the setting of interest, and we analyze the relationship between these properties and the studied blame attribution methods. Interestingly, we show that some of the well known blame attribution methods, such as Shapley value, are not performance-incentivizing, while others, such as Banzhaf index, may over-blame agents. To mitigate these value misalignment and fairness issues, we introduce a novel blame attribution method, unique in the set of properties it satisfies, which trade-offs explanatory power (by under-blaming agents) for the aforementioned properties. 
We further show how to account for uncertainty about agents' decision making policies, and we experimentally: a) validate the qualitative properties of the studied blame attribution methods, and b) analyze their robustness to uncertainty. 
\end{abstract}

%% file: 1_introduction.tex
% !TEX root =  main.tex
%%%%%%%%%%%%%%%%%%%%%%%%%%%%%%%%%%%%%
%%%%%%%%%%%%%%%%%%%%%%%%%%%%%%%%%%%%%

{\em ... a body of people\footnote{Originally, and by modern standards outdated, Thomas Paine used phrasing with the word {\em men}.}, holding themselves accountable to nobody, ought not to be trusted by anybody.}

\hfill {\em ---Thomas Paine, A philosopher and a political activist.}

\section{Introduction}

With the widespread usage of artificial intelligence (AI) in everyday life~\cite{propublica_story, khandani2010consumer, esteva2017dermatologist}, accountability has become one of the central problems in the study of AI. Much recent research studied what constitutes accountability in the context of AI and how to design accountable AI systems~\cite{doshi2017accountability, kroll2016accountable,  wieringa2020account}, and recent policies and legislations~\cite{ethicsEU} are increasingly highlighting the importance of accountability, aiming to provide guidelines for developing and deploying accountable AI systems.

Accountability is a relatively broad term, and it typically involves an actor (or multiple actors) justifying their decisions  and facing consequences for actions taken \cite{wieringa2020account,bovens2007analysing}.
Hence, two critical aspects of accountability are explainability and blame attribution. Recent work proposed various methods for explaining, interpreting, understanding, and certifying algorithmic decision-making and its outcomes \cite{ribeiro2016model, datta2016algorithmic, lundberg2017unified, rudin2019stop, doshi2017towards, dann2019policy}. In this paper we study the other critical aspect of accountability -- \emph{blame attribution}.

In multi-agent decision making, one of the central roles of blame attribution is assigning blame for undesirable outcomes or, broadly speaking, for the system's inefficiency. Prior work on responsibility and blame in AI \cite{chockler2004responsibility,halpern2016actual,halpern2018towards,friedenberg2019blameworthiness} has recognized some of the core challenges in attributing blame, including the fact that disentangling agents' contributions to the final outcome is not a trivial task. Such challenges are particularly prominent 
in sequential settings where past decisions influence the future ones~\cite{halpern2016actual}. 

In this paper, we consider the task of allocating a score to an agent, which represents the degree of its blame, and reflects its contributions to the total inefficiency of the multi-agent system. We focus on cooperative sequential decision making, formalized by multi-agent Markov decision processes (MMDPs)~\cite{boutilier1996planning}, where the outcome of interest is the expected discounted return of the agents' joint policy. Concretely, given an MMDP and the agents' joint policy (true or estimated), we ask: {\em How to score each agent so that the agents' scores satisfy desirable properties? }

To answer this question, we turn to cooperative game theory and consider blame attribution methods that are derived from or inspired by existing concepts in the cost sharing, data valuation, and coalition formation literature~\cite{von2007theory, jain2007cost, balcan2015learning, balkanski2017statistical, jia2019towards, agarwal2019marketplace, shoham2008multiagent, chalkiadakis2004bayesian}, such as core~\cite{gillies1959solutions}, Shapley value~\cite{shapley201617, shapley1954method}, or Banzhaf index~\cite{banzhaf1964weighted, banzhaf1968one}. Taking this perspective on blame attribution, we study blame attribution for accountable multi-agent sequential decision making. More concretely: 
\begin{itemize}
    \item We formalize desirable properties that blame attribution methods should satisfy in cooperative multi-agent sequential decision making.
    We identify properties that are typically not considered in the cost-sharing literature, yet are important for decision making. In particular, we introduce two novel properties:
    a) {\em performance monotonicity}, which states that, having fixed all the other agents to their policies, the blame assigned to an agent should not increase if the agent adopts a policy that results in a higher expected discounted return (implying that the method is performance-incentivizing); b) {\em Blackstone consistency},\footnote{This property is inspired by Blackstone's ratio: ``It is better that ten guilty persons escape than that one innocent suffer"~\cite{blackstone1893commentaries}.} which states that an agent should not receive a higher blame just because the agents' policies are not fully known to the blame attribution procedure.
    \item We characterize the properties of the studied blame attribution methods. We show that some blame assignment methods, such as, Shapley value, are not performance-monotonic (and, hence, performance-incentivizing), while
    others, such as Banzhaf index, may over-blame agents. Motivated by these results, we introduce a novel blame attribution method that trade-offs explanatory power (by under-blaming agents) for the aforementioned properties.
    \item We provide algorithms for making the studied blame attribution methods Blackstone-consistent when the agents' policies are estimated. We also characterize the effect of uncertainty on blame attribution methods. 
    \item Using a simulation-based testbed, we experimentally analyze the studied blame attribution methods, their qualitative properties, as well as their robustness to uncertainty. The experiments showcase the importance of the robustness considerations we study and indicate that typically more efficient blame attribution methods (i.e., those that assign more blame in total) are less robust to uncertainty. 
\end{itemize}

%% file: 1.1_related_work.tex
% !TEX root =  main.tex
%%%%%%%%%%%%%%%%%%%%%%%%%%%%%%%%%%%%%
%%%%%%%%%%%%%%%%%%%%%%%%%%%%%%%%%%%%%
\subsection{Other Related Work}

Apart from the works mentioned in the previous paragraphs, our work relates to different areas of moral philosophy, law, and AI, and here we highlight some of the most relevant references. Research in moral philosophy and law has extensively studied the problem of blame attribution, both in terms of human actors \cite{scanlon2009moral, shoemaker2011attributability, van2015moral}, as well as AI actors \cite{coeckelbergh2020artificial, torrance2008ethics, asaro2007robots, lima2021human}. We take some of the well known principles in moral philosophy and law in determining properties relevant for blame attribution, e.g., Blackstone consistency is inspired by Blackstone's ratio \cite{blackstone1893commentaries}. 
In AI, blame attribution has been studied through a more formal lens, utilizing causality \cite{chockler2004responsibility,halpern2016actual,halpern2018towards} and/or game theory \cite{friedenberg2019blameworthiness,ijcai2021-244}, and primarily focusing on nuances related to defining notions and degrees of responsibility, blame, and blameworthiness.
In contrast, we focus on cooperative sequential decision making, and analyze how different blame attribution methods from cooperative game theory fare under different blame attribution properties. Finally, our work is generally related to the {\em credit assignment} problem \cite{minsky1961steps, sutton2018reinforcement}, and more specifically to the credit assignment problem in multi-agent reinforcement learning~\cite{tumer2007distributed,foerster2018counterfactual,wang2020shapley}. However, our focus is not on supporting the learning processes of agents by reducing computational and statistical challenges of learning, but on evaluating the agents' contributions to the system's inefficiency, ideally in a fair and interpretable manner.

%% file: 2_preliminaries.tex
% !TEX root =  main.tex
%%%%%%%%%%%%%%%%%%%%%%%%%%%%%%%%%%%%%
%%%%%%%%%%%%%%%%%%%%%%%%%%%%%%%%%%%%%
\section{Formal Setting}\label{sec.setting}

\vspace{-0.1cm}
In this section, we describe our formal setting, based on multi-agent Markov decision processes (MMDPs), and we formally model the blame attribution problem in sequential decision making. This section also introduces a set of desirable formal properties of blame attribution methods.

\vspace{-0.1cm}
\subsection{Preliminaries}\label{sec.preliminaries}

\vspace{-0.1cm}
We consider a cooperative multi-agent setting, formalized as a class of MMDPs $\mathcal M$ with $n$ agents $\{1, ..., n\}$. Each MMDP in this class is a tuple $M = (\mathcal{S}, \{1, ..., n\}, \mathcal{A}, R, P, \gamma, \sigma)$ \cite{boutilier1996planning}, where: $\mathcal{S}$ is the state space; $\mathcal{A} = \times_{i = 1}^n \mathcal{A}_i$ is the action space, with $\mathcal{A}_i$ being the action space of agent $i$; $R$ is the reward function $R: \mathcal{S} \times \mathcal{A} \rightarrow \mathds R$ specifying the reward obtained when agents $\{1,..., n\}$ take a joint action; $P$ specifies transitions with $P(s, a, s')$ denoting the probability of transitioning to $s'$ from $s$ when agents $\{1,..., n\}$ take joint action $a = (a_1,..., a_n)$; $\gamma$ is the discount factor; and $\sigma$ is the initial state distribution. $\mathcal{S}$ and $\mathcal{A}$ are finite and discrete.
A (stationary) joint policy $\pi$ is a mapping 
$\pi: \mathcal{S} \rightarrow \mathcal D(\mathcal{A})$, where $\mathcal D(\mathcal{A})$ is a probability simplex over $\mathcal{A}$, 
with $\pi(a|s)$ denoting the probability of taking joint action $a$ in $s$. We assume that a joint policy $\pi$ is factorizable into agents' policies, $\pi_i$, i.e., $\pi(a|s) = \pi_1(a_1|s) \cdots \pi_n(a_n|s)$. Therefore, we can define an agent $i$'s policy  $\pi_i$ as a mapping from states to a distribution of agent $i$'s actions, i.e., $\pi_i: \mathcal{S} \rightarrow \mathcal D(\mathcal{A}_i)$. We denote the set of all policies by $\Pi = \times_{i=1}^n \Pi_i$.
We also define a standard performance measure. The expected discounted return of a joint policy $\pi$ is defined as $\return(\pi) = \expct{\sum_{t = 1}^{\infty} \gamma^{t-1} R(s_t, a_t)| s_1 \sim \sigma, \pi}$,
where the initial state $s_1$ is sampled from $\sigma$, and the state-joint action pair of time-step $t$, $(s_t, a_t)$, is obtained by executing joint policy $\pi$. We abuse our notation by denoting $\return(\pi_i',\pi_{-i}) = \return(\pi_1, ..., \pi_i',..., \pi_n)$. Similarly, $\return(\pi_S',\pi_{-S}) = \return(\pi'')$  for some $S \subseteq \{1, ..., n\}$, where $\pi_i'' = \pi_i$ if $i \notin S$ and $\pi_i'' = \pi_i'$ if $i \in S$.

\vspace{-0.1cm}
\subsection{Blame Attribution}\label{sec.blame_assignment}

\vspace{-0.1cm}
Our goal is to assign blame to agents for failing to jointly achieve optimal performance. Given the agents' behavior policy, denoted by $\bpi$, the inefficiency of the considered multi-agent system can be defined as $\inef = \return(\optpi) - \return(\bpi)$,
where $\optpi \in \argmax_{\pi} \return(\pi)$ is an optimal joint policy. Similarly, we define the marginal inefficiency of a subset of agents $S \subseteq \{1, ..., n \}$ as $\inef_S = \return(\optpirb_S,\bpi_{-S}) - \return(\bpi)$, as well as the marginal inefficiency of an agent $i$ as $\inef_i = \return(\optpirb_i,\bpi_{-i}) - \return(\bpi)$, where $\optpirb_S$ (resp. $\optpirb_i$) denotes an optimal policy of $S$ (resp. $i$) assuming all other policies are fixed, i.e., $\optpirb_S \in \argmax_{\pi_S} \return(\pi_S,\bpi_{-S})$ (resp.  $\optpirb_i \in \argmax_{\pi_i} \return(\pi_i,\bpi_{-i})$).
A blame attribution method is a mapping $\blamefunc: \mathcal M \times \polspace \rightarrow \mathds R_{\ge 0}^n$, where $\blamefunc(M, \bpi)$ distributes blame for inefficiency $\inef$ by assigning score $\blamei$ to agent $i$. The output of $\blamefunc$, i.e., the blame assignment, is denoted by $\blame$.

\textbf{Uncertainty considerations.} Since the agents' behavior policy $\bpi$ might not be known to the blame attribution procedure, we also define blame attribution under uncertainty as a mapping $\widehat{\blamefunc}: \mathcal M \times \mathcal P(\polspace) \rightarrow \mathds R^n_{\ge 0}$ that outputs a blame assignment estimate $\widehat \blame$. Here, $\mathcal P(\polspace)$ represents a set whose elements express the knowledge about $\bpi$. Inspired by the literature on robust MDPs \cite{iyengar2005robust,nilim2005robust,tamar2014scaling}, we encode such knowledge with uncertainty sets $\mathcal P(\bpi)$, one associated to each state $s$, $\mathcal P(\bpi, s)$, defined by the set of probability measures on $\mathcal{A}$. 
We assume that $\mathcal P(\bpi)$ is consistent with $\bpi$, i.e., $\bpi(\cdot|s)$ is in $\mathcal P(\bpi, s)$,\footnote{Such $\mathcal P(\bpi)$ can be derived from data containing agents' trajectories and be based on confidence intervals.}\footnote{$\bpi(\cdot|s)$ could be in $\mathcal P(\bpi, s)$ w.h.p., provided Blackstone consistency in Section \ref{sec.properties} is similarly adjusted. }
and that every $\pi(\cdot|s)$ in $\mathcal P(\bpi, s)$ factorizes to $\pi(a|s) = \pi_1(a_1|s) \cdots \pi_n(a_n|s)$. Therefore, $\mathcal P(\bpi, s)$ identifies the set of plausible stochastic actions that agent $i$ takes in state $s$.
 
\vspace{-0.1cm}
\subsection{Desirable Properties}\label{sec.properties}

\vspace{-0.1cm}
Our goal is to specify functions $\blamefunc$ and $\widehat{\blamefunc}$, such that the blame assignments $\blame$ and $\widehat{\blame}$ satisfy desirable properties. In the following text we denote these properties by $\prop$. 
Below, we define properties that are taken from or inspired by the game theory  literature~\cite{shoham2008multiagent,brandt2016handbook,jain2007cost}, but translated to our setting\footnote{Note that the terminology is slightly different.}:
\begin{itemize}
    \item {\em Validity} ($\prop_V$): We say that $\blamefunc$ is valid if it never distributes more blame than the observed inefficiency $\inef$.  More formally, $\blamefunc$ satisfies $\prop_V$ (resp. $\epsilon$-$\prop_V$) if for every $M$ and $\bpi$, $\sum_{i = 1}^n \blamei \le \inef$ (resp. $\sum_{i = 1}^n \blamei \le \inef + \epsilon$), where $\blame = \blamefunc(M, \bpi)$. 
    \item {\em Efficiency} ($\prop_E$): A more strict condition is that the total distributed blame is equal to $\inef$. That is,  $\blamefunc$ satisfies $\prop_E$ (resp. $\epsilon$-$\prop_E$) if for every $M$ and $\bpi$, $\sum_{i = 1}^n \blamei = \inef$ (resp. $|\sum_{i = 1}^n \blamei - \inef| \le \epsilon$), where $\blame = \blamefunc(M, \bpi)$.
    \item {\em Rationality} ($\prop_R$): Similar to validity is rationality, which requires that blame distributed to any subset of agents $S$ is not greater than $\inef_S$. That is,  $\blamefunc$ satisfies $\prop_R$ (resp. $\epsilon$-$\prop_R$) if for every $M$, $\bpi$, and $S \subseteq \{1, ..., n\}$, $\sum_{i\in S} \blamei \le \inef_S$ (resp. $\sum_{i\in S} \blamei \le \inef_S + \epsilon$), where $\blame = \blamefunc(M, \bpi)$.
    \item {\em Symmetry} ($\prop_S$):  We say that $\blamefunc$ is symmetric if it treats equal agents equally, i.e., agents that equally contribute to the inefficiency should receive the same blame. More formally, $\blamefunc$ satisfies $\prop_S$ (resp. $\epsilon$-$\prop_S$) if for every $M$ and $\bpi$, $\blame_i = \blame_j$ (resp. $|\blame_i - \blame_j| \le \epsilon$) whenever $\inef_{S \cup \{i \}} = \inef_{S \cup \{j \}}$ for all $S \subseteq \{1, ..., n\} \backslash \{i, j\}$, where $\blame = \blamefunc(M, \bpi)$.
    \item {\em Invariance} ($\prop_I$): We say that $\blamefunc$ is invariant if it assigns zero blame to agents who do not marginally contribute to inefficiency. More formally, $\blamefunc$ satisfies $\prop_I$ (resp. $\epsilon$-$\prop_I$) if for every $M$ and $\bpi$, $\blame_i = 0$ (resp. $\blame_i \le \epsilon$) whenever $\inef_{S \cup \{i \}} = \inef_{S}$ for all $S$, where $\blame = \blamefunc(M, \bpi)$.
\end{itemize}
Note that $\epsilon > 0$ in the definitions of  $\epsilon$-$\prop$, and that we use these properties in our characterization result for blame attribution under uncertainty.  
Additionally, we consider two properties that relate the blame attribution output to the MMDP structure and the agents' behavior policies.
\begin{itemize}
    \item {\em Contribution monotonicity} ($\prop_{CM}$)\cite{young1985monotonic}: We say that $\blamefunc$ is contribution-monotonic if the blame it assigns to an agent depends only on its marginal contributions and monotonically so. More formally, $\blamefunc$ satisfies $\prop_{CM}$ (resp. $\epsilon$-$\prop_{CM}$) if for every two 
    $(M^1, \bpi^1)$ and $(M^2, \bpi^2)$, $\blamei^1 \geq \blamei^2$ (resp. $\blamei^1 \geq \blamei^2 - \epsilon$) whenever $\inef^1_{S \cup \{i \}} - \inef^1_S \geq \inef^2_{S \cup \{i \}} - \inef^2_S$ for all $S$, where $\blame^1 = \blamefunc(M^1, \bpi^1)$ and $\blame^2 = \blamefunc(M^2, \bpi^2)$.
    \item {\em Performance monotonicity} ($\prop_{PerM}$): We say that $\blamefunc$ is performance-monotonic if it does not assign greater blame to agent $i$ for adopting a policy that results in an equal or higher performance, assuming the other agents' policies fixed. More formally, consider any MMDP $M$, and any $\bpi_{-i}$, $\pi_i$ and  $\pi_i'$ such that $\return(\pi_i,\bpi_{-i}) \le \return(\pi_i',\bpi_{-i})$. We say that $\blamefunc$ satisfies $\prop_{PerM}$ (resp. $\epsilon$-$\prop_{PerM}$) if $\blame_i \ge \blame_i'$ (resp. $\blame_i \ge \blame_i' - \epsilon$) where $\blame = \blamefunc(M, (\pi_i,\bpi_{-i}))$ and $\blame' = \blamefunc(M, (\pi_i', \bpi_{-i}))$.
\end{itemize}
The above definitions directly extend to $\widehat{\blamefunc}$ except that we require them to hold for all $\mathcal{P}(\pi_b)$. 
Additionally, we identify the following property for $\widehat{\blamefunc}$:
\begin{itemize}
\item{\em Blackstone consistency} ($\prop_{BC}$): We say that $\widehat{\blamefunc}$ is Blackstone-consistent with $\blamefunc$ if it never attributes more blame to an agent than $\blamefunc$. More formally, $\widehat{\blamefunc}$ satisfies $\prop_{BC}(\blamefunc)$ if for any $M$, $\bpi$ and $\mathcal{P}(\bpi)$, $\widehat{\blamei} \le \blamei$, where $\blame = \blamefunc(M, \bpi)$ and $\widehat{\blame} = \widehat{\blamefunc}(M, \mathcal{P}(\bpi))$.
\end{itemize}

%% file: 3_full_information_setting.tex
% !TEX root =  main.tex
%%%%%%%%%%%%%%%%%%%%%%%%%%%%%%%%%%%%%
%%%%%%%%%%%%%%%%%%%%%%%%%%%%%%%%%%%%%
\section{Game-Theoretic Approaches to Blame Attribution}\label{sec.approaches_to_blame}
In this section, we study blame attribution methods based on well known game theoretic notions, such as the core~\cite{gillies1959solutions}, Shapley value~\cite{shapley201617, shapley1954method}, or Banzhaf index~\cite{banzhaf1964weighted, banzhaf1968one}. We also introduce a novel blame attribution method, unique in the set of properties it satisfies. The proofs of our results can be found in \iftoggle{longversion}{Appendices \ref{sec.prop_proofs_123}, \ref{sec.proof_sv}, and \ref{sec.proof_ap}}{the supplementary material}.

\subsection{Max-Efficient Rationality}\label{sec.efficient_rationality}

We start with a relatively simple blame assignment method that puts rationality as a strict condition, and maximizes the efficiency of blame assignment under this constraint. We call this blame assignment method {\em max-efficient rationality}. More formally, max-efficient rationality can be defined
via the following linear program:
\begin{align}
	\label{prob.rationality}
	\tag{P1}
	&\quad \blamefunc_{MER}(M, \bpi) := \max_{\blame} \sum_{i=1}^n \blamei
	\quad \quad \mbox{ s.t. }  \quad  \sum_{i\in S} \blamei \le \inef_S \quad \forall S \subseteq \{1, ..., n\},
\end{align}
where $\inef_S$ are precomputed. 
Max-efficient rationality is inspired by the notion of core,
but unlike the core, max-efficient rationality does not require $\prop_E$ (efficiency) to hold. It is easy to show that the following properties are satisfied by any optimizer of \eqref{prob.rationality}, $\blamefunc_{MER}$.
\begin{proposition}\label{prop.mer_properties}
Every solution to the optimization problem \eqref{prob.rationality}, i.e., $\blamefunc_{MER}$, satisfies $\prop_V$ (validity), $\prop_R$ (rationality) and $\prop_I$ (invariance). 
\end{proposition}
Since there might exist multiple optimal solutions to \eqref{prob.rationality}, a tie breaking rule might be needed to decide on the method's output, $\blamefunc_{MER}$. 
We account for this fact in the experiments from Section \ref{sec.experiments}. 
Note that the constraints in \eqref{prob.rationality} are quite restrictive, leading to blame assignments that typically distribute very little blame in total. The amount of total blame assigned is important for explanatory power. Namely, a trivial blame attribution method that assigns the score of $0$ to every agent satisfies all of the properties from the previous section except $\prop_{E}$ (efficiency), but provides no information regarding the agents' contributions to the outcome. 

\subsection{Marginal Contribution}\label{sec.marginal_contribution}

Another intuitive blame assignment method is what we refer to as {\em marginal contribution}. This method simply quantifies an agent's potential to increase the performance of the system, assuming that the other agents keep their policies fixed. That is, the blame assigned to agent $i$ is equal to $\blamei = \inef_i$. The following properties hold: 
\begin{proposition}\label{prop.mc_properties}
$\blamefunc_{MC}(M, \bpi) = (\inef_1, ..., \inef_n)$ satisfies $\prop_S$ (symmetry), $\prop_I$ (invariance), $\prop_{CM}$ (contribution monotonicity) and $\prop_{PerM}$ (performance monotonicity).
\end{proposition}
Unlike max-efficient rationality, marginal contribution does not satisfy validity, i.e., it can over-blame a group of agents by assigning them total score that exceeds the improvement they can achieve, i.e., $\inef$. Given that an agent's marginal inefficiency is not always a good indicator of the agent's influence on the system's performance, this method can be highly inefficient (distributing very little blame) when coordination among agents is required, as we show in Section \ref{sec.experiments}.

\subsection{Shapley Value and Banzhaf Index}\label{sec.shapley_value}

In the context of the sequential decision making setting studied in this paper, Shapley value can be defined as $\blame = \blamefunc_{SV}(M, \bpi)$ such that
\begin{align}\label{eq.def_sv}
     \blame_i = \sum_{S \subseteq \{1, ..., n\} \backslash \{i\}} \weight_S \cdot \left [ \return(\optpirb_{S \cup \{ i \}}, \bpi_{-S\cup\{i\}}) - \return(\optpirb_{S}, \bpi_{-S}) \right ],
\end{align}
where coefficients $\weight_S$ are set to $\weight_S = \frac{|S|! ( n - |S| - 1)!}{n!}$. We restate (and in 
\iftoggle{longversion}{Appendix \ref{sec.proof_sv}}{the supplementary material},
 prove the claim for our setting) a well known uniqueness result for Shapley value:
\begin{theorem}\label{thm.shapley_value.unique}\cite{young1985monotonic}
$\blamefunc_{SV}(M, \bpi) = (\beta_1, ..., \beta_n)$, where $\beta_i$ is defined by Eq. \eqref{eq.def_sv} and $\weight_S = \frac{|S|! ( n - |S| - 1)!}{n!}$,  is a unique blame attribution method satisfying $\prop_E$ (efficiency), $\prop_S$ (symmetry) and $\prop_{CM}$ (contribution monotonicity). Additionally, $\blamefunc_{SV}$ satisfies $\prop_V$ (validity) and $\prop_I$ (invariance).
\end{theorem}
As we show in Section \ref{sec.experiments}, Shapley value does not satisfy properties $\prop_R$ (rationality) nor $\prop_{PerM}$ (performance monotonicity).

Banzhaf index, denoted by $\blamefunc_{BI}$, is similar to Shapley value, and in fact, it has the same functional form but different coefficients ($w_S = \frac{1}{2^{n-1}}$), leading to a slightly different uniqueness result.
\iftoggle{longversion}{Appendix \ref{sec.bi}}{The supplementary material}
discusses Banzhaf index and its properties in greater detail. Here, we note that Banzhaf index is equivalent to Shapley value for two agents. However, in general, Banzhaf index does not satisfy $\prop_E$ (efficiency), but a version of it, called $2$-efficiency~\cite{malawski2002equal}. As it is the case with Shapley value, Banzhaf index does not satisfy $\prop_{PerM}$ (performance monotonicity) nor $\prop_R$ (rationality). Interestingly, $\prop_V$ (validity) might also not hold (see Section \ref{sec.experiments}).

\subsection{Average Participation}\label{sec.avg_participation}

Motivated by the fact that $\prop_{PerM}$ (performance monotonicity) is important for incentivizing good performance and $\prop_E$ (efficiency) is important for explanatory power, we introduce a novel blame assignment method, which can be seen as a combination of marginal contribution and Shapley value. 
We first show the following result, which shows that there is an inherent trade-off between $\prop_{PerM}$ and $\prop_E$, assuming $\prop_S$ (symmetry) and $\prop_I$ (invariance) hold. 
\begin{proposition}\label{prop.imposs}
No blame attribution method $\blamefunc$ satisfies $\prop_E$ (efficiency), $\prop_S$ (symmetry), $\prop_I$ (invariance) and $\prop_{PerM}$ (performance monotonicity).
\end{proposition}
Given this result, we instead consider two new properties $\prop_{AE}$ (average efficiency) and $\prop_{cPerM}$ (c-performance monotonicity), which are weaker variants of $\prop_E$ (efficiency) and $\prop_{PerM}$ (performance monotonicity) respectively. Importantly, $\prop_{AE}$ is not satisfied by $\blamefunc_{MC}$ and $\prop_{cPerM}$ by $\blamefunc_{SV}$. In addition, we also consider two variants of $\prop_{CM}$ (contribution monotonicity):
$\prop_{cParM}$ (c-participation monotonicity) and $\prop_{RcParM}$ (relative c-participation monotonicity). 
To define the new properties, we introduce a contribution function 
$c: \mathcal{M} \times \Pi \times \{1,..., n\} \rightarrow \{0,1\}$
that indicates whether an agent is {\em pivotal}, i.e., marginally contributes to the inefficiency of some subset of $\{1, ...,n\}$:
\begin{align*}
    c(M, \bpi, i) = \begin{cases}
            0  \quad &\mbox{ if } \return(\optpirb_{S \cup \{ i \}}, \bpi_{-S\cup\{i\}}) = \return(\optpirb_{S}, \bpi_{-S}) \quad \forall S \subseteq \{1, ..., n\} \\ 
            1 \quad &\mbox{ otherwise  }
    \end{cases}.
\end{align*} 
Alternatively, an agent $i$ is pivotal if and only if its Shapley value is strictly greater than $0$, i.e., $c(M, \bpi, i) = \ind{\blamei > 0}$, where $\beta = \blamefunc_{SV}(M, \bpi)$ and $\ind{.}$ is an indicator function. The new properties are then defined as follows:

\begin{itemize}
\item {\em Average efficiency} ($\prop_{AE}$): $\blamefunc$ satisfies $\prop_{AE}$ (resp. $\epsilon$-$\prop_{AE}$) if for every $M$ and $\bpi$, $\sum_{i = 1}^n \blamei = \sum_{S \subseteq \{1, ..., n\}} \frac{1}{2^n - 1} \cdot \inef_S$ (resp. $|\sum_{i = 1}^n \blamei - \sum_{S \subseteq \{1, ..., n\}} \frac{1}{2^n - 1} \cdot \inef_S| \le \epsilon$), where $\blame = \blamefunc(M, \bpi)$.
\item {\em c-Performance monotonicity} ($\prop_{cPerM}$): Consider any MMDP $M$, and any $\bpi_{-i}$, $\pi_i$ and $\pi_i'$ s.t. $\return(\pi_i,\bpi_{-i}) \le \return(\pi_i',\bpi_{-i})$ and $c(M, (\pi_i,\bpi_{-i}), j) = c(M, (\pi_i',\bpi_{-i}), j)$ for every $j$. We say that $\blamefunc$ satisfies $\prop_{cPerM}$ (resp. $\epsilon$-$\prop_{cPerM}$) if $\blame_i \ge \blame_i'$ (resp. $\blame_i \ge \blame_i' - \epsilon$) where $\blame = \blamefunc(M, (\pi_i,\bpi_{-i}))$ and $\blame' = \blamefunc(M, (\pi_i', \bpi_{-i}))$. 
\item {\em c-Participation monotonicity} ($\prop_{cParM}$): $\blamefunc$ satisfies $\prop_{cParM}$ (resp. $\epsilon$-$\prop_{cParM}$) if for every $(M^1, \bpi^1)$ and $(M^2, \bpi^2)$
s.t. 
$c(M^1, \bpi^1, i) = c(M^2, \bpi^2, i)$ for every $i$, $\blamej^1 \geq \blamej^2$ (resp. $\blamej^1 \geq \blamej^2 - \epsilon$) whenever $\inef^1_{S \cup \{j \}} \geq \inef^2_{S \cup \{j \}}$ for all $S$, where $\blame^1 = \blamefunc(M^1, \bpi^1)$ and $\blame^2 = \blamefunc(M^2, \bpi^2)$.   
\item {\em Relative c-participation monotonicity} ($\prop_{RcParM}$): $\blamefunc$ satisfies $\prop_{RcParM}$ (resp. $\epsilon$-$\prop_{RcParM}$) if  for every $(M^1, \bpi^1)$ and $(M^2, \bpi^2)$ s.t. $c(M^1, \bpi^1, i) = c(M^2, \bpi^2, i)$  for every $i$, $\blamej^1 - \blamej^2 \geq \blamek^1 - \blamek^2$ (resp. $\blamej^1 - \blamej^2 \geq \blamek^1 - \blamek^2 - \epsilon$) whenever $c(M^1, \bpi^1, j) = c(M^1, \bpi^1, k)$ and $\inef^1_{S \cup \{j \}} - \inef^2_{S \cup \{j \}}\geq \inef^1_{S \cup \{k \}} - \inef^2_{S \cup \{k \}}$ for all $S \in \{1,...,n\} \backslash \{j,k\}$, where $\blame^1 = \blamefunc(M^1, \bpi^1)$ and $\blame^2 = \blamefunc(M^2, \bpi^2)$.
\end{itemize}

Before describing the main results of this subsection, we briefly outline the intuition behind the above definitions. $\prop_{AE}$ (average efficiency) is similar to $\prop_E$ (efficiency), however it requires less total blame to be distributed. Whereas $\prop_E$ requires that the total blame is equal to the total inefficiency $\inef$, $\prop_{AE}$ requires that the total blame is equal to the average marginal inefficiency of subsets of agents, i.e., the average value of $\inef_S$.\footnote{This average does not include $\inef_{\emptyset}$, which is equal to $0$. Note also that $\inef \ge \inef_S$ for every $S \subseteq \{1, ..., n\}$, so this average is upper bounded by $\inef$.} Compared to $\prop_{PerM}$ (performance monotonicity), $\prop_{cPerM}$ (c-performance monotonicity) additionally accounts for the pivotality of agents through contribution function $c$, treating each set of pivotal agents as a separate case. $\prop_{cParM}$ (c-participation monotonicity) accounts for agents' pivotality in a similar manner. Moreover, $\prop_{cParM}$ resembles contribution monotonicity $\prop_{CM}$, but instead of requiring blame monotonicity to hold w.r.t. the agent's influence on the marginal inefficiency of subsets $S$ (i.e., $\inef_{S \cup \{i \}} - \inef_S$), it considers blame monotonicity w.r.t. the marginal inefficiency of subsets that contain the agent (i.e., $\inef_{S \cup \{j \}}$).
Relative c-participation monotonicity $\prop_{RcParM}$ is similar to $\prop_{cParM}$, but its blame monotonicity requirement is based on a pairwise comparison of agents with the same pivotality degree. In particular, $\prop_{RcParM}$ requires that the blame increase is higher for an agent who is in subsets with a greater marginal inefficiency increase (i.e., $\blamej^1 - \blamej^2 \ge \blamek^1 - \blamek^2$ whenever $\inef^1_{S \cup \{j \}} - \inef^2_{S \cup \{j \}} \ge \inef^1_{S \cup \{k \}} - \inef^2_{S \cup \{k \}}$).

{\bf Average participation}: Now, we describe the new blame assignment method, which we call {\em average participation}. This blame assignment method can be defined as $\blame = \blamefunc_{AP}(M, \bpi)$ such that
\begin{align}\label{eq.def_avg_part}
     \blame_i = \sum_{S \subseteq \{1, ..., n\} \backslash \{i\}} \weight \cdot \frac{c(M, \bpi, i)}{\sum_{j \in S} c(M, \bpi, j) + 1} \cdot \inef_{S \cup \{ i \}}, 
\end{align}
where coefficient $w$ is set to $\weight = \frac{1}{2^n - 1}$. 
Intuitively, $\blamefunc_{AP}$ equally distributes blame for the marginal inefficiency of a subset of agents among the pivotal agents in that subset.
Hence, an agent $i$ that is pivotal receives blame for each subset $S \cup \{i\}$ equal to $\inef_{S \cup \{i\}}$ divided by the number of pivotal agents in $S \cup \{i\}$ and scaled by coefficient $\weight$. 
Agents that are not pivotal, obtain $0$ blame.
Average participation uniquely satisfies the following properties.
\begin{theorem}\label{thm.average_participation.unique}
$\blamefunc_{AP}(M, \bpi) = (\beta_1, ..., \beta_n)$, where $\beta_i$ is defined by Eq. \eqref{eq.def_avg_part} and $\weight = \frac{1}{2^n - 1}$,  is a unique blame attribution method that satisfies $\prop_{AE}$ (average-efficiency), $\prop_{S}$ (symmetry), $\prop_{I}$ (invariance), $\prop_{cParM}$ (c-participation monotonicity) and $\prop_{RcParM}$ (relative c-participation monotonicity). Furthermore, $\blamefunc_{AP}$ satisfies $\prop_{cPerM}$ (c-performance monotonicity) and $\prop_{V}$ (validity).
\end{theorem}
Unlike marginal contribution, average participation is valid (never over-blames agents), however it satisfies a weaker version of performance monotonicity. Still, this version is not satisfied by Shapley value. On the other hand, Shapley value is efficient, unlike average participation, which satisfies a weaker requirement---average efficiency. We also showcase these trade-offs in Section \ref{sec.experiments}.

%% file: 4_uncertainty_in_execution_behavior.tex
% !TEX root =  main.tex
%%%%%%%%%%%%%%%%%%%%%%%%%%%%%%%%%%%%%
%%%%%%%%%%%%%%%%%%%%%%%%%%%%%%%%%%%%%
\section{Blame Attribution under Uncertainty}\label{sec.blame_under_uncertainty}

In this section, we study blame attribution methods that do not have direct access to $\bpi$. As mentioned in Section \ref{sec.setting}, we focus on the case where the knowledge about $\bpi$ is defined by the uncertainty set $\mathcal P(\bpi)$, and it is defined state-wise so that each state is associated with a set of probability measures on $\mathcal{A}$ identifying plausible candidates for $\bpi(\cdot|s)$.\footnote{Such definition implies a rectangularity of the uncertainty set \citep{iyengar2005robust,tamar2014scaling}.} We denote $\pi \in \mathcal P(\bpi)$ if $\pi$ is plausible by $\mathcal P(\bpi)$. 

\subsection{Shapley Value under Uncertainty}\label{sec.sv_under_uncertainty}

In explaining approaches to handling uncertainty, we focus on Shapley value. Arguably, the simplest way to operate under uncertainty is to derive a point estimate of $\bpi$, denoted by $\bpie$,\footnote{For example, this estimate can be derived from data containing the agents' trajectories.} and apply $\blamefunc_{SV}$ on this estimate to obtain blame assignment $\widehat \blame = \blamefunc_{SV}(M, \bpie)$. Albeit being simple, this approach does not satisfy desirable properties, most notably, $\prop_V$ (validity) and $\prop_{BC}$ (Blackstone consistency). 

\textbf{Validity.}
Now, note that $\widehat \blame = \blamefunc_{SV}(M, \bpie)$ satisfies $\sum_{i=1}^n \widehat \blamei = \return(\optpi) - \return(\bpie)$. Therefore, instead of relying on a point estimate $\bpie$, we could utilize a policy $\bpie$ for which $\return(\optpi) - \return(\bpie) \le \inef$. Namely, in that case $\widehat \blame = \blamefunc_{SV}(M, \bpie)$ results in a blame assignment that satisfies $\prop_V$ (validity). 
Since this inequality holds for a solution to the optimization problem $\max_{\pi \in \mathcal P(\bpi)} \return(\pi)$, we obtain:
\begin{proposition}\label{prop.uncertainty_validity}
Let $\bpie$ be a solution to the optimization problem $\max_{\pi \in \mathcal P(\bpi)} \return (\pi)$. Then
$\widehat{\blamefunc}_{SV, V}(M, \mathcal P(\bpi)) = \blamefunc_{SV}(M, \bpie)$ satisfies $\prop_V$ (validity).
\end{proposition}
\textbf{Blackstone consistency.} As we show in Section \ref{sec.experiments}, although $\widehat{\blamefunc}_{SV, V}$ is valid, it might not be Blackstone consistent w.r.t. $\blamefunc_{SV}$. In particular, although the total blame is never overestimated, an agent $i$ might receive higher blame than it would receive under $\blamefunc_{SV}$. To ensure Blackstone consistency, we can assign blame to agent $i$ equal to 
$\min_{\pi \in \mathcal P(\bpi)} \blamei^\pi$ s.t. $\blame^\pi = \blamefunc_{SV}(M, \pi)$. 
Together with Eq. \eqref{eq.def_sv}, this implies that agent $i$'s blame is obtained by solving
\begin{align}
	\label{prob.robust_sv}
	\tag{P2}
	&\quad \min_{\pi \in \mathcal P(\bpi)} \sum_{S \subseteq \{1, ..., n\} \backslash \{i\}} \weight_S \cdot \left [ \return(\optpir_{S \cup \{ i \}}, \pi_{-S\cup\{i\}}) - \return(\optpir_{S}, \pi_{-S}) \right ],
\end{align}
where $\weight_S = \frac{|S|! ( n - |S| - 1)!}{n!}$ and $\optpir_S \in \argmax_{\pi'_S} \return(\pi'_S,\pi_{-S})$. We have the following result:
\begin{proposition}\label{prop.uncertainty_consistency}
Let $\blame^i_i$ be the minimum value of the objective in \eqref{prob.robust_sv}.
Then $\widehat{\blamefunc}_{SV, BC}(M, \mathcal P(\bpi)) = (\blame^1_1, ..., \blame^n_n)$ satisfies $\prop_V$ (validity) and $\prop_{BC}(\blamefunc_{SV})$ (Blackstone consistency  $w.r.t.$ $\blamefunc_{SV}(M, \bpi)$).
\end{proposition}
Note that $\widehat{\blamefunc}_{SV, BC}$ distributes less total blame than $\widehat{\blamefunc}_{SV, V}$, since it takes the worst case perspective for each agent separately, while under $\widehat{\blamefunc}_{SV, V}$ the blame assigned to all agents is computed with the same joint behavior policy. Moreover, the objective function in \eqref{prob.robust_sv} is more complex than in classical robust MDP settings~\cite{iyengar2005robust,nilim2005robust}, making classical approaches for robust MDPs hard to apply. In practice, we can relax \eqref{prob.robust_sv} and optimize a lower bound of the objective; this preserves $\prop_{BC}(\blamefunc_{SV})$, but at the expense of distributing less blame to the agents. In our experiments, we solve $\min_{\pi \in \mathcal P'(\bpi)} \return(\optpir_{S \cup \{ i \}}, \pi_{-S\cup\{i\}})$ and $\max_{\pi \in \mathcal P'(\bpi)}\return(\optpir_{S}, \pi_{-S})$ for each subset $S$ and with appropriately chosen $\mathcal P'(\bpi) \supseteq \mathcal P(\bpi)$ (see 
\iftoggle{longversion}{Appendix \ref{sec.app_uncertainty}}{the supplementary material}),
and we apply Eq. \eqref{eq.def_sv} to obtain the blame assignment. 
This implies that agent $i$'s blame is obtained by solving
\begin{align*}
	&\quad \sum_{S \subseteq \{1, ..., n\} \backslash \{i\}} \weight_S \cdot \left [ \min_{\pi \in \mathcal P'(\bpi)} \return(\optpir_{S \cup \{ i \}}, \pi_{-S\cup\{i\}}) - \max_{\pi \in \mathcal P'(\bpi)}\return(\optpir_{S}, \pi_{-S}) \right ].
\end{align*}

\textbf{Other Blame Attribution Methods.} Similar approaches also work for other blame assignment methods discussed in Section \ref{sec.approaches_to_blame}. For example, and focusing on Blackstone consistency, $\widehat{\blamefunc}_{BI, BC}(M, \mathcal P(\bpi))$ can be obtained in the same way as $\widehat{\blamefunc}_{SV, BC}(M, \mathcal P(\bpi))$, but with $\weight_S = \frac{1}{2^{n-1}}$, while $\widehat{\blamefunc}_{MC, BC}(M, \mathcal P(\bpi))$ can be implemented as $\widehat{\blamefunc}_{MC, BC}(M, \mathcal P(\bpi)) = (\tilde \inef_1, ..., \tilde \inef_n)$ where 
$\tilde \inef_i = \min_{\pi \in \mathcal P'(\bpi)} \return(\optpir_i,\pi_{-i}) - \max_{\pi \in \mathcal P'(\bpi)} \return(\pi)$.
Implementing Blackstone consistent $\widehat{\blamefunc}_{MER, BC}(M, \mathcal P(\bpi))$ and $\widehat{\blamefunc}_{AP, BC}(M, \mathcal P(\bpi))$ is more nuanced, and we discuss it in
\iftoggle{longversion}{Appendix \ref{sec.app_uncertainty}}{the supplementary material}.

\subsection{Characterization Result}\label{sec.approx_satisfability}

Notice that the described Blackstone consistent methods $\widehat \blamefunc(M, \mathcal P(\bpi))$ are not guaranteed to satisfy the properties that their counterparts $\blamefunc(M, \bpi)$ satisfy. However, as long as $\widehat \blamefunc(M, \mathcal P(\bpi))$ and $\blamefunc(M, \bpi)$ output similar enough blame assignments, properties that hold under $\blamefunc(M, \bpi)$ will approximately hold under $\widehat \blamefunc(M, \mathcal P(\bpi))$. 
More formally, we have the following results. 
\begin{theorem}\label{thm.approx_satisfability}
Consider $\widehat \blamefunc$ and $\blamefunc$ s.t. $\norm{\widehat \blamefunc(M, \mathcal P(\bpi)) -\blamefunc(M, \bpi)}_{1} \le \epsilon$ for any $M$, $\bpi$, and  $\mathcal P(\bpi)$. Then if $\blamefunc$ satisfies a property $\prop \in \{ \prop_V, \prop_E, \prop_R, \prop_S, \prop_I,  \prop_{AE}\}$,  $\widehat \blamefunc$ satisfies $\epsilon$-$\prop$. Moreover, if $\blamefunc$ satisfies a property $\prop \in \{ \prop_{CM} ,\prop_{PerM}, \prop_{cPerM}, \prop_{cParM}, \prop_{RcParM}\}$,  $\widehat \blamefunc$ satisfies $2\epsilon$-$\prop$.
\end{theorem}
\vspace{-0.1cm}
This theorem allows us to quantify the robustness of the blame attribution methods---the closer $\widehat \blamefunc$ is to $\blamefunc$, the more robust it is to uncertainty. 
Interestingly, a trivial blame attribution method that assigns $0$ blame to all the agents is robust in this sense. However, as we already mentioned, this trivial blame assignment is not informative as it does not attribute any blame. 
In fact, if agents receive no penalties for bad behavior, such a blame attribution method might have adverse effects. 
We provide a broader discussion on the negative side-effects of under-blaming in 
\iftoggle{longversion}{Appendix \ref{sec.disc}}{the supplementary material}.
Importantly, this example suggests that efficiency (in a broad sense, i.e., how much blame is being distributed) and robustness are sometimes at odds, which we also demonstrate in the experiments.

%% file: 5_experiments.tex
% !TEX root =  main.tex
%%%%%%%%%%%%%%%%%%%%%%%%%%%%%%%%%%%%%
%%%%%%%%%%%%%%%%%%%%%%%%%%%%%%%%%%%%%

\section{Experiments}\label{sec.experiments}

To demonstrate the efficacy of the studied blame attribution methods, we consider two environments, {\em Gridworld} and {\em Graph}, depicted in Fig. \ref{fig: gridworld} and Fig. \ref{fig: graph}. Both environments are adapted from \cite{voloshin2019empirical} and modified to be multi-agent.
The experiments evaluate blame attribution methods along three axis:
\begin{itemize}
    \item {\em Performance monotonicity:} First, we test blame attribution methods for the $\prop_{PerM}$ (performance monotonicity) property, which we deem important for accountability. To do that, we consider the 
    gridworld environment: this is a two-agent environment in which one of the agents, \agenttwo, optimizes its policy using a model of the other agent, \agentone. Importantly, by controlling the correctness of \agenttwo's model of \agentone, we can validate whether a blame attribution method satisfies $\prop_{PerM}$. Namely, if \agenttwo~does not receive the minimum blame when its model of \agentone~is the correct model, the corresponding method is not performance incentivizing, i.e., it does not satisfy $\prop_{PerM}$.
    \item {\em Coordination:} Second, we evaluate the efficacy of  blame attribution methods when a higher degree of coordination among agents is needed to yield improvements over the baseline behavior. 
    For this, we consider the graph environment, which includes configurations where an agent cannot improve the joint performance by unilaterally changing its policy. 
    Thus, this environment is suitable for evaluating whether blame attribution methods incorporate more nuanced counterfactual reasoning.
    \item {\em Robustness:} Finally, we evaluate the robustness of blame attribution methods under uncertainty. In this case, both environments (Gridworld and Graph) are used for testing purposes, and we control for the level of uncertainty over the agents' behavior policies. 
\end{itemize}
\iftoggle{longversion}{Appendix \ref{sec.appendix_exp}}{The supplementary material}
provides more details on the experimental setup and implementation. Below we provide a more detailed description of the considered environments and discuss our findings.

\textbf{Environment 1}: This is a gridworld environment, in which two agents control the same actor but 
\begin{minipage}{0.79\textwidth}
with different priorities. In the single-agent version of the environment, an agent, agent \agentone, controls the movement of the actor. 
In our multi-agent version, there is an additional agent, agent \agenttwo, who can intervene and override \agentone's actions. The two agents select their actions simultaneously.
Cells denoted with $S$ are the initial states, blank cells indicate areas of small negative reward, $F$ cells indicate areas of slightly increased cost and $H$ cells are areas of severe penalty. The cell denoted by $G$ is the terminal state of the environment and has a positive reward.
When agent \agenttwo~intervenes in some state, the actor takes the action that an optimal policy would select in the single-agent mode, but also pays a cost of
\end{minipage}
\hfill
\begin{minipage}{0.19\textwidth}
\captionsetup{type=figure}
\captionsetup{aboveskip=2pt,belowskip=2pt}
\includegraphics[width=\textwidth]{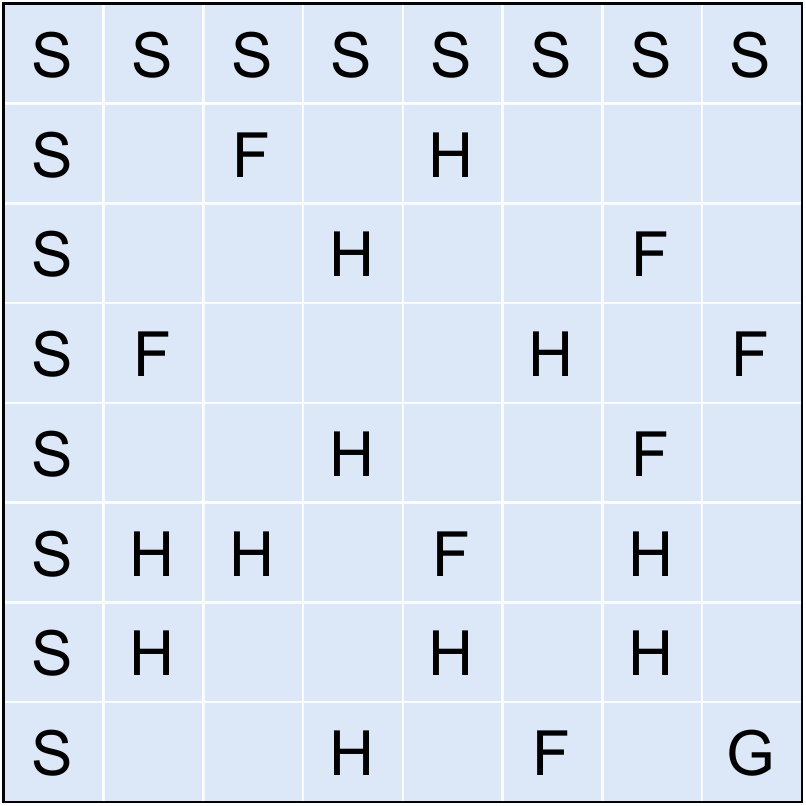}
\captionof{figure}{Gridworld}
\label{fig: gridworld}
\end{minipage}
intervention $C$. The behavior policy $\bpi_1$ of agent \agentone~is parameterized by variable $\alpha$, which specifies the probability that \agentone~takes an action determined by an optimal single-agent policy, instead of its personal policy.
The personal policy of \agentone~is a mixture of an optimal single-agent policy for correctly specified costs and a single-agent policy that is optimal but for misspecified costs of $F$ and $H$ cells---it assumes that they have the same cost as the blank cells.
\agenttwo's behavior policy $\bpi_2$~optimizes the expected discounted return and is trained with a model of \agentone~specified by the true personal policy of \agentone~and variable $\alpha'$ (not necessarily equal to $\alpha$). \agenttwo~is meant to rectify potential mistakes of \agentone~that could inflict cost greater than $C$.
In $\prop_{PerM}$ experiments we set $\alpha=0.4$. In robustness experiments, we only consider uncertainty over the personal policy of \agentone, and we set $\alpha=0.2$ and $\alpha'=0.5$.

{\bf Performance monotonicity}:
Fig. \ref{fig: per-monotonicity} validates our theoretical results regarding $\prop_{PerM}$ (performance monotonicity). More specifically, methods $\blamefunc_{AP}$ and $\blamefunc_{MC}$ assign the minimum blame to \agenttwo~when it acts optimally w.r.t. the true policy of \agentone~, i.e., when $\alpha'=\alpha$.
However, this is not the case for methods $\blamefunc_{SV}$ and $\blamefunc_{BI}$, which implies that these methods are not incentivizing \agenttwo~to act optimally w.r.t. its belief about \agentone. 
$\blamefunc_{MER}$ and $\blamefunc_{BI}$ assign the same blame to \agenttwo~as $\blamefunc_{MC}$ and $\blamefunc_{SV}$, respectively.

\begin{minipage}{0.79\textwidth}
\textbf{Environment 2}: This is a graph environment in which 4 agents simultaneously select actions. The graph consists of one starting and one terminal node, as well as 8 intermediate nodes that can be grouped according to their index number; nodes with even index number are located on the upper level of the graph and nodes with odd index number on the lower level. At each time-step every agent
\end{minipage}
\hfill
\begin{minipage}{0.19\textwidth}
\captionsetup{type=figure}
\captionsetup{aboveskip=2pt,belowskip=2pt}
\includegraphics[width=\textwidth]{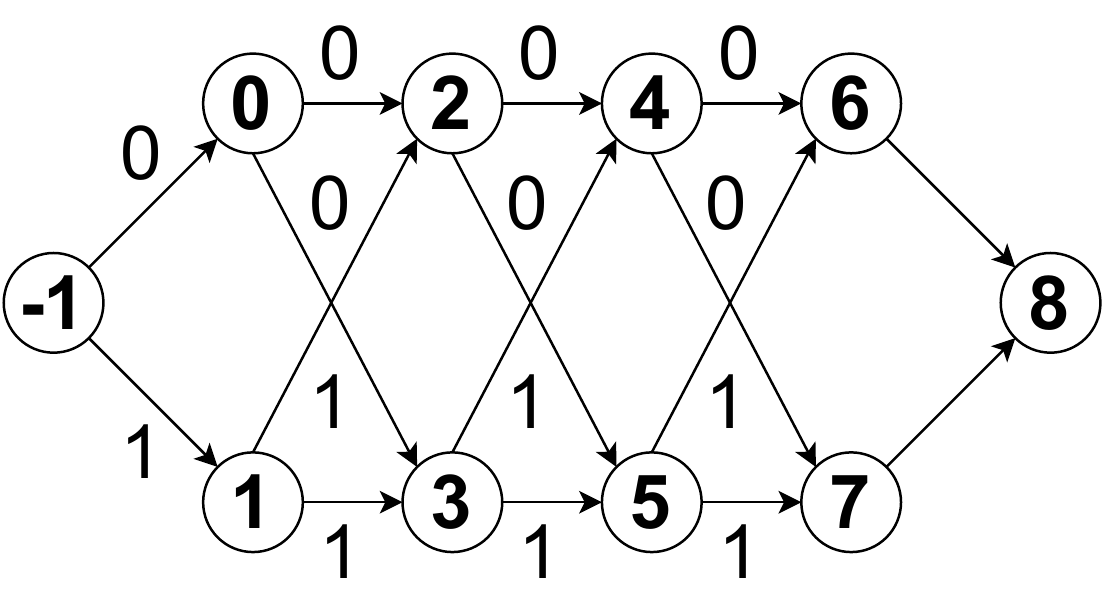}
\captionof{figure}{Graph}
\label{fig: graph}
\end{minipage}
chooses to take either action $0$ and move to the upper level or action $1$ and move to the lower level. We test multiple variants of this environment, each of which defines a different reward function. In all variants, the reward at each time-step is $+1$ if some formation constraint is satisfied and $-1$ if not.
In the first set of experiments (Coordination), we consider $4$ different formation constraints: in formation constraint $m \in \{1, ..., 4\}$, at least $m$ agents need to select action $1$ for the constraint to be satisfied. Each behavior policy $\bpi_i$ takes action $a_i = 0$ in every node. 
In the second set of experiments (Robustness), we consider one formation constraint that is satisfied if the agents are arranged equally between the two levels. In states where agents are balanced between the levels, each behavior policy $\bpi_i$ takes the action from the previous time-step with probability $p_i$;  in unbalanced states, the action that leads to the level with the least number of agents is taken with probability $p_i$.

\textbf{Coordination}: 
Fig. \ref{fig: graph-coordination} shows how much blame in total the blame attribution methods assign for the four different levels of required coordination ($m = 1,..., 4$). Observe, that when the constraint can be satisfied by every agent ($m = 1$), $\blamefunc_{MC}$ violates $\prop_V$ (validity). For $m = 2$, $\blamefunc_{MER}$ and $\blamefunc_{MC}$ assign zero blame to all agents, while $\blamefunc_{BI}$ violates $\prop_V$ (validity). Although always valid, $\blamefunc_{AP}$ assigns significantly less blame as $m$ increases. $\blamefunc_{SV}$ is always efficient, and its total blame 
does not vary with $m$. $\blamefunc_{SV}$, $\blamefunc_{BI}$, $\blamefunc_{MC}$ and $\blamefunc_{AP}$ do not satisfy $\prop_R$ 
(they assign more total blame than $\blamefunc_{MER}$).

\captionsetup[figure]{belowskip=-17pt}
\begin{figure*}
\captionsetup[subfigure]{aboveskip=0pt,belowskip=1pt}
\centering
\begin{subfigure}[b]{0.32\textwidth}
\centering
    \includegraphics[width=\textwidth, height=0.035\textheight]{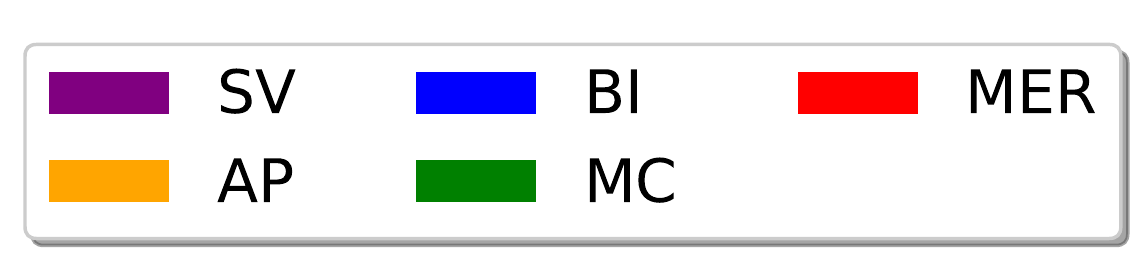}
\end{subfigure}
\begin{subfigure}[b]{0.32\textwidth}
\centering
    \includegraphics[width=\textwidth, height=0.035\textheight]{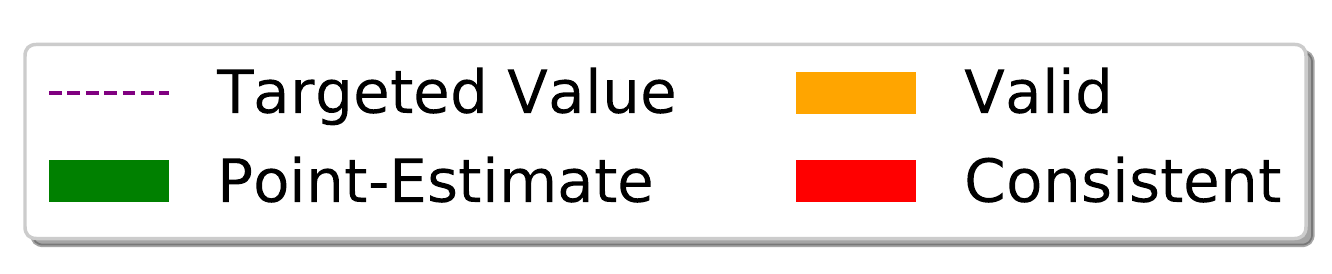}
\end{subfigure}
\begin{subfigure}[b]{0.32\textwidth}
    \includegraphics[width=\textwidth, height=0.035\textheight]{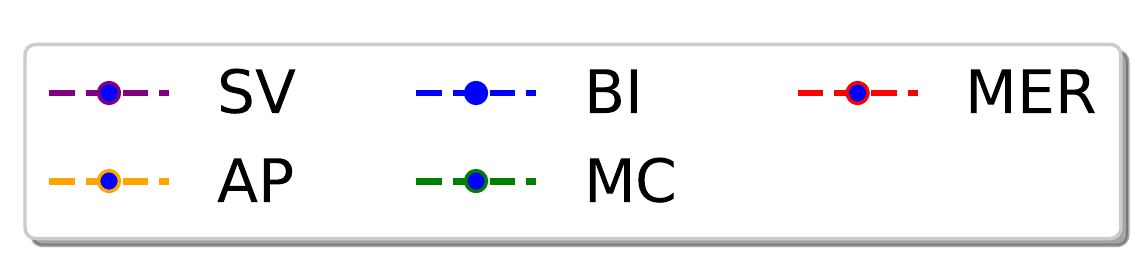}
\end{subfigure}
\\[-0.75ex]
\begin{subfigure}[b]{0.24\textwidth}
    \includegraphics[width=\textwidth]{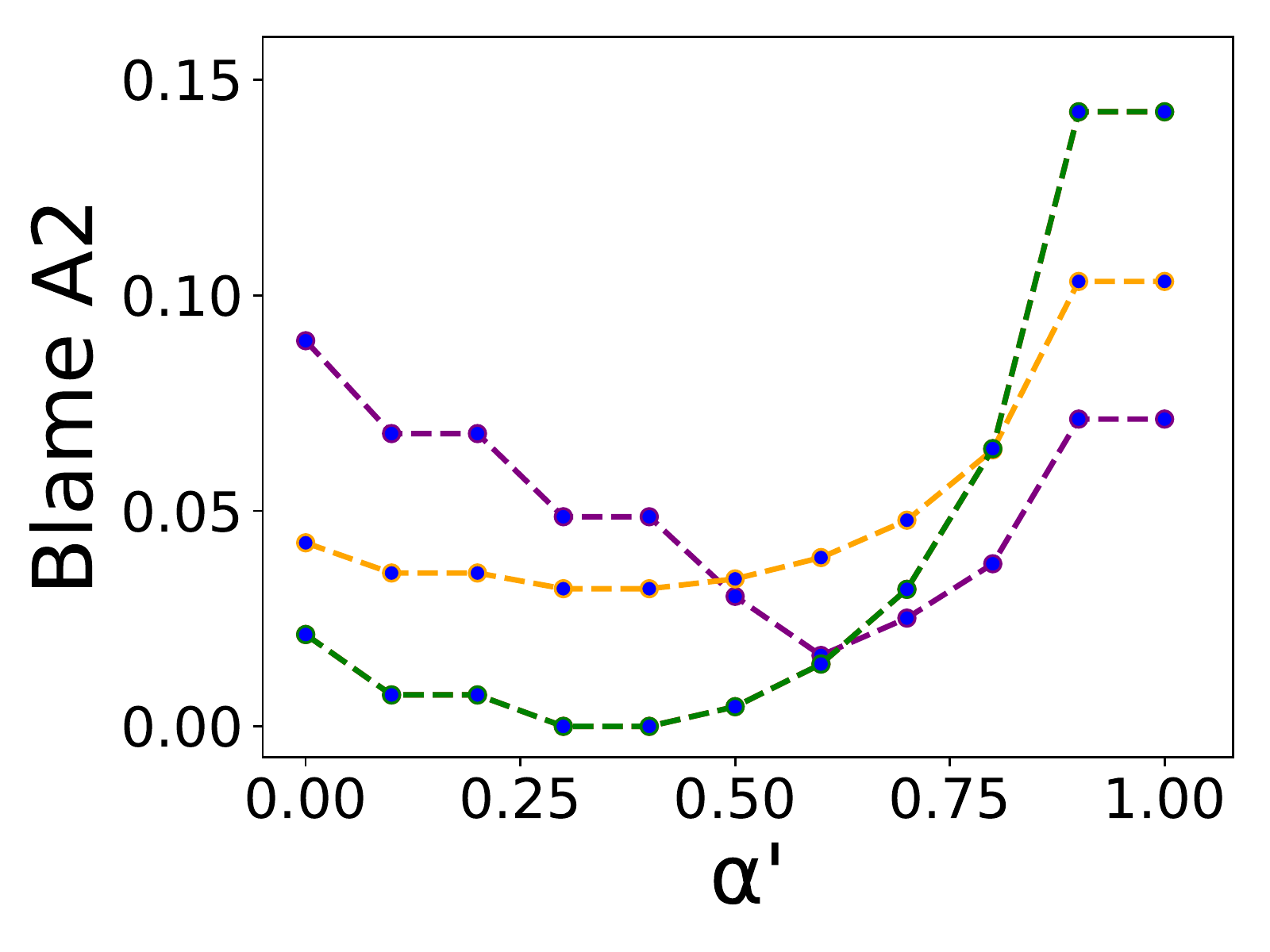}
    \caption{Gridworld: $\prop_{PerM}$}
    \label{fig: per-monotonicity}
\end{subfigure}
\begin{subfigure}[b]{0.24\textwidth}
    \includegraphics[width=\textwidth]{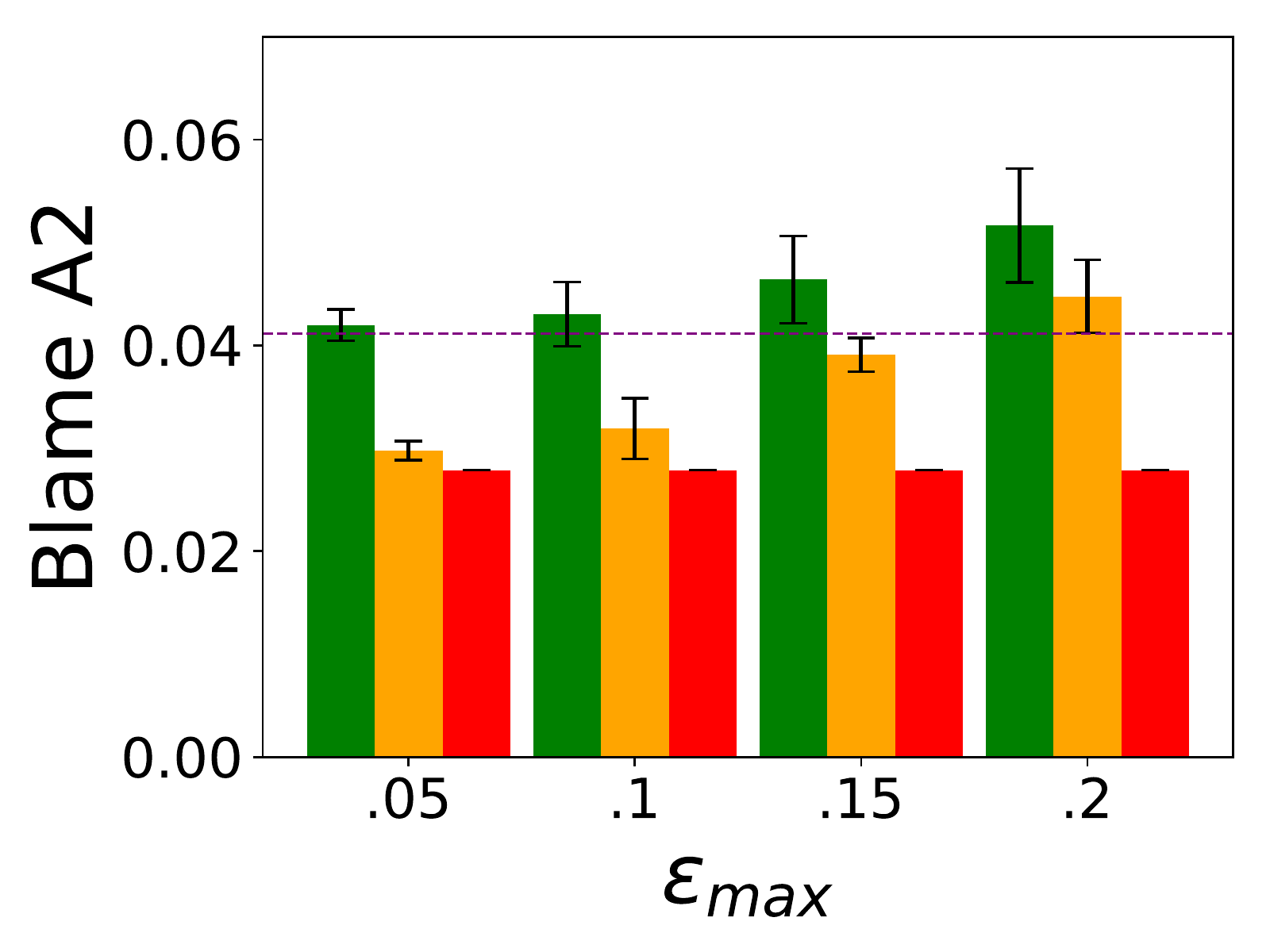}
    \caption{Gridworld: SV}
    \label{fig: gridworld-sv-methods}
\end{subfigure}
\begin{subfigure}[b]{0.24\textwidth}
    \includegraphics[width=\textwidth]{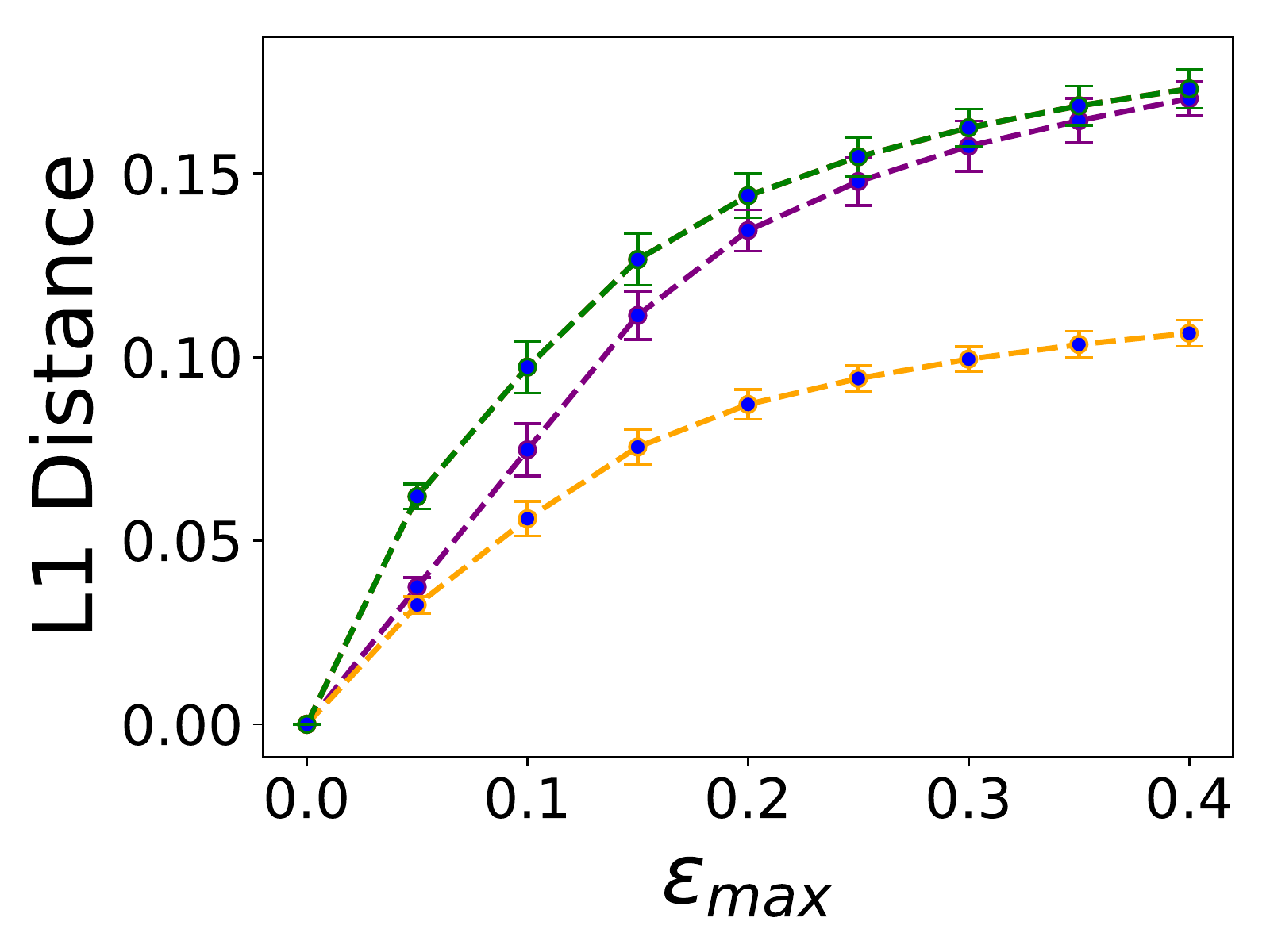}
    \caption{Gridworld: L1 Distance}
    \label{fig: gridworld-distance}
\end{subfigure}
\begin{subfigure}[b]{0.24\textwidth}
    \includegraphics[width=\textwidth]{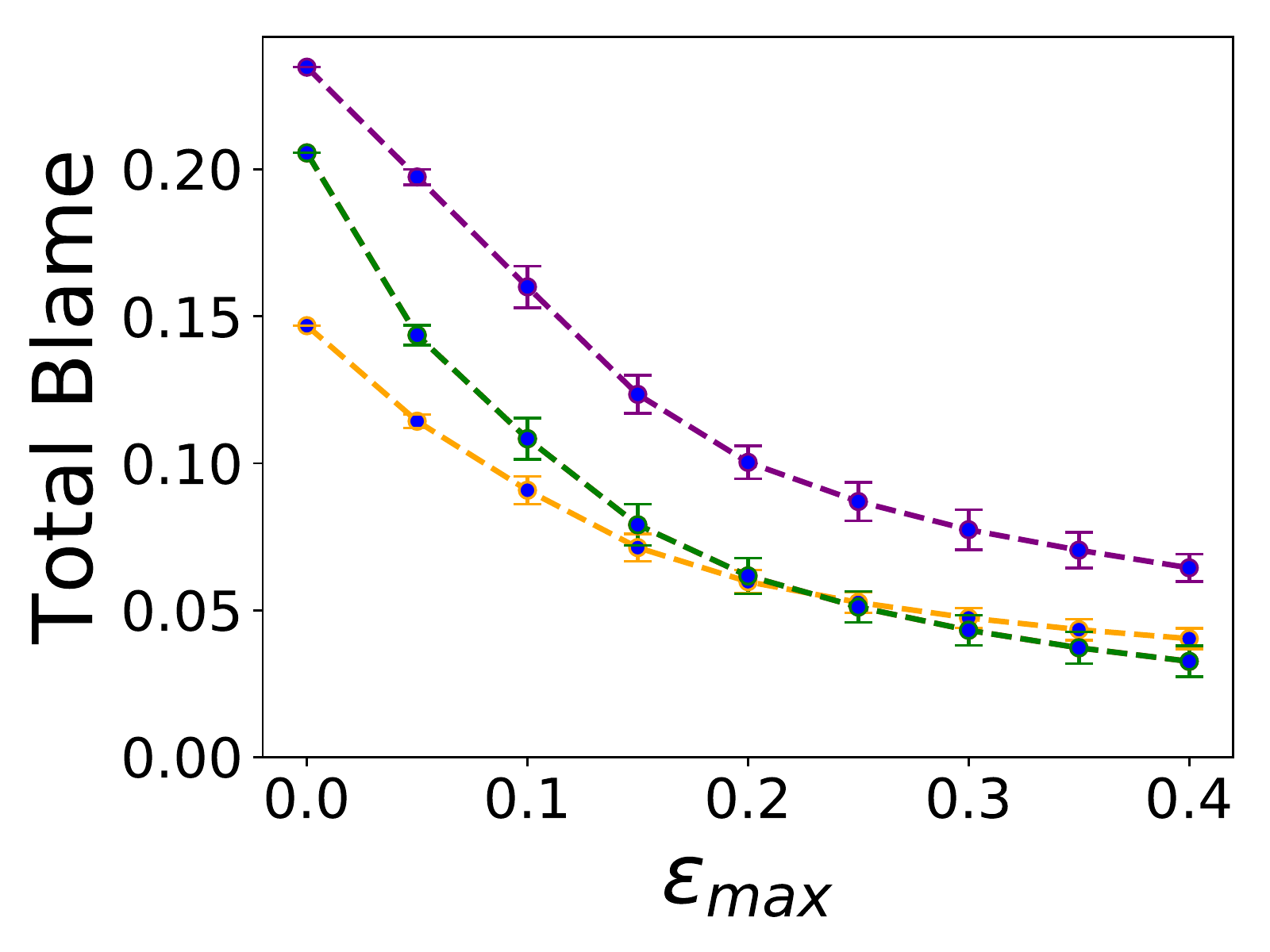}
    \caption{Gridworld: Blame}
    \label{fig: gridworld-total-blame}
\end{subfigure}
\begin{subfigure}[b]{0.24\textwidth}
    \includegraphics[width=\textwidth]{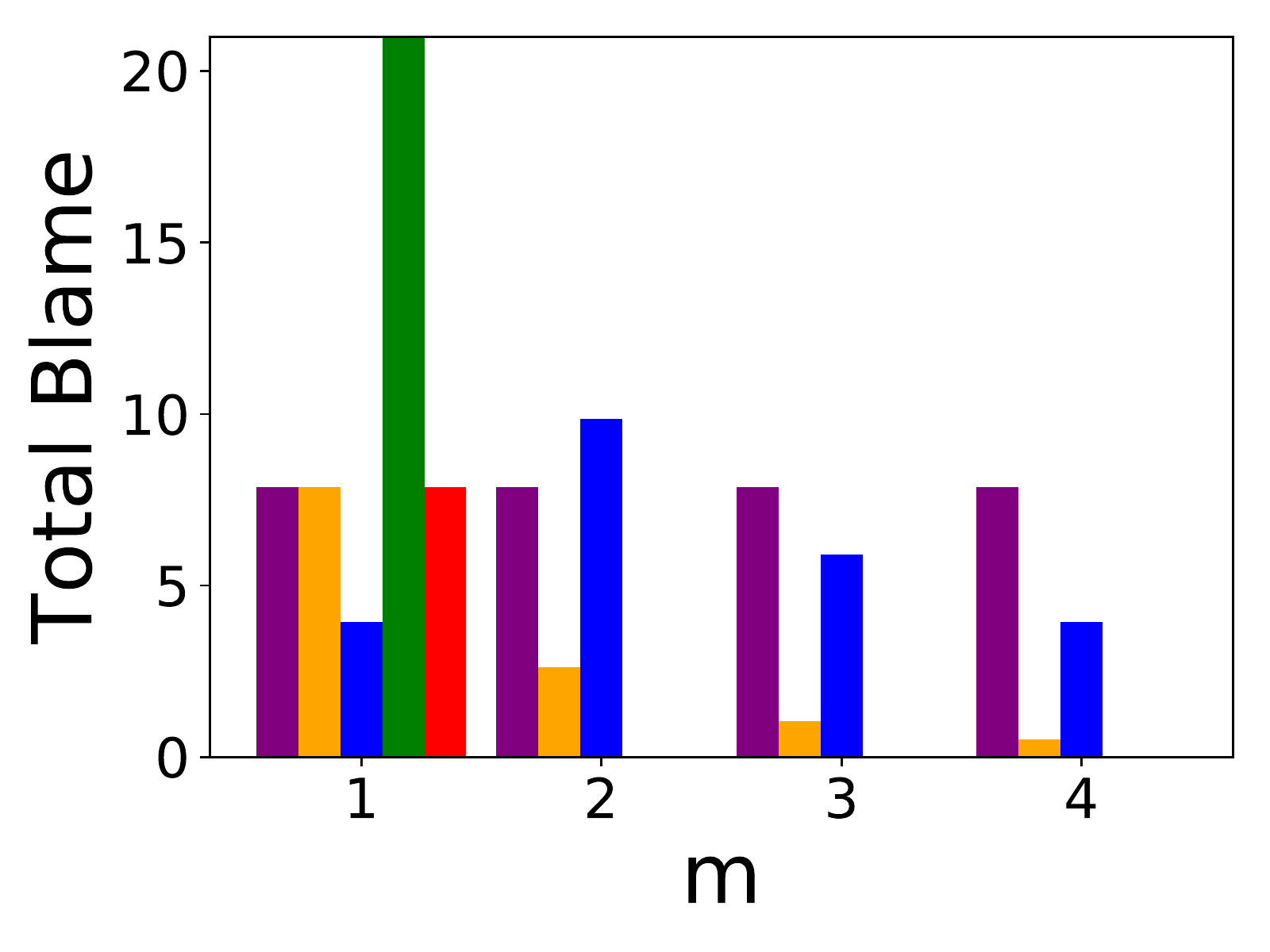}
    \caption{Graph: Coordination}
    \label{fig: graph-coordination}
\end{subfigure}
\begin{subfigure}[b]{0.24\textwidth}
    \includegraphics[width=\textwidth]{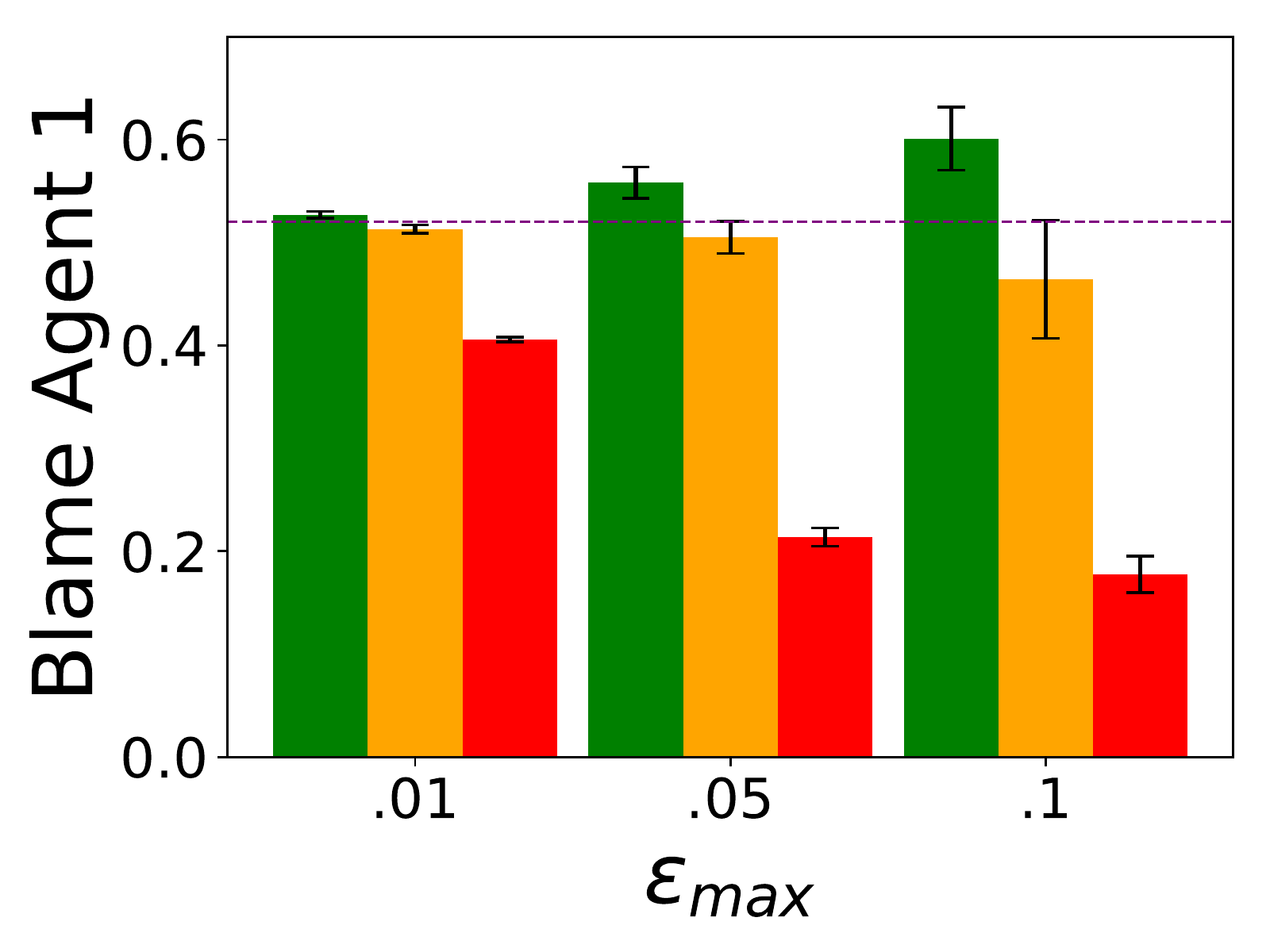}
    \caption{Graph: SV}
    \label{fig: graph-sv-methods}
\end{subfigure}
\begin{subfigure}[b]{0.24\textwidth}
    \includegraphics[width=\textwidth]{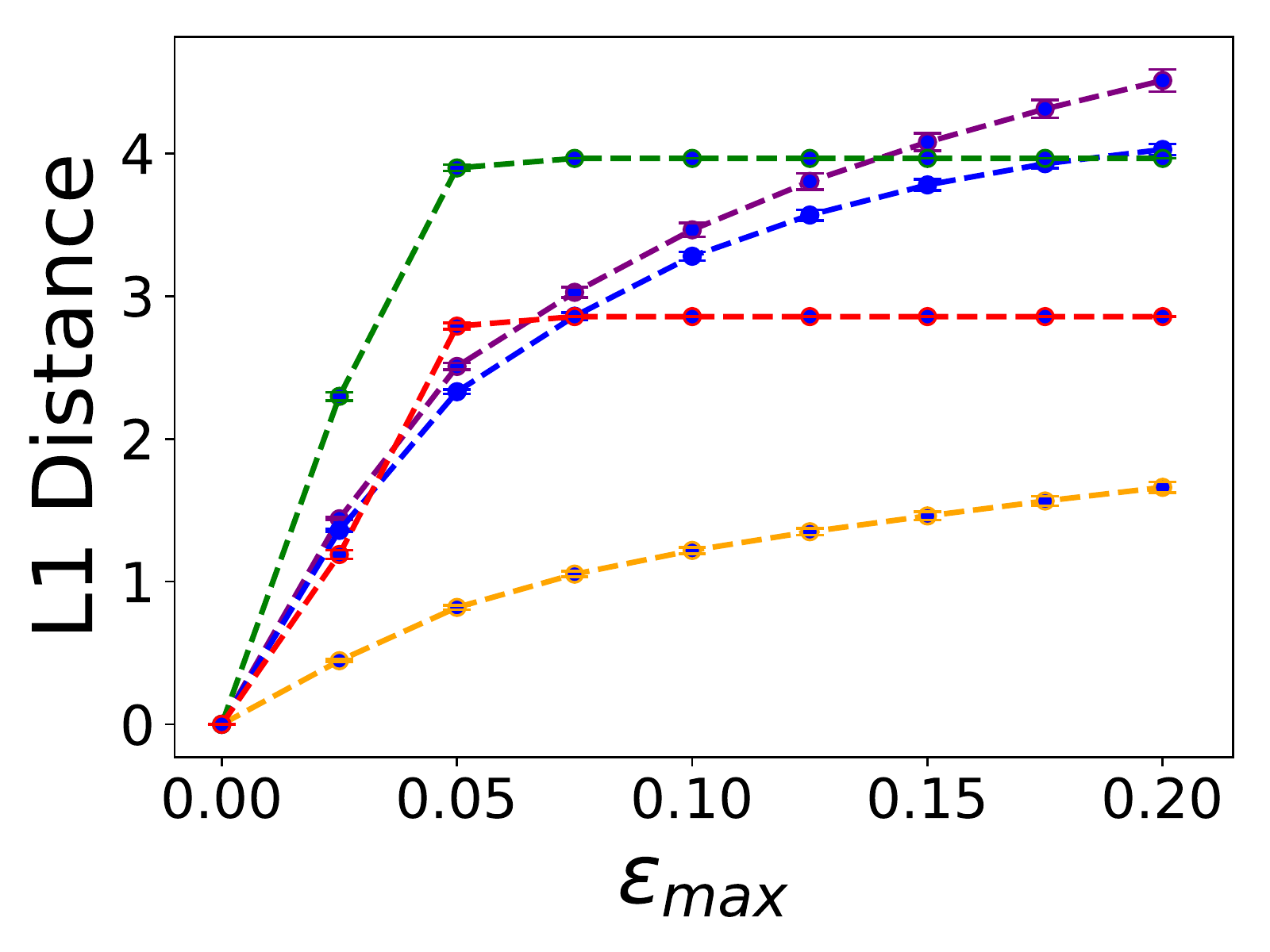}
    \caption{Graph: L1 Distance}
    \label{fig: graph-distance}
\end{subfigure}
\begin{subfigure}[b]{0.24\textwidth}
    \includegraphics[width=\textwidth]{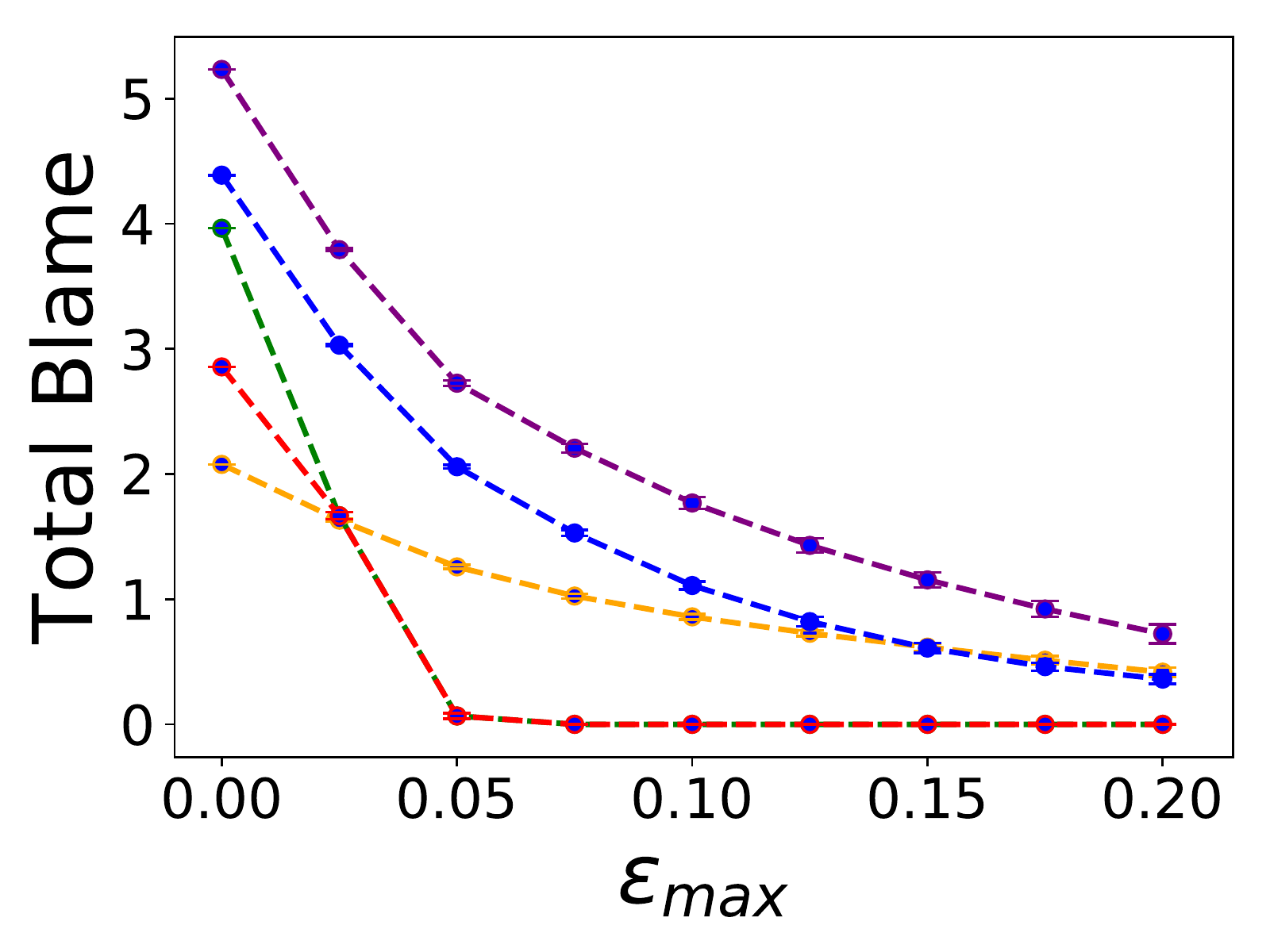}
    \caption{Graph: Blame}
    \label{fig: graph-total-blame}
\end{subfigure}
\caption{Experimental results for the Gridworld and Graph environments. Plot (\protect\subref{fig: per-monotonicity}) tests methods for $\prop_{PerM}$. Plot (\protect\subref{fig: graph-coordination}) shows the effect of varying coordination level. Plots (\protect\subref{fig: gridworld-sv-methods},\protect\subref{fig: gridworld-distance},\protect\subref{fig: gridworld-total-blame},\protect\subref{fig: graph-sv-methods},\protect\subref{fig: graph-distance},\protect\subref{fig: graph-total-blame}) show the effect of varying $\epsilon_{max}$ in different Shapley value approaches (\protect\subref{fig: gridworld-sv-methods},\protect\subref{fig: graph-sv-methods}) and blame attribution methods (\protect\subref{fig: gridworld-distance},\protect\subref{fig: gridworld-total-blame},\protect\subref{fig: graph-distance},\protect\subref{fig: graph-total-blame}).}
\end{figure*}

\textbf{Robustness}: 
We test the robustness of the blame attribution methods by controlling the amount of uncertainty in the estimates of the agents' behavior policies. 
To model uncertainty, we consider maximum estimation error $\epsilon_{max}$, and to obtain uncertainty sets $\mathcal P(\bpi)$, we  sample (uniformly at random) 
$\bpie_i(s)$ such that $\frac{1}{2}\norm{\bpie_i(s) - \bpi_i(s)}_1 \leq \epsilon_{max}$.
Moreover, $\mathcal P(\bpi, s)$ contains all policies $\pi$ such that $\frac{1}{2}\norm{\bpie_i(s) - \pi_i(s)}_1 \leq \epsilon_{max}$.
In our experiments, we take $\bpie$ to be the point estimate of $\bpi$. 

{\em Comparing estimation approaches}: Fig. \ref{fig: gridworld-sv-methods} and \ref{fig: graph-sv-methods} show how the approaches for estimating $\widehat{\blamefunc}_{SV}$ from Section \ref{sec.blame_under_uncertainty} fare under different levels of uncertainty. The point estimate approach typically over-blames an agent and the amount of over-blaming increases with the level of uncertainty.
$\widehat{\blamefunc}_{SV, BC}$ never over-blames any agent, but it becomes less efficient (in distributing blame) as $\epsilon_{max}$ increases (Fig. \ref{fig: graph-sv-methods}). $\widehat{\blamefunc}_{SV, V}$ is more efficient than $\widehat{\blamefunc}_{SV, BC}$, but violates $\prop_{BC}$ (Blackstone consistency) (Fig. \ref{fig: gridworld-sv-methods}).

{\em Comparing attribution approaches}:
Fig. \ref{fig: gridworld-distance} and \ref{fig: graph-distance} show for each consistent blame attribution method $\widehat{\blamefunc}$ from Section \ref{sec.blame_under_uncertainty} the $L_1$ distance between its output and the output of its counterpart $\blamefunc$ (``targeted assignment'').
 Fig. \ref{fig: gridworld-total-blame} and \ref{fig: graph-total-blame} show the total blame assigned by these methods. $\widehat{\blamefunc}_{AP,BC}$ consistently outperforms the other methods in terms of the $L_1$ distance from its ``targeted assignment''. Compared to $\widehat{\blamefunc}_{AP,BC}$, $\widehat{\blamefunc}_{SV, BC}$ is consistently better in terms of efficiency (in distributing blame). Similar, albeit less prominent effects can be seen when comparing  $\widehat{\blamefunc}_{AP,BC}$ and $\widehat{\blamefunc}_{BI,BC}$. These results indicate a tendency where efficiency (in distributing blame) and robustness are at odds, as we also discuss in Section \ref{sec.approx_satisfability}.
$\widehat{\blamefunc}_{MER,BC}$ and $\widehat{\blamefunc}_{MC,BC}$ assign zero total blame even for smaller $\epsilon_{max}$, indicating that they are the least robust to uncertainty.

%% file: 6_conclusion.tex
 % !TEX root =  main.tex
%%%%%%%%%%%%%%%%%%%%%%%%%%%%%%%%%%%%%
%%%%%%%%%%%%%%%%%%%%%%%%%%%%%%%%%%%%%
\vspace{-0.1cm}
\section{Conclusion}\label{sec.conclusion}
\vspace{-0.1cm}

In summary, the focus of our work is to provide an overview of possible computational approaches for attributing blame in multi-agent sequential decision making. We discuss the strengths and weaknesses of different methods in order to guide practitioners and policy makers in designing tools that support accountability. We conclude that there is no single best choice for blame attribution methods, since there are inherent trade-offs among properties that one might consider important. 
Looking forward, we recognize several research directions that could address the limitations of our results, some of which we highlight here. a) In this work we primarily focused on the agents' joint return as the outcome of interest. However, it is often important to pinpoint actual causes that led to more fine grained outcomes. Utilizing a causal perspective would be beneficial in this regard and could link our results to prior work (e.g., \cite{halpern2016actual}). b) We considered model-based approaches to blame assignment. Learning blame attribution directly from data (e.g., with model-free counterfactual RL) might be more practical in settings where an approximate model is hard to obtain. c) More generally, ensuring scalability both in the number of agents and the the richness of  environments is one of the most important steps for making this work more widely applicable. We deem approaches from multi-agent RL as suitable candidate solutions for resolving this problem. 
d) We primarily studied blame assignment properties that are taken from or closely relate to those from the game theory literature. This list could be extended and include more principles from moral philosophy and law. For example, in this paper, we adopted a consequentialist approach to blame attribution, focusing on the outcomes of the agents' behavior. Alternatively, one could take a deontological perspective, and focus on the alignment of an agent's behavior with a set of rules. 
We further discuss different perspectives on blame attribution in
\iftoggle{longversion}{Appendix \ref{sec.disc}}{the supplementary material}.
Finally, we would like to draw particular attention to the fact that there is no universal prioritization of properties that applies to all blame attribution problems and hence treating any generic analysis like ours as panacea without further justification, might have a negative impact to the agents that are being blamed.
To that end, we would like to emphasize that we see this work not as a final solution to the blame attribution problem, but as a starting point that shows challenges and trade-offs in distributing blame.

%% file: 9.0_appendix_main.tex
% !TEX root =  main.tex
%%%%%%%%%%%%%%%%%%%%%%%%%%%%%%%%%%%%% APPENDIX: MAIN
%%%%%%%%%%%%%%%%%%%%%%%%%%%%%%%%%%%%%
\section{List of Appendices}
In this section we provide a brief description of the content provided in the appendices of the paper.
\begin{itemize}
    \item Appendix \ref{sec.table} provides a table that summarizes the the results in Section \ref{sec.approaches_to_blame}.
    \item Appendix \ref{sec.bi} provides additional details on Banzhaf index.
    \item Appendix \ref{sec.app_uncertainty} provides additional details on blame attribution under uncertainty.
    \item Appendix \ref{sec.appendix_exp} provides additional details on experimental setup and implementation.
    \item Appendix \ref{sec.disc} provides an extended discussion on different perspectives on blame attribution and the negative side-effects of under-blaming agents.
    \item Appendix \ref{sec.prop_proofs_123} contains the proofs of the proposition from Section \ref{sec.approaches_to_blame} (Proposition \ref{prop.mer_properties}, Proposition \ref{prop.mc_properties}, and Proposition \ref{prop.imposs}).
    \item Appendix \ref{sec.proof_sv} contains the proof of Theorem \ref{thm.shapley_value.unique} from Section \ref{sec.approaches_to_blame}.
    \item Appendix \ref{sec.proof_ap} contains the proof of Theorem \ref{thm.average_participation.unique} from Section \ref{sec.approaches_to_blame}.
    \item Appendix \ref{sec.proof_uncertainty} contains the proofs of the formal results from Section \ref{sec.blame_under_uncertainty} (Proposition \ref{prop.uncertainty_validity}, Proposition \ref{prop.uncertainty_consistency}, and Theorem \ref{thm.approx_satisfability}).
\end{itemize}

\input{9.1_table_of_methods_and_properties}
\input{9.2_banzhaf_index}
\input{9.3_methods_for_uncertainty}
\input{9.4_experimental_setup}
\input{9.5_discussion}
\input{9.6_proofs}

%% file: 9.1_table_of_methods_and_properties.tex
% !TEX root =  main.tex
%%%%%%%%%%%%%%%%%%%%%%%%%%%%%%%%%%%%% APPENDIX: TABLE OF METHODS AND PROPERTIES
%%%%%%%%%%%%%%%%%%%%%%%%%%%%%%%%%%%%%

\section{Table of Methods and Properties}\label{sec.table}

In this section we provide a table that summarizes the results of Section \ref{sec.approaches_to_blame} and describes which blame attribution methods satisfy which properties. We use $(\checkmark)$ to denote that a method does not satisfy the exact property but a weaker version of it.

\captionsetup[table]{skip=5pt}
\begin{table}[h!]
    \begin{center}
        \resizebox{0.6\textwidth}{!}{
            \begin{tabular}{|c|c|c|c|c|c|}
                \hline
                 & $\blamefunc_{MER}$ & $\blamefunc_{MC}$ & $\blamefunc_{SV}$ & $\blamefunc_{BI}$ & $\blamefunc_{AP}$  \\
                \hline
                $\prop_{V}$ & \checkmark & & \checkmark & & \checkmark\\
                \hline
                $\prop_{E}$ & & & \checkmark & & (\checkmark)\\
                \hline
                $\prop_{R}$ & \checkmark & & & &\\
                \hline
                $\prop_{S}$ &  & \checkmark & \checkmark & \checkmark & \checkmark\\
                \hline
                $\prop_{I}$ & \checkmark & \checkmark & \checkmark & \checkmark & \checkmark\\
                \hline
                $\prop_{CM}$ & & \checkmark & \checkmark & \checkmark &\\
                \hline
                $\prop_{PerM}$ & & \checkmark & & & (\checkmark)\\
                \hline
            \end{tabular}}
    \captionof{table}{Summary of the characterization results from Section \ref{sec.approaches_to_blame}}
    \end{center}
\end{table} 
Method $\blamefunc_{AP}$ satisfies properties $\prop_{AE}$ and $\prop_{cPerM}$ which are weaker versions of $\prop_{E}$ and $\prop_{PerM}$, respectively.

%% file: 9.2_banzhaf_index.tex
% !TEX root =  main.tex
%%%%%%%%%%%%%%%%%%%%%%%%%%%%%%%%%%%%% APPENDIX: BANZHAF INDEX
%%%%%%%%%%%%%%%%%%%%%%%%%%%%%%%%%%%%%

\section{Banzhaf Index}\label{sec.bi}
In this section, we discuss in a greater detail Banzhaf index and its properties. In the context of the sequential decision making setting studied in this paper, Banzhaf Index can be defined as $\blame = \blamefunc_{BI}(M, \bpi)$ such that
\begin{align}\label{eq.def_bi}
     \blame_i = \sum_{S \subseteq \{1, ..., n\} \backslash \{i\}} \weight_S \cdot \left [ \return(\optpirb_{S \cup \{ i \}}, \bpi_{-S\cup\{i\}}) - \return(\optpirb_{S}, \bpi_{-S}) \right ],  
\end{align}
where coefficients $\weight_S$ are set to $\weight_S = \frac{1}{2^{n-1}}$. The following properties hold: 
\begin{proposition}\label{prop.bi_properties}
$\blamefunc_{BI}(M, \bpi) = (\beta_1, ..., \beta_n)$, where $\beta_i$ is defined by Eq. \eqref{eq.def_bi} and $\weight_S = \frac{1}{2^{n-1}}$, is a blame attribution method satisfying $\prop_S$ (symmetry), $\prop_I$ (invariance) and $\prop_{CM}$ (contribution monotonicity).
\end{proposition}
\begin{proof}
First, notice that Banzhaf Index can be redefined as $\blame = \blamefunc_{BI}(M, \bpi)$ such that:
\begin{align}\label{eq.def_bi_inef}
     \blame_i = \sum_{S \subseteq \{1, ..., n\} \backslash \{i\}} \weight_S \cdot \left [ \inef_{S \cup \{i\}} - \inef_S \right ].   
\end{align}
We prove the properties as follows:
\begin{itemize}
    \item \underline{$\prop_S$ (symmetry):} Consider $M$, $\bpi$, and agents $i$ and $j$ such that  $\inef_{S \cup \{i \}} = \inef_{S \cup \{j \}}$ for all $S \subseteq \{1, ..., n\} \backslash \{i, j\}$. Notice that $\inef_{S \cup \{i\}} - \inef_S = \inef_{S \cup \{j\}} - \inef_S$ and $\inef_{S \cup \{i,j\}} - \inef_{S \cup \{j\}} = \inef_{S \cup \{i,j\}} - \inef_{S \cup \{i\}}$ for all $S \subseteq \{1, ..., n\} \backslash \{i, j\}$. Given the definition of $\blame = \blamefunc_{BI}(M, \bpi)$, this implies that $\blame_i = \blame_j$, and hence property $\prop_S$ (symmetry) is satisfied.
    \item \underline{$\prop_I$ (invariance):} Consider $M$, $\bpi$, and agent $i$
    such that $\inef_{S \cup \{i \}} = \inef_{S}$ for all $S$. Given the definition of $\blame = \blamefunc_{BI}(M, \bpi)$, this implies that $\blamei = 0$, and hence property $\prop_I$ (invariance) is satisfied.
    \item \underline{$\prop_{CM}$ (contribution monotonicity):} Consider $M^1$, $\bpi^1$, $M^2$, $\bpi^2$, and agent $i$ such that $\inef^1_{S \cup \{i \}} - \inef^1_S \geq \inef^2_{S \cup \{i \}} - \inef^2_S$ for all $S$. By using the definitions of $\blame^1 = \blamefunc_{BI}(M^1, \bpi^1)$ and $\blame^2 = \blamefunc_{BI}(M^2, \bpi^2)$, this implies that:
    \begin{align*}
        \blamei^1 =& \sum_{S \subseteq \{1, ..., n\} \backslash \{i\}} \weight_S \cdot \left [ \inef^1_{S \cup \{i\}} - \inef^1_S \right ] \geq \\
        \geq& \sum_{S \subseteq \{1, ..., n\} \backslash \{i\}} \weight_S \cdot \left [ \inef^2_{S \cup \{i\}} - \inef^2_S \right ] = \\
        =& \blamei^2,
    \end{align*}
\end{itemize}
and hence property $\prop_{CM}$ (contribution monotonicity) is satisfied.
\end{proof}
In general, Banzhaf index satisfies a property called $2$-efficiency~\cite{malawski2002equal} which leads to a slightly different uniqueness result than the one of Theorem \ref{thm.shapley_value.unique}. This property and the corresponding analysis are out of the scope of this paper, and we refer the reader to \cite{malawski2002equal,datta2015program} for more details. 

%% file: 9.3_methods_for_uncertainty.tex
% !TEX root =  main.tex
%%%%%%%%%%%%%%%%%%%%%%%%%%%%%%%%%%%%% APPENDIX: METHODS FOR UNCERTAINTY
%%%%%%%%%%%%%%%%%%%%%%%%%%%%%%%%%%%%%

\section{Additional Information on Blame Attribution under Uncertainty}\label{sec.app_uncertainty}

In this section, we provide additional information on the optimization problems defined in Section \ref{sec.sv_under_uncertainty} and the implementation of Blackstone consistent $\widehat{\blamefunc}_{MER, BC}(M, \mathcal P(\bpi))$ and $\widehat{\blamefunc}_{AP, BC}(M, \mathcal P(\bpi))$.

\subsection{Implementation of Optimization Problems}

In this section, we provide implementation details on the optimization problems defined in Section \ref{sec.sv_under_uncertainty}, for obtaining Valid and Blackstone consistent blame attribution methods. More specifically, we focus on the optimization problems $\min_{\pi \in \mathcal P'(\bpi)} \return(\optpir_{S \cup \{i\}}, \pi_{-S \cup \{i\}})$ and $\max_{\pi \in \mathcal P'(\bpi)}\return(\optpir_{S}, \pi_{-S})$, where $S \subseteq \{1, ..., n\}$ and $\mathcal P'(\bpi) \supseteq \mathcal P(\bpi)$. 
We consider 
\begin{align*}
    \mathcal P(\bpi) = \bigg\{& \pi | \pi(a|s) = \pi_1(a_1|s) \cdots \pi_n(a_n|s), \frac{1}{2} \cdot \norm{\pi_i(\cdot|s) - \pi_i^{bas}(\cdot|s)}_1 \le C, 0 \le \pi_i(a_i|s) \le 1, \\
    &\sum_{a_i \in \mathcal{A}_i} \pi_i(a_i|s) = 1\bigg\},
\end{align*}
where $C$ is a non-negative constant and $\pi^{bas}$ is a baseline joint policy.
In specific cases, we can set $\mathcal P'(\bpi) = \mathcal P(\bpi)$ and we discuss these cases below. In general, to more directly relate the optimization problems to prior work on robust optimization in MDPs~\cite{iyengar2005robust,nilim2005robust}, we relax the constraint that $\pi$ factorizes to $\pi(a|s) = \pi_1(a_1|s) \cdots \pi_n(a_n|s)$, and consider 
\begin{align*}
    \mathcal P'(\bpi) = \bigg\{& \pi | \prod_{i = 1}^{n} \max(\pi_i^{bas}(a_i|s) - C, 0)\le \pi(a_1, ..., a_n|s) \le \prod_{i = 1}^{n} \min(\pi_i^{bas}(a_i|s) + C, 1), \\
    &\sum_{(a_1, ..., a_n) \in \mathcal{A}} \pi(a_1, ..., a_n|s) = 1 \bigg\}.
\end{align*}
Notice that since $\sum_{a_i \in \mathcal{A}_i} \pi^{bas}_i(a_i|s) = 1$, we have that $\pi_i^{bas}(a_i|s) - C\le \pi_i(a_i|s) \le \pi_i^{bas}(a_i|s) + C$ for every $\pi \in \mathcal P(\bpi)$, and hence $\mathcal P''(\bpi) \supseteq \mathcal P(\bpi)$, where 
\begin{align*}
    \mathcal P''(\bpi) = \bigg\{& \pi | \pi(a|s) = \pi_1(a_1|s) \cdots \pi_n(a_n|s), \max(\pi_i^{bas}(a_i|s) - C, 0) \le \pi_i(a_i|s) \\
    & \le \min(\pi_i^{bas}(a_i|s) + C, 1), \sum_{a_i \in \mathcal{A}_i} \pi_i(a_i|s) = 1 \bigg\}.
\end{align*}
Importantly, $\mathcal P'(\bpi) \supseteq \mathcal P''(\bpi)$ implies that $\mathcal P'(\bpi) \supseteq \mathcal P(\bpi)$, which means that $\max_{\pi \in \mathcal P'(\bpi)}\return(\optpir_{S}, \pi_{-S})$ upper bounds $\max_{\pi \in \mathcal P(\bpi)}\return(\optpir_{S}, \pi_{-S})$ and $\min_{\pi \in \mathcal P'(\bpi)} \return(\optpir_{S \cup \{i\}}, \pi_{-S \cup \{i\}})$ lower bounds $\min_{\pi \in \mathcal P(\bpi)} \return(\optpir_{S \cup \{i\}}, \pi_{-S \cup \{i\}})$. Therefore, $\min_{\pi \in \mathcal P'(\bpi)} \return(\optpir_{S \cup \{i\}}, \pi_{-S \cup \{i\}})$ and $\max_{\pi \in \mathcal P'(\bpi)}\return(\optpir_{S}, \pi_{-S})$ can be used for deriving valid and Blackstone consistent blame assignments (e.g., by applying Eq. \eqref{eq.def_sv} with the obtained solutions). Next, we discuss how to solve these optimization problems. 

While \cite{iyengar2005robust,nilim2005robust} consider uncertainty over transitions dynamics instead of behavior policies, we can solve $\max_{\pi \in \mathcal P'(\bpi)}\return(\optpir_{S}, \pi_{-S})$ and $\min_{\pi \in \mathcal P'(\bpi)} \return(\optpir_{S \cup \{i\}}, \pi_{-S \cup \{i\}})$ by adapting their robust optimization techniques. To solve the optimization problem $\min_{\pi \in \mathcal P'(\bpi)} \return(\optpir_{S \cup \{ i \}}, \pi_{-S\cup\{i\}})$ for subset $S$, we apply the following recursion (in each iteration updating values for each state $s$):
\begin{align*}
    & \tilde \pi(\cdot|s) \leftarrow \argmin_{\pi(\cdot | s) \in \mathcal P'(\bpi, s)} \max_{a_{S\cup \{ i \}}} \sum_{a_{-S\cup \{ i \}}} \pi_{-S\cup \{ i \}}(a_{-S\cup \{ i \}}|s) \cdot \bigg [R(s, a) 
    + \gamma \cdot \sum_{s'} P(s, a, s') \cdot V^k(s')\bigg ],\\
    & V^{k+1}(s) \leftarrow \max_{a_{S\cup \{ i \}}} \sum_{a_{-S\cup \{ i \}}} \tilde \pi_{-S\cup \{ i \}} (a_{-S\cup \{ i \}}|s) \cdot \bigg [R(s, a) 
    + \gamma \cdot \sum_{s'} P(s, a, s') \cdot V^k(s')\bigg ],
\end{align*}
for $k = 1, 2, \dots$, where $V: \mathcal{S} \rightarrow \mathbb{R}_{\geq 0}$ is the value function, $a_S$ denotes the joint action of agents $S$, $a_{-S}$ denotes the joint action of agents $\{1, ..., n\}\backslash S$, and $a$ is the joint action of all the agents. The optimization problem for finding $\tilde \pi$ can be solved via a linear program that minimizes a dummy variable which is constrained to be at least as large as $$\sum_{a_{-S\cup \{ i \}}} \pi_{-S\cup \{ i \}}(a_{-S\cup \{ i \}}|s) \cdot \bigg [R(s, a) 
    + \gamma \cdot \sum_{s'} P(s, a, s') \cdot V^k(s')\bigg ]$$ for all $a_{S \cup \{ i \}}$. 
The optimization problem for finding $V^{k+1}$ can be solved by simply searching over all possible $a_{S \cup \{ i \}}$. 
Similarly, we can solve $\max_{\pi \in \mathcal P'(\bpi)}\return(\optpir_{S}, \pi_{-S})$ with the following recursion:
\begin{align*}
    & \tilde \pi(\cdot|s) \leftarrow \argmax_{\pi(\cdot | s) \in \mathcal P'(\bpi, s)} \max_{a_{S}} \sum_{a_{-S}} \pi_{-S}(a_{-S}|s) \cdot \bigg [R(s, a) 
    + \gamma \cdot \sum_{s'} P(s, a, s') \cdot V^k(s')\bigg ],\\
    & V^{k+1}(s) \leftarrow \max_{a_{S}} \sum_{a_{-S}} \tilde \pi_{-S}(a_{-S}|s) \cdot \bigg [R(s, a) 
    + \gamma \cdot \sum_{s'} P(s, a, s') \cdot V^k(s')\bigg ],
\end{align*}
for $k = 1, 2, \dots$. The optimization problem for finding $\tilde \pi$ can be solved by searching over all $a_S$ and selecting one that maximizes $$\max_{\pi(\cdot | s) \in \mathcal P'(\bpi, s)} \sum_{a_{-S}} \pi_{-S}(a_{-S}|s) \cdot \bigg [R(s, a) 
    + \gamma \cdot \sum_{s'} P(s, a, s') \cdot V^k(s')\bigg ]$$ ---the solution to this problem gives us the corresponding $\tilde \pi$. The optimization problem for finding $V^{k+1}$ can be solved by searching over all possible $a_{S}$. 
    The two recursions described above define dynamic programming  techniques that are analogs of those in~\cite{iyengar2005robust,nilim2005robust}, but applied for uncertainty over behavior policies. They can be solved efficiently for smaller action spaces $\mathcal{A}$, e.g., as those in our experiments. 

Now, in specific cases, we can set $\mathcal P'(\bpi) = \mathcal P(\bpi)$, which in turn can lead to more efficient blame assignments (since the estimates are tighter). We consider the following two cases:

\begin{itemize}
    \item First, when there are only two agents in an MMDP, $-S \cup \{ i \}$ contains at most one agent. Therefore, we could run the first recursion on $\{ \pi_j | \frac{1}{2} \cdot \norm{\pi_j(\cdot|s) - \pi_j^{bas}(\cdot|s)}_1 \le C, 0 \le \pi_j(a_j|s) \le 1, \sum_{a_j \in \mathcal{A}_j} \pi_j(a_j|s) = 1 \}$ instead of $\{\pi_j | \max(\pi_j^{bas}(a_j|s) - C, 0) \le \pi_j(a_j|s) \le \min(\pi_j^{bas}(a_j|s) + C, 1), \sum_{a_j \in \mathcal{A}_j} \pi_j(a_j|s) = 1 \}$ and thus solve $\min_{\pi \in \mathcal P(\bpi)} \return(\optpir_{S \cup \{i\}}, \pi_{-S \cup \{i\}})$. Also in that case, $-S$ contains at most one agent whenever $S \neq \emptyset$, and hence we could run the second recursion on $\{ \pi_j | \frac{1}{2} \cdot \norm {\pi_j(\cdot|s) - \pi_j^{bas}(\cdot|s)}_1 \le C, 0 \le \pi_j(a_j|s) \le 1, \sum_{a_j \in \mathcal{A}_j} \pi_j(a_j|s) = 1 \}$, and thus solve $\max_{\pi \in \mathcal P(\bpi)}\return(\optpir_{S}, \pi_{-S})$ for every $S \neq \emptyset$.
    In addition, when the optimal policies of one of the agents, agent $i$, are independent of which policy the other agent, agent $j$, follows we can directly compute an optimal policy for $i$ on 
    $\{ \pi_i | \frac{1}{2} \cdot\norm{\pi_i(\cdot|s) - \pi_i^{bas}(\cdot|s)}_1 \le C, 0 \le \pi_i(a_i|s) \le 1, \sum_{a_i \in \mathcal{A}_i} \pi_i(a_i|s) = 1 \}$, by fixing an arbitrary policy to agent $j$. Then, by fixing agent $i$ to its optimal policy, we can can directly compute an optimal policy of agent $j$ on 
    $\{ \pi_j | \frac{1}{2} \cdot\norm{\pi_j(\cdot|s) - \pi_j^{bas}(\cdot|s)}_1 \le C, 0 \le \pi_j(a_j|s) \le 1, \sum_{a_j \in \mathcal{A}_j} \pi_j(a_j|s) = 1 \}$. This implies that we can run the second recursion directly on $\mathcal P(\bpi)$ for $S = \emptyset$ and thus solve $\max_{\pi \in \mathcal P(\bpi)}\return(\pi)$. We use these facts in our experiments for the Gridworld environment, where the optimal policies of \agentone~are independent of \agenttwo's policy.
    \item Another specific case is when action spaces $\mathcal{A}_i$ are binary, and in this case, we can directly solve $\max_{\pi \in \mathcal P(\bpi)}\return(\optpir_{S}, \pi_{-S})$. Namely, we can think of this optimization problem as searching for an optimal joint policy in an MMDP where the actions of agents $-S$ have  reduced ``influence''. Since an optimal joint policy in the reduced MMDP is deterministic, the optimal solution to $\max_{\pi \in \mathcal P(\bpi)}\return(\optpir_{S}, \pi_{-S})$ sets $\pi_j(a_j|s)$ of agent $j \in -S$ either to its maximum or its minimum value, $\pi_j^{bas}(a_j|s) + C$ and $\pi_j^{bas}(a_j|s) - C$ respectively. In the former case, this means that agent $j$ chooses $a_j$ in the MMDP with the reduced influence, in the latter, this means that agent $j$ chooses the other action. We use this fact in our experiments for the Graph environment.
\end{itemize}

To conclude, in our experiments we directly solve the optimization problems $\min_{\pi \in \mathcal P(\bpi)} \return(\optpir_{S \cup \{i\}}, \pi_{-S \cup \{i\}})$ and $\max_{\pi \in \mathcal P(\bpi)}\return(\optpir_{S}, \pi_{-S})$ for the Gridworld environment, and $\max_{\pi \in \mathcal P(\bpi)}\return(\optpir_{S}, \pi_{-S})$ for the Graph environment.

\subsection{Max-Efficient Rationality and Average Participation under Uncertainty}\label{sec.ap_mer_uncertainty}

In this section we discuss the implementation of Blackstone consistent $\widehat{\blamefunc}_{MER, BC}(M, \mathcal P(\bpi))$ and $\widehat{\blamefunc}_{AP, BC}(M, \mathcal P(\bpi))$ from Section \ref{sec.sv_under_uncertainty}. We begin with $\widehat{\blamefunc}_{MER, BC}(M, \mathcal P(\bpi))$, which can be obtained by solving the optimization problem \eqref{prob.rationality} with $\inef_S$ replaced by 
$\tilde \inef_S = \min_{\pi \in \mathcal P'(\bpi)} \return (\optpir_S,\pi_{-S}) - \max_{\pi \in \mathcal P'(\bpi)} \return(\pi)$.
A solution to this optimization problem $\widehat{\blame}$ will for at least one solution $\blame$ of \eqref{prob.rationality} (with $\inef_S$) satisfy $\widehat{\blamei} \le \blamei$ for all $i$. In that sense, $\widehat{\blamefunc}_{MER, BC}(M, \mathcal P(\bpi))$  satisfies $\prop_{BC}(\blamefunc_{MER})$ (Blackstone consistency w.r.t. $\blamefunc_{MER}(M, \bpi)$). However, note that $\prop_{BC}(\blamefunc_{MER})$ might not hold if \eqref{prob.rationality} has multiple solutions (e.g., when calculating $\widehat{\blamefunc}_{MER, BC}$ or $\blamefunc_{MER}$) and we consider only one solution  (e.g., obtained through a tie breaking rule).

Let us now consider $\widehat{\blamefunc}_{AP, BC}(M, \mathcal P(\bpi))$. $\widehat{\blame} = \widehat{\blamefunc}_{AP,BC}(M, \mathcal P(\bpi))$ can be implemented as
\begin{align*}
     \widehat{\blamei} = \sum_{S \subseteq \{1, ..., n\} \backslash \{i\}} \weight \cdot \frac{\tilde{c}(M, \mathcal P(\bpi), i)}{|S| + 1} \cdot \tilde \inef_{S \cup \{i\}}, 
\end{align*}
where $\tilde{c}(M, \mathcal P(\bpi), i) = 
\ind{\widehat{\blame}_{SV, i} > 0}$ with $\widehat{\blame}_{SV} = \widehat{\blamefunc}_{SV,BC}(M, \mathcal P(\bpi))$
(see Section \ref{sec.sv_under_uncertainty} for how to calculate $\widehat{\blamefunc}_{SV,BC}$), 
$\weight = \frac{1}{2^n - 1}$ and 
$\tilde \inef_{S \cup \{i\}} = \min_{\pi \in \mathcal P'(\bpi)} \return(\optpir_{S \cup \{i\}},\pi_{-S \cup \{i\}}) - \max_{\pi \in \mathcal P'(\bpi)} \return(\pi)$. 
Here, we used the fact that $c$ (in this case, estimate $\tilde c$) can be obtained via Shapley value (in this case, Blackstone consistent Shapley value).

%% file: 9.4_experimental_setup.tex
% !TEX root =  main.tex
%%%%%%%%%%%%%%%%%%%%%%%%%%%%%%%%%%%%% APPENDIX: EXPERIMENTAL SETUP
%%%%%%%%%%%%%%%%%%%%%%%%%%%%%%%%%%%%%

\section{Experimental Setup and Implementation Details}\label{sec.appendix_exp}
In this section, we provide additional information on experimental setup and implementation details.
\subsection{Additional Information on Experimental Setup}\label{sec.add_inf_exp}
\textbf{Environment 1:} The exact penalties and rewards of the Gridworld environment (Fig. \ref{fig: gridworld}) are as follows: $-0.01$ for blank cells and $S$ cells, $-0.02$ for $F$ cells, $-0.5$ for $H$ cells and $+1$ for cell $G$. Moreover, the cost of intervention $C$ is $-0.05$. The size of the environment's state space is $64$ (the state space represents cells of the Gridworld). The action space of agent \agenttwo~is $\{0, 1\}$, which corresponds to \textit{don't intervene} and \textit{intervene}, and the action space of agent \agentone~is $\{0, 1, 2, 3\}$, i.e. \textit{move left}, \textit{move right}, \textit{move up} and \textit{move down}. Note also that the actor remains at the same cell if it takes an action which would take it out of the environment. 
\iftoggle{submission}{\iftrue}{\iffalse}Finally, the specification of the personal policy of agent \agentone~can be found in the source code, and more precisely in function \verb|instantiate_behavior_policy_1()| of \verb|env_gridworld.py|.
\fi

\textbf{Environment 2:} The state space of the Graph environment (Fig. \ref{fig: graph}) is defined by possible distributions of the $4$ agents over the nodes of the graph, $66$ states in total. The action space of each agent is $\{0, 1\}$, and the time-horizon of the environment is $5$. We test multiple variants of this environment, each of which defines a different reward function. In all the variants, the reward at each time-step $t < 4$ is $+1$ if some formation constraint is satisfied and $-1$ if not, at time-step $t=4$ the reward is always $0$. 
Next, we describe in more detail the formation constraints and behavior policies for the Graph environment in the first (Coordination) and the second (Robustness) set of experiments.

\textbf{Coordination:} In the first set of experiments, we assign weights $w_1 = 1$, $w_2 = 2$, $w_3 = 3$ and $w_4 = 4$ to the four agents. We also consider $4$ different formation constraints which are satisfied if $\sum_{i \in \{1, 2, 3, 4\}} w_i \cdot a_i \geq h_m$, where $a_i$ is the action taken by agent $i$ and $h_m$ is a threshold specific to the constraint $m \in \{1, 2, 3, 4\}$.
We consider four thresholds: $h_1 = 1$, $h_2 = 7$, $h_3 = 9$ and $h_4 = 10$. For each constraint $m$ to be satisfied, at least $m$ number of agents need to select action $1$.
Each behavior $\bpi_i$ takes action $0$ in every state.

\textbf{Robustness:} In the second set of experiments, we consider one formation constraint that is satisfied if agents are arranged equally between the two levels of the graph, $\sum_{i \in \{1, 2, 3, 4\}} a_i = 2$. When the agents are in nodes $-1$, $6$, $7$ or $8$, each behavior policy $\bpi_i$ takes each action with $0.5$ probability. In states where agents are balanced between the levels, each behavior policy $\bpi_i$ takes the action from the previous time-step with probability $p_i$; in unbalanced states, the action that leads to the level with the least number of agents is taken with probability $p_i$. We consider $p_i = 1 - (i-1) \cdot 0.2$ for each agent $i \in \{1, 2, 3, 4\}$.

Discount factor $\gamma$ is set to $0.99$ in both environments.

\iftoggle{submission}{\iftrue}{\iffalse}
\subsection{Implementation Details}\label{sec.impl}

The solutions to the evaluation and optimization problems utilized by the blame attribution methods can be computed efficiently using standard (robust) optimization techniques.
In the source code, the solvers of these problems are implemented as the following functions:
\begin{itemize}
    \item policy performance evaluation in function \verb|recursion_1_a()| (see \verb|recursion_graph.py| and \verb|recursion_gridworld.py|).
    \item problem $\argmax_{\pi_S} \return(\pi_S,\bpi_{-S})$ in functions \verb|recursion_1_c()| (see \verb|recursion_graph.py|) and \verb|recursion_1_c_ag1()|, \verb|recursion_1_c_ag2()| (see \verb|recursion_gridworld.py|).
    \item problem $\argmax_{\pi \in \mathcal P'(\bpi)} \return(\pi)$ in function \verb|recursion_2_a()| (see \verb|recursion_graph.py| and \verb|recursion_gridworld.py|).
    \item problem $\min_{\pi \in \mathcal P'(\bpi)} \return(\optpir_{S \cup \{i\}}, \pi_{-S \cup \{i\}})$ in functions \verb|recursion_3_a()| (see \verb|recursion_graph.py|) and \verb|recursion_3_a_ag1()|, \verb|recursion_3_a_ag2()| (see \verb|recursion_gridworld.py|).
    \item problem $\max_{\pi \in \mathcal P'(\bpi)}\return(\optpir_{S}, \pi_{-S})$ in functions \verb|recursion_3_b()| (see \verb|recursion_graph.py|) and \verb|recursion_3_b_ag1()|, \verb|recursion_3_b_ag2()| (see \verb|recursion_gridworld.py|).
\end{itemize}
\fi

\subsection{Solutions to the Optimization Problem \eqref{prob.rationality}}

\eqref{prob.rationality} might have multiple optimal solutions. Therefore, when calculating $\blamefunc_{MER}$ (Section \ref{sec.efficient_rationality}) or $\widehat{\blamefunc}_{MER, BC}$ (Section \ref{sec.sv_under_uncertainty} and Appendix \ref{sec.app_uncertainty}), a way to decide which solution is going to be the blame assignment output is needed. For the experiments on the Gridworld environment the optimal solution assigning the maximum blame to \agenttwo~was always selected. For the experiments on the Graph environment, an LP solver was applied: in the case of the Graph environment, our experiments only require the total blame assigned to the agents so any optimal solution to the LP produces the same results (see below). 

\textbf{$\bm{L_1}$ Distance:} For the Max-Efficient Rationality method in Fig. \ref{fig: gridworld-distance} and \ref{fig: graph-distance} of Section \ref{sec.experiments}, we consider
the $L_1$ distance between an output $\widehat{\blame}$ of the consistent method $\widehat{\blamefunc}_{MER, BC}$ and an output $\blame$ of $\blamefunc_{MER}$, such that $\widehat{\blamei} \le \blamei$ for all $i$. Notice that the $L_1$ distance between any two such blame assignments is equal to their difference in total blame, $\sum_{i \in \{1, ..., n\}} |\blamei - \widehat{\blamei}| = \sum_{i \in \{1, ..., n\}}\blamei - \sum_{i \in \{1, ..., n\}}\widehat{\blamei}$. Notice also that the total blame $\sum_{i \in \{1, ..., n\}}\blamei$ (resp. $\sum_{i \in \{1, ..., n\}}\widehat{\blamei}$) is the same for all optimal solutions $\blame$ (resp. $\widehat{\blame}$) of \eqref{prob.rationality} with $\inef_S$  (resp. $\tilde \inef_S$), since they maximize the same objective. Hence, for obtaining the $L_1$ distance between the output of the consistent method $\widehat{\blamefunc}_{MER, BC}$ and its ``targeted assignment'', it suffices to compute the difference $\sum_{i \in \{1, ..., n\}}\blamei - \sum_{i \in \{1, ..., n\}}\widehat{\blamei}$ for any two optimal solutions $\blame$ and $\widehat{\blame}$.

\textbf{Total Blame:} The total blame assigned by the Max-Efficient Rationality method in each of the figures \ref{fig: graph-coordination}, \ref{fig: gridworld-total-blame} and \ref{fig: graph-total-blame} of Section \ref{sec.experiments} remains the same for all the optimal solutions of \eqref{prob.rationality}.

\subsection{Total Amount of Compute and Type of Resources}\label{sec.time}

All experiments were run on a personal laptop (with Intel Core i7-8750H CPU). Experiments were also run multiple times for $10$ different seeds, and we report averages and standard deviations. The total running time of the experiments on the Gridworld environment is a few minutes ($\sim$10) and of the  experiments on the Graph environment a few hours ($\sim$3).
Tables \ref{tab.time_grid} and \ref{tab.time_graph} show how much (CPU) time it takes to compute 
Shapley value under uncertainty (using the approaches from Section \ref{sec.blame_under_uncertainty}), for $\epsilon_{max} = \{0.01, 0.05, 0.1, 0.15, 0.2\}$. 
Note that $\blamefunc_{SV}$ does not depend on $\epsilon_{max}$---its running time for the Gridworld environment is $0.453125\pm0.19111$ sec and for the Graph environment is $2.02187\pm0.06853$ sec. 

\captionsetup[table]{skip=5pt}
\renewcommand{\arraystretch}{2}
\begin{table}[ht]
    \begin{center}
        %\resizebox{1\textwidth}{!}{
            \begin{tabular}{|c|c|c|c|c|}
                \hline
                 & $\widehat \blamefunc_{SV}$ & $\widehat \blamefunc_{SV,V}$ & $\widehat \blamefunc_{SV,BC}$\\
                \hline
                $\epsilon_{max} = 0.05$
                & $0.45625\pm0.19848$ & $1.19843\pm0.51044$ & $1.38906\pm0.60939$ \\
                \hline
                $\epsilon_{max} = 0.10$
                & $0.46093\pm0.20049$ & $1.21093\pm0.55609$ & $1.45781\pm0.71592$ \\
                \hline
                $\epsilon_{max} = 0.15$
                & $0.47187\pm0.20925$ & $1.14062\pm0.51864$ & $1.45468\pm0.66985$ \\
                \hline
                $\epsilon_{max} = 0.20$ 
                & $0.46093\pm0.22011$ & $1.20937\pm0.60934$ & $1.72500\pm0.96822$ \\
                \hline
            \end{tabular}
        %}
    \captionof{table}{Running times of different approaches for SV under uncertainty on the Gridworld environment. All times are measured in seconds (sec). 
    \label{tab.time_grid}}
    \end{center}
\end{table}
\renewcommand{\arraystretch}{1}

\captionsetup[table]{skip=5pt}
\renewcommand{\arraystretch}{2}
\begin{table}[ht]
    \begin{center}
        %\resizebox{1\textwidth}{!}{
            \begin{tabular}{|c|c|c|c|}
                \hline
                 & $\widehat \blamefunc_{SV}$ & $\widehat \blamefunc_{SV,V}$ & $\widehat \blamefunc_{SV,BC}$\\
                \hline
                $\epsilon_{max} = 0.01$ &
                $2.06250\pm0.10892$ & $3.80625\pm0.13243$ & $92.38750\pm1.59190$ \\
                \hline
                $\epsilon_{max} = 0.05$ &
                $2.07031\pm0.18077$ & $3.91718\pm0.18944$ & $91.60000\pm1.43320$ \\
                \hline
                $\epsilon_{max} = 0.10$ &
                $1.97500\pm0.05466$ & $3.84218\pm0.12798$ & $92.84375\pm3.45815$ \\
                \hline
            \end{tabular}
        %}
        \captionof{table}{Running times of different approaches for SV under uncertainty on the Graph environment. All times are measured in seconds (sec). 
        \label{tab.time_graph}
        }
    \end{center}
\end{table}
\renewcommand{\arraystretch}{1}

$\widehat \blamefunc_{SV,BC}$ has the largest computing time, while $\blamefunc_{SV}$ and $\widehat \blamefunc_{SV}$ have  the lowest computing times. These results are not surprising given that $\blamefunc_{SV}$ and $\widehat \blamefunc_{SV}$ only need to compute the values once and they are not running robust optimization. Moreover, $\widehat \blamefunc_{SV,BC}$ solves $\min_{\pi \in \mathcal P'(\bpi)} \return(\optpir_{S \cup \{ i \}}, \pi_{-S\cup\{i\}})$ and $\max_{\pi \in \mathcal P'(\bpi)}\return(\optpir_{S}, \pi_{-S})$ for each $S$ separately,
unlike $\widehat \blamefunc_{SV,V}$, which only requires robust optimization for finding a solution to the optimization problem $\argmax_{\pi \in \mathcal P'(\bpi)} \return(\pi)$.
The running times of methods that compute $\widehat \blamefunc_{SV}$ do not appear to have  strong dependency on $\epsilon_{max}$. This is expected for $\widehat \blamefunc_{SV}$ since it is based on point estimates, and does not use robust optimization.

Note that the computation results obtained when calculating  the aforementioned Shapley value blame assignments can be reused in computing the blame assignments of the other blame attribution methods, which we do in our experiments.

%% file: 9.5_discussion.tex
% !TEX root =  main.tex
%%%%%%%%%%%%%%%%%%%%%%%%%%%%%%%%%%%%% APPENDIX: Different Perspectives on Blame Attribution
%%%%%%%%%%%%%%%%%%%%%%%%%%%%%%%%%%%%%

\section{Extended Discussion}\label{sec.disc}
This section of the appendix discusses different perspectives on blame attribution, and the potential negative side-effects of under-blaming agents.

\subsection{Different Perspectives on Blame Attribution}

\textbf{Consequentialism:} In this paper we follow a \textit{consequentialist} \cite{blackburn2005oxford} approach to the blame attribution problem, in the sense that we consider the amount of an agent's blame to depend solely on the outcome of its policy. More specifically, we consider blame attribution methods and desirable properties that measure how good or bad an agent's policy is based only on the inefficiency it causes to the multi-agent system.\footnote{This is well-aligned with the main idea of \textit{utilitarianism} \cite{blackburn2005oxford}, which measures how good or bad an action is based only on the overall utility of its consequences.} A common objection to this type of approaches is that they do not blame an agent for violating common ground rules, i.e. they concentrate only on the ends rather than the means \cite{scheffler1988consequentialism}. For example, consider an intersection accident scenario that involves two drivers: the first driver, $D_1$, proceeds north and the second driver $D_2$ proceeds east, both of them drive below the speed limit. Assume that $D_2$ violates a stop sign but could not do anything different to avoid the accident, while if $D_1$ would drive above the speed limit then with high probability the accident would have been avoided. According to consequentialism, in this example driver $D_1$ deserves more blame than $D_2$, although $D_2$ is the one that breaks the law.
    
\textbf{Deontology:} Consequentialism is often contrasted to another major approach in normative ethics, \textit{deontology} \cite{blackburn2005oxford, kant2020groundwork}. From a deontological perspective, the quality of an agent's policy is based on how well it follows a clear set of rules or duties\footnote{Deontology takes root from the Greek word \textit{deon}, which means duty.}, rather than its consequences. Therefore, a deontological approach to blame attribution would assign more blame to the second driver, from the example above, because they violate a well-known traffic regulation. Of course, deontological approaches face criticism too, for instance people argue that deontological ethics are rigid---they focus on rules, ignoring the (potentially) severe consequences of one's behavior \cite{anscombe1958modern}. For instance, avoiding a car crash may be more important than not violating the speed limit in the example above.

The problem of assigning blame is inherently multi-dimensional and can be viewed through both deontological and consequentialist lenses (among others). In this paper we take a consequentialist viewpoint because it provides clear and practical guidance, at least when estimating (counterfactual) outcomes is plausible. However, we do not see the two normative ethical theories as mutually exclusive \cite{moore2010placing}, and thus our intention is not to replace deontological approaches, but to complement them.

\subsection{Under-Blaming Agents}

Apart from serving justice, blame attribution is also important for incentivizing decision makers to adopt policies that will minimize the system's inefficiency. To that end, we introduce in Section \ref{sec.properties} the performance monotonicity property, the purpose of which is to motivate agents to individually improve their policies. The second property we introduce, Blackstone consistency, aims to ensure that no agent will be over-blamed when the behavior policies are not fully known to the blame attribution procedure. As expected, experimental results from Section \ref{sec.experiments} show that Blackstone consistent methods end up under-blaming agents instead. Just like over-blaming, under-blaming has its own adverse effects. Such an effect is incentivizing bad behaviors, since the agents receive reduced penalties. Therefore, there seems to be a trade-off between ensuring that no one is unjustly blamed under uncertainty and providing incentives for good behavior.

%% file: 9.6_proofs.tex
% !TEX root =  main.tex
%%%%%%%%%%%%%%%%%%%%%%%%%%%%%%%%%%%%% APPENDIX: PROOFS
%%%%%%%%%%%%%%%%%%%%%%%%%%%%%%%%%%%%%

\section{Proofs of the Propositions from Section \ref{sec.approaches_to_blame}}\label{sec.prop_proofs_123}
This section of the appendix contains the proofs of the propositions from Section \ref{sec.approaches_to_blame}, in particular: Proposition \ref{prop.mer_properties}, Proposition \ref{prop.mc_properties}, and Proposition \ref{prop.imposs}. 

\subsection{Proof of Proposition \ref{prop.mer_properties}}

\addtocounter{proposition}{-6}
\begin{proposition}
Every solution to the optimization problem \eqref{prob.rationality}, i.e., $\blamefunc_{MER}$, satisfies $\prop_V$ (validity), $\prop_R$ (rationality) and $\prop_I$ (invariance).
\end{proposition}
\begin{proof}
We prove the properties as follows:
\begin{itemize}
    \item \underline{$\prop_V$ (validity):} Consider $M$, $\bpi$. Every solution to the optimization problem \eqref{prob.rationality}, i.e., $\blame = \blamefunc_{MER}(M, \bpi)$, satisfies the constraint $\sum_{i\in \{1, ..., n\}} \blamei \le \inef_{\{1, ..., n\}}$. The last inequality can be rewritten as $\sum_{i = 1}^n \blamei \le \inef$, and hence property $\prop_V$ (validity) is satisfied.
    \item \underline{$\prop_R$ (rationality):} Consider $M$, $\bpi$ and $S \subseteq \{1, ..., n\}$. Every solution to the optimization problem \eqref{prob.rationality}, i.e., $\blame = \blamefunc_{MER}(M, \bpi)$, satisfies the constraint $\sum_{i\in S} \blamei \le \inef_S$, and hence property $\prop_R$ (rationality) is satisfied.
    \item \underline{$\prop_I$ (invariance):} Consider $M$, $\bpi$, and an agent $i$ 
    such that $\inef_{S \cup \{i \}} = \inef_{S}$ for all $S$. This implies that $\inef_i = \inef_{\emptyset} = 0$.
    Now, due to the constraints of the optimization problem \eqref{prob.rationality}, every solution to the optimization problem \eqref{prob.rationality}, i.e., $\blame = \blamefunc_{MER}(M, \bpi)$, satisfies the constraint $\blamei \le \inef_i = 0$. Note also that $\sum_{j \in S} \beta_j \le \inef_{S \cup \{i \}} - \beta_i = \inef_{S} - \beta_i$, but also $\sum_{j \in S} \beta_j \le \inef_{S}$ (where $i \notin S$). Therefore, the constraints in which agent $i$ participates can be replaced by the the constraint $\blamei \le 0$. Together with the fact that the objective function is the total blame, this implies that the optimal $\blamei$ is independent of $\blamej$ ($j \ne i$), and furthermore that its value is equal to $\blamei = 0$.
    Hence, property $\prop_I$ (invariance) is satisfied.
\end{itemize}
\end{proof}

\subsection{Proof of Proposition \ref{prop.mc_properties}}

\begin{proposition}
$\blamefunc_{MC}(M, \bpi) = (\inef_1, ..., \inef_n)$ satisfies $\prop_S$ (symmetry), $\prop_I$ (invariance), $\prop_{CM}$ (contribution monotonicity) and $\prop_{PerM}$ (performance monotonicity).
\end{proposition}
\begin{proof}
We prove the properties as follows:
\begin{itemize}
    \item \underline{$\prop_S$ (symmetry):} Consider $M$, $\bpi$, and agents $i$ and $j$ such that  $\inef_{S \cup \{i \}} = \inef_{S \cup \{j \}}$ for all $S \subseteq \{1, ..., n\} \backslash \{i, j\}$. Notice that $\inef_i = \inef_j$.  By using the definition of $\blame = \blamefunc_{MC}(M, \bpi)$, we have that $\blamei = \inef_i = \inef_j = \blamej$. Hence, property $\prop_S$ (symmetry) is satisfied.
    \item \underline{$\prop_I$ (invariance):} Consider $M$, $\bpi$, and agent $i$
    such that $\inef_{S \cup \{i \}} = \inef_{S}$ for all $S$. Given the definition of $\blame = \blamefunc_{MC}(M, \bpi)$, this implies that $\blamei = \inef_i = \inef_{\emptyset} = 0$. Hence, property $\prop_I$ (invariance) is satisfied.
    \item \underline{$\prop_{CM}$ (contribution monotonicity):} Consider $M^1$, $\bpi^1$, $M^2$, $\bpi^2$, and agent $i$ such that $\inef^1_{S \cup \{i \}} - \inef^1_S \geq \inef^2_{S \cup \{i \}} - \inef^2_S$ for all $S$. By using the definitions of $\blame^1 = \blamefunc_{MC}(M^1, \bpi^1)$ and $\blame^2 = \blamefunc_{MC}(M^2, \bpi^2)$, we have that $\blamei^1 = \inef^1_i = \inef^1_{\emptyset \cup \{ i \}} - \inef^1_{\emptyset} \geq \inef^2_{\emptyset \cup \{ i \}} - \inef^2_{\emptyset}= \inef^2_i = \blamei^2$. Hence, property $\prop_{CM}$ (contribution monotonicity) is satisfied.
    \item \underline{$\prop_{PerM}$ (performance monotonicity):} Consider $M$, $\bpi_{-i}$, $\pi_i$ and $\pi_i'$ such that $\return(\pi_i,\bpi_{-i}) \le \return(\pi_i',\bpi_{-i})$. This implies that:
    \begin{align*}
        & \return(\pi_i,\bpi_{-i}) \le \return(\pi_i',\bpi_{-i}) \Rightarrow\\
        \Rightarrow&  \return(\optpirb_i,\bpi_{-i}) - \return(\pi_i,\bpi_{-i}) \ge \return(\optpirb_i,\bpi_{-i}) - \return(\pi_i',\bpi_{-i}) \Rightarrow\\
        \Rightarrow& \inef_i \ge \inef'_i.
    \end{align*}
    By using the definitions of $\blame = \blamefunc_{MC}(M, (\pi_i,\bpi_{-i}))$ and $\blame' = \blamefunc_{MC}(M, (\pi_i',\bpi_{-i}))$, we obtain that $\blamei = \inef_i \ge \inef'_i = \blamei'$. Hence, property $\prop_{PerM}$ (performance monotonicity) is satisfied.
\end{itemize}
\end{proof}

\subsection{Proof of Proposition \ref{prop.imposs}}

\begin{proposition}
No blame attribution method $\blamefunc$ satisfies $\prop_E$ (efficiency), $\prop_S$ (symmetry), $\prop_I$ (invariance) and $\prop_{PerM}$ (performance monotonicity).
\end{proposition}
\begin{proof}
We prove the stated impossibility result by contradiction. Suppose that there is a blame attribution method $\blamefunc$ that satisfies $\prop_E$ (efficiency), $\prop_S$ (symmetry), $\prop_I$ (invariance) and $\prop_{PerM}$ (performance monotonicity). 

Consider an MMDP $M$ with two agents $\{1, 2\}$, two states---the initial state and the terminal state---and the action space $\mathcal{A} = \{0, 1, 2\} \times \{0, 1, 2\}$. In the initial state,
the agents obtain zero reward when they both take action $0$, reward equal to $2$ when one of them takes action $0$ and the other one action $2$ or they both take action $2$, and reward equal to $0.9$ when they take any other pair of actions. After the agents perform their actions in the initial state, the MMDP transitions to the terminal state. Consider also the deterministic policies: $\bpi_2$ that takes action $0$, $\pi_1$ that takes action $0$ and $\pi'_1$ that takes action $1$, in the initial state. 

We have the following three observations:
\begin{itemize}
\item Note that $\return(\pi_1,\bpi_2) \le \return(\pi_1',\bpi_2)$ and hence from property $\prop_{PerM}$ (performance monotonicity) we have that $\blame_1 \ge \blame'_1$, where $\blame = \blamefunc(M, (\pi_1,\bpi_2))$ and $\blame' = \blamefunc(M, (\pi_1', \bpi_2))$.

\item Note that $\inef_{\{1\}} = \inef_{\{2\}} = 2$ and thus from property $\prop_S$ (symmetry) it follows that $\blame_1 = \blame_2$. Also, from property $\prop_E$ (efficiency) we have that $\blame_1 + \blame_2 = \inef = 2$, and hence $\blame_1 = 1$ and $\blame_2 = 1$.

\item Note that $\inef'_{\{2\}} = 0$ and $\inef'_{\{1, 2\}} = \inef'_{\{1\}} = 1.1$ and thus from property $\prop_I$ (invariance) it follows that $\blame'_2 = 0$. From property $\prop_E$ (efficiency) we have that $\blame'_1 + \blame'_2 = \inef' = 1.1$ and hence $\blame'_1 = 1.1$, which contradicts the first two observations.
\end{itemize}
\end{proof}

\section{Proof of Theorem \ref{thm.shapley_value.unique}}\label{sec.proof_sv}

In this section, we provide a proof of Theorem \ref{thm.shapley_value.unique}. Since this proof utilizes the results of \cite{pinter2015young}, we first provide some background details on these results.

\subsection{Background}\label{sec.background_proof_sv}

To prove the uniqueness result for the Shapley Value method, Theorem \ref{thm.shapley_value.unique}, we use a result from \cite{pinter2015young}. Before we embark on the proof, we set the necessary background. Let $N$ be a set, such that $N \neq \emptyset$, $|N| < \infty$, and $u: 2^N \rightarrow \mathbb{R}$ be a function such that $u(\emptyset) = 0$. Then we call $N$ set of agents and $u$ game, and denote with $\mathcal{G}^N$ the class of games with player set $N$. We say that a game $u \in \mathcal{G}^N$ is \textit{monotone}, if for each $S, T \subseteq N$, $S \subseteq T$; $u(S) \le u(T)$. Moreover, we say that function $\psi:G \rightarrow \mathbb{R}^N$ is a solution on the class $G \in \mathcal{G}^N$. Next, we state three axioms from \cite{pinter2015young}:
\begin{itemize}
    \item {\em Pareto Optimality (PO)}: We say that a solution $\psi$ on class of games $G \subseteq \mathcal{G}^N$ satisfies \textit{PO} (Pareto optimality), if for each game $u \in G$: $\sum_{i \in N} \psi_i(u) = u(N)$.
    \item {\em Equal Treatment Property (ETP)}: We say that a solution $\psi$ on class of games $G \subseteq \mathcal{G}^N$ satisfies \textit{ETP} (equal treatment property), if for each game $u \in G$ and $i,j \in N$; $\psi_i(u) = \psi_j(u)$, whenever $u(S\cup\{i\}) - u(S) = u(S\cup\{j\}) - u(S)$ for every $S \subseteq N \backslash \{i,j\}$.
    \item {\em Marginality (M)}: We say that a solution $\psi$ on class of games $G \subseteq \mathcal{G}^N$ satisfies \textit{M} (marginality), if for all games $u, v \in G$ and $i \in N$: $\psi_i(u) = \psi_i(v)$, whenever $u(S\cup\{i\}) - u(S) = v(S\cup\{i\}) - v(S)$ for every $S \subseteq N$.
\end{itemize}
We also define the Shapley value method for this setting. For any game $u \in \mathcal{G}^N$, the Shapley value solution $\phi$ is given by
\begin{align}\label{eq.def_sv_coal}
     \phi_i(u) = \sum_{S \subseteq N \backslash \{i\}} \weight_S \cdot \left [ u(S\cup\{i\}) - u(S) \right ],   
\end{align}
where coefficients $\weight_S$ are set to $\weight_S = \frac{|S|! ( |N| - |S| - 1)!}{|N|!}$.

Next we restate Theorem 3.9 from \cite{pinter2015young}:
\begin{theorem}\label{thm.pinter_sv_uniqueness}
Solution $\psi$ defined on the class of monotone games satisfies axiom \textit{PO} (Pareto optimality), \textit{ETP} (equal treatment Property) and \textit{M} (marginality), iff it is the Shapley value solution.
\end{theorem}
We introduce a slightly different axiom than \textit{M} (marginality):
\begin{itemize}
    \item {\em Unequal Marginality (UM)}: We say that a solution $\psi$ on class of games $G \subseteq \mathcal{G}^N$ satisfies \textit{UM} (unequal marginality), if for all games $u, v \in G$ and $i \in N$: $\psi_i(u) \geq \psi_i(v)$, whenever $u(S\cup\{i\}) - u(S) \geq v(S\cup\{i\}) - v(S)$ for every $S \subseteq N$.
\end{itemize}
We also state a Corollary of Theorem \ref{thm.pinter_sv_uniqueness}:
\begin{corollary}\label{cor.pinter_sv_uniqueness}
Solution $\psi$ defined on the class of monotone games satisfies axiom \textit{PO} (Pareto optimality), \textit{ETP} (equal treatment Property) and \textit{UM} (unequal marginality), iff it is the Shapley value solution.
\end{corollary}
\begin{proof}
We prove that Shapley value solution $\phi$ satisfies axiom \textit{UM} (unequal marginality). Consider monotone games $u$, $v$, and agent $i \in N$ such that $u(S\cup\{i\}) - u(S) \geq v(S\cup\{i\}) - v(S)$ for every $S \subseteq N$, then;
\begin{align*}
    \phi_i(u) =& \sum_{S \subseteq N \backslash \{i\}} \weight_S \cdot \left [ u(S\cup\{i\}) - u(S) \right ] \geq \\
    \geq& \sum_{S \subseteq N \backslash \{i\}} \weight_S \cdot \left [ v(S\cup\{i\}) - v(S) \right ] = \\
    =& \phi_i(v).
\end{align*}
Since \textit{UM} (unequal marginality) is a stronger axiom than \textit{M} (marginality), and Shapley value solution satisfies it, the uniqueness result stated in the Corollary holds because of Theorem \ref{thm.pinter_sv_uniqueness}.
\end{proof}

Consider $M$, $\bpi$ and notice that $\inef_\emptyset = \return(\bpi) - \return(\bpi) = 0$. We say that set of agents $N$ and game $u$ are defined by $M$, $\bpi$, if $N = \{1, ..., n\}$ and $u(S) = \inef_S$ for every $S$. We denote with $\mathcal{H}$ the class of games that can be defined in that way. 
Let $\psi_{SV}$ be the solution on class $\mathcal{H}$ such that for every $M$, $\bpi$, $\psi_{SV}(u) = \blamefunc_{SV}(M, \bpi)$, where game $u$ is defined by $M$, $\bpi$. Given Eq. \eqref{eq.def_sv_coal}, this implies that $\psi_{SV}$ is the Shapley Value solution on $\mathcal{H}$.

We state three simple lemmas that show a one to one correspondence between the axioms \textit{PO} (Pareto optimality), \textit{ETP} (equal treatment property) and \textit{UM} (unequal marginality) and blame attribution properties:

\begin{lemma}\label{lem.po}
Let $\blamefunc$ be a blame attribution method and $\psi$ a solution on $\mathcal{H}$, such that for every $M$, $\bpi$, $\blamefunc(M, \bpi) = \psi(u)$, where game $u$ is defined by $M$, $\bpi$. Then, $\blamefunc$ satisfies $\prop_{E}$ (efficiency) iff $\psi$ satisfies \textit{PO} (Pareto optimality) on $\mathcal{H}$.
\end{lemma}
\begin{proof}
Consider $M$, $\bpi$ and game $u$ defined by $M$, $\bpi$. Then the statement is true because $u(N) = \inef_{\{1, ..., n\}} = \inef$.
\end{proof}

\begin{lemma}\label{lem.etp}
Let $\blamefunc$ be a blame attribution method and $\psi$ a solution on $\mathcal{H}$, such that for every $M$, $\bpi$, $\blamefunc(M, \bpi) = \psi(u)$, where game $u$ is defined by $M$, $\bpi$. Then, $\blamefunc$ satisfies $\prop_{S}$ (symmetry) iff $\psi$ satisfies \textit{ETP} (equal treatment property) on $\mathcal{H}$.
\end{lemma}
\begin{proof}
Consider $M$, $\bpi$ and game $u$ defined by $M$, $\bpi$. Given that $u(S) = \inef_S$ for every $S$, we have that for every $i$ and $j$, $\inef_{S \cup \{i\}} - \inef_S = \inef_{S \cup \{j\}} - \inef_S$ iff $u(S \cup \{i\}) - u(S) = u(S \cup \{j\}) - u(S)$. Hence, the statement is true.
\end{proof}

\begin{lemma}\label{lem.um}
Let $\blamefunc$ be a blame attribution method and $\psi$ a solution on $\mathcal{H}$, such that for every $M$, $\bpi$, $\blamefunc(M, \bpi) = \psi(u)$, where game $u$ is defined by $M$, $\bpi$. Then, $\blamefunc$ satisfies $\prop_{CM}$ (contribution monotonicity) iff $\psi$ satisfies \textit{UM} (unequal marginality) on $\mathcal{H}$.
\end{lemma}
\begin{proof}
Consider $M^1$, $\bpi^1$ and $M^2$, $\bpi^2$, and games $u^1$ and $u^2$ defined by $M^1$, $\bpi^1$ and $M^2$, $\bpi^2$, respectively. Given that $u^1(S) = \inef^1_S$ and $u^2(S) = \inef^2_S$ for every $S$, we have that for every $i$, $\inef^1_{S\cup\{i\}} - \inef^1_{S} \geq \inef^2_{S\cup\{i\}} - \inef^2_{S}$ iff $u^1(S\cup\{i\}) - u^1(S) \geq u^2(S\cup\{i\}) - u^2(S)$. Hence, the statement is true.
\end{proof}

\subsection{Proof}\label{sec.proof_sv_after_background}

\addtocounter{theorem}{-4}
\begin{theorem}
$\blamefunc_{SV}(M, \bpi) = (\beta_1, ..., \beta_n)$, where $\beta_i$ is defined by Eq. \eqref{eq.def_sv} and $\weight_S = \frac{|S|! ( n - |S| - 1)!}{n!}$,  is a unique blame attribution method satisfying $\prop_E$ (efficiency), $\prop_S$ (symmetry) and $\prop_{CM}$ (contribution monotonicity). Additionally, $\blamefunc_{SV}$ satisfies $\prop_V$ (validity) and $\prop_I$ (invariance).
\end{theorem}
\begin{proof}
Consider $M$, $\bpi$ and game $u$ defined by $M$, $\bpi$. Consider also $S$ and $T$ such that $S \subseteq T$. We have that:
\begin{align*}
    & \return(\optpirb_T,\bpi_{-T}) \geq \return(\optpirb_S,\bpi_{-S}) \Rightarrow\\
    \Rightarrow& \return(\optpirb_T,\bpi_{-T}) - \return(\bpi) \geq \return(\optpirb_S,\bpi_{-S}) - \return(\bpi) \Rightarrow\\
    \Rightarrow& \inef_T \geq \inef_S \Rightarrow u(T) \geq u(S).
\end{align*}
This implies that class $\mathcal{H}$ consists only of monotone games, and hence by Corollary \ref{cor.pinter_sv_uniqueness} we have that $\psi_{SV}$ is a unique solution on $\mathcal{H}$ satisfying \textit{PO} (Pareto optimality), \textit{ETP} (equal treatment property) and \textit{UM} (unequal marginality). Given Lemmas \ref{lem.po}, \ref{lem.etp}, and \ref{lem.um}, this implies that $\blamefunc_{SV}$ is a unique blame attribution method satisfying $\prop_E$ (efficiency), $\prop_S$ (symmetry) and $\prop_{CM}$ (contribution monotonicity).

We also prove the properties $\prop_V$ (validity) and $\prop_I$ (invariance) as follows:
\begin{itemize}
    \item  \underline{$\prop_V$ (validity):} Consider $M$, $\bpi$. Given that $\blamefunc_{SV}$ satisfies property $\prop_E$ (efficiency), it holds that $\sum_{i \in \{1, ..., n\}} \blame_i = \inef$. Hence, property $\prop_V$ (validity) is satisfied.
    \item \underline{$\prop_I$ (invariance):} Consider $M$, $\bpi$, and agent $i$ such that $\inef_{S \cup \{i \}} = \inef_{S}$ for all $S$. This implies that $\return(\optpirb_{S \cup \{ i \}}, \bpi_{-S\cup\{i\}}) = \return(\optpirb_{S}, \bpi_{-S})$ for all $S$. Given the definition of $\blamefunc_{SV}(M, \bpi)$, we have that $\blame_i = 0$. Hence, property $\prop_I$ (invariance) is satisfied.
\end{itemize}
\end{proof}

\section{Proof of Theorem \ref{thm.average_participation.unique}}\label{sec.proof_ap}

Before we proceed with the proof of Theorem \ref{thm.average_participation.unique}, notice that the contribution function $c$ from Section \ref{sec.avg_participation} can be rewritten in the equivalent form:
\begin{align*}
    c(M, \bpi, i) = \begin{cases}
            0  \quad &\mbox{ if } \inef_{S \cup \{ i \}} = \inef_S, \quad \forall S \subseteq \{1, ..., n\} \\ 
            1 \quad &\mbox{ otherwise  }
    \end{cases}.
\end{align*} 

We also state the following lemmas:

\begin{lemma}\label{lem.inef_existence}
Consider a function $f: 2^{\{1, ..., n\}} \rightarrow \mathbb{R}_{\ge 0}$. There exist some MMDP $M$ and agents' behavior joint policy $\bpi$ such that the marginal inefficiency of every subset of agents $S$ is equal to $f(S)$, iff $f(\emptyset) = 0$ and $f(S_1) \le f(S_2)$ whenever $S_1 \subseteq S_2$, where $S_1$ and $S_2$ are subsets of $\{1, ..., n\}$.
\end{lemma}

\begin{proof}
First, we show that the conditions on function $f$ are necessary:
\begin{itemize}
    \item Suppose that there exist $M$, $\bpi$ such that $\inef_{\emptyset} > 0$. Given the definition of marginal inefficiency this would imply that $\return(\bpi) > \return(\bpi)$. Hence, we reach a contradiction.
    \item Suppose that there exist $M$, $\bpi$ such that $\inef_{S_1} > \inef_{S_2}$, where $S1 \subseteq S_2$. Given the definition of marginal inefficiency this would imply that $\return(\optpirb_{S_1}, \bpi_{- S_1}) > \return(\optpirb_{S_2}, \bpi_{- S_2})$. Hence, we reach a contradiction.
\end{itemize}

Next we show that the conditions on function $f$ are sufficient. Consider an MMDP $M$ with two states---the initial state and the terminal state---and the action space $\mathcal{A} = \times_{i = 1}^n \{0, 1\}$. In the initial state, the agents obtain zero reward when they all take action $0$ and reward $f(S)$ when agents in $S$ take action $1$ and the rest of the agents take action $0$. Consider also the deterministic joint policy $\bpi$, where every agent takes action $0$. Notice that $\return(\bpi) = 0$.

For every subset of agents $S$ it holds that $\return(\optpirb_S, \bpi_{- S}) = f(S)$, because taking action $1$ is the best that every agent in $S$ can do. Hence, for the marginal inefficiency of $S$ we have that $\inef_S = \return(\optpirb_S, \bpi_{- S}) - \return(\bpi) = f(S)$.
\end{proof}

\begin{lemma}\label{lem.cparm}
Let $\blamefunc$ satisfy $\prop_{cParM}$ (c-participation monotonicity). Then, for every $M^1$, $\bpi^1$ and $M^2$, $\bpi^2$ such that $c(M^1, \bpi^1, i) = c(M^2, \bpi^2, i)$ for every $i$, $\blamei^1 = \blamei^2$ whenever $\inef^1_{S \cup \{i \}} = \inef^2_{S \cup \{i \}}$ for all $S$, where $\blame^1 = \blamefunc(M^1, \bpi^1)$ and $\blame^2 = \blamefunc(M^2, \bpi^2)$.
\end{lemma}

\begin{proof}
Consider agent $i$ such that $\inef^1_{S \cup \{i \}} = \inef^2_{S \cup \{i \}}$ for all $S$. Given that $\blamefunc$ satisfies $\prop_{cParM}$ (c-participation monotonicity), this implies $\blamei^1 \ge \blamei^2$ and $\blamei^1 \le \blamei^2$, and hence $\blamei^1 = \blamei^2$.
\end{proof}

\begin{lemma}\label{lem.rcparm}
Let $\blamefunc$ satisfy $\prop_{RcParM}$ (relative c-participation monotonicity). Then, for every $M^1$, $\bpi^1$ and $M^2$, $\bpi^2$ such that $c(M^1, \bpi^1, i) = c(M^2, \bpi^2, i)$  for every $i$, $\blamej^1 - \blamej^2 = \blamek^1 - \blamek^2$ whenever $c(M^1, \bpi^1, j) = c(M^1, \bpi^1, k)$ and $\inef^1_{S \cup \{j \}} - \inef^2_{S \cup \{j \}} = \inef^1_{S \cup \{k \}} - \inef^2_{S \cup \{k \}}$ for every $S \subseteq \{1, ..., n\} \backslash \{j, k\}$, where $\blame^1 = \blamefunc(M^1, \bpi^1)$ and $\blame^2 = \blamefunc(M^2, \bpi^2)$.
\end{lemma}

\begin{proof}
Consider agents $j$ and $k$ such that $c(M^1, \bpi^1, j) = c(M^1, \bpi^1, k)$ and $\inef^1_{S \cup \{j \}} - \inef^2_{S \cup \{j \}} = \inef^1_{S \cup \{k \}} - \inef^2_{S \cup \{k \}}$ for all $S \subseteq \{1, ..., n\} \backslash \{j, k\}$. Given that $\blamefunc$ satisfies $\prop_{RcParM}$ (relative c-participation monotonicity), this implies $\blamej^1 - \blamej^2 \ge \blamek^1 - \blamek^2$ and $\blamej^1 - \blamej^2 \le \blamek^1 - \blamek^2$, and hence $\blamej^1 - \blamej^2 = \blamek^1 - \blamek^2$.
\end{proof}

\subsubsubsection{\textbf{Proof of Theorem \ref{thm.average_participation.unique}}}

\begin{theorem}
$\blamefunc_{AP}(M, \bpi) = (\beta_1, ..., \beta_n)$, where $\beta_i$ is defined by Eq. \eqref{eq.def_avg_part} and $\weight = \frac{1}{2^n - 1}$, is a unique blame attribution method that satisfies $\prop_{AE}$ (average-efficiency), $\prop_{S}$ (symmetry), $\prop_{I}$ (invariance), $\prop_{cParM}$ (c-participation monotonicity) and $\prop_{RcParM}$ (relative c-participation monotonicity). Furthermore, $\blamefunc_{AP}$ satisfies $\prop_{cPerM}$ (c-performance monotonicity) and $\prop_{V}$ (validity).
\end{theorem}

\begin{proof}
The proof is separated into two parts. In the first part we prove that $\blamefunc_{AP}$ satisfies the mentioned properties, while in the second part we show that if a blame attribution method satisfies all mentioned properties, it must be the $\blamefunc_{AP}$ method.

\subsubsubsection{\textbf{First Part}}

We prove the properties as follows:
\begin{itemize}
    \item \underline{$\prop_{AE}$ (average-efficiency):} 
    Consider $M$, $\bpi$. By using the definition of $\blame = \blamefunc_{AP}(M, \bpi)$:
    \begin{align*}
        \sum_{i = 1}^n \blamei &= 
        \sum_{i = 1}^n\sum_{S \subseteq \{1, ..., n\} \backslash \{i\}} \weight \cdot \frac{c(M, \bpi, i)}{\sum_{j \in S} c(M, \bpi, j) + 1} \cdot \inef_{S \cup \{ i \}} = \\
        &= \frac{1}{2^n - 1} \cdot \sum_{i \in \{1, ..., n \} | c(M, \bpi, i) = 1} \sum_{S \subseteq \{1, ..., n\} \backslash \{i\}} \frac{1}{\sum_{j \in S} c(M, \bpi, j) + 1} \cdot \inef_{S \cup \{ i \}} =  \\
        &= \frac{1}{2^n - 1} \cdot \sum_{i \in \{1, ..., n \} | c(M, \bpi, i) = 1} \sum_{S \subseteq \{1, ..., n\} | i \in S} \frac{1}{\sum_{j \in S} c(M, \bpi, j)} \cdot \inef_S =  \\
        &= \frac{1}{2^n - 1} \cdot \sum_{S \subseteq \{1, ..., n\}} \sum_{i \in S | c(M, \bpi, i) = 1} \frac{1}{\sum_{j \in S} c(M, \bpi, j)} \cdot \inef_S = \\
        &= \frac{1}{2^n - 1} \cdot \sum_{S \subseteq \{1, ..., n\}} \inef_S,
    \end{align*}
    and hence property $\prop_{AE}$ (average-efficiency) is satisfied.
    \item \underline{$\prop_{S}$ (symmetry):} 
    Consider $M$, $\bpi$, and agents $i$ and $j$ such that  $\inef_{S \cup \{i \}} = \inef_{S \cup \{j \}}$ for all $S \subseteq \{1, ..., n\} \backslash \{i, j\}$. Notice that if $\inef_{S \cup \{i\}} = \inef_S$ for all $S$ then $\inef_{S \cup \{j\}} = \inef_S$ and $\inef_{S \cup \{i, j\}} = \inef_{S \cup \{j\}} = \inef_{S \cup \{i\}}$ for every $S \subseteq \{1, ..., n\} \backslash \{i, j\}$, and hence $\inef_{S \cup \{j\}} = \inef_S$ for all $S$. Given the definition of contribution function $c$, this implies that if $c(M, \bpi, i) =  0$, then $c(M, \bpi, j) = 0$. For similar reasons, it also holds that if $c(M, \bpi, j) =  0$, then $c(M, \bpi, i) = 0$, and hence $c(M, \bpi, i) = c(M, \bpi, j)$.
    By using the definition of $\blame = \blamefunc_{AP}(M, \bpi)$, we have that:
    \begin{align*}
        \blamei =& 
        \sum_{S \subseteq \{1, ..., n\} \backslash \{i\}} \weight \cdot \frac{c(M, \bpi, i)}{\sum_{k \in S} c(M, \bpi, k) + 1} \cdot \inef_{S \cup \{ i \}} = \\
        =& \sum_{S \subseteq \{1, ..., n\} \backslash \{i, j\}} \weight \cdot \frac{c(M, \bpi, i)}{\sum_{k \in S} c(M, \bpi, k) + 1} \cdot \inef_{S \cup \{ i \}} + \\
        &+ \sum_{S \subseteq \{1, ..., n\} \backslash \{i, j\}} \weight \cdot \frac{c(M, \bpi, i)}{\sum_{k \in S \cup \{ j \}} c(M, \bpi, k) + 1} \cdot \inef_{S \cup \{ i, j \}} =\\
        =& \sum_{S \subseteq \{1, ..., n\} \backslash \{i, j\}} \weight \cdot \frac{c(M, \bpi, j)}{\sum_{k \in S} c(M, \bpi, k) + 1} \cdot \inef_{S \cup \{ j \}} + \\
        &+ \sum_{S \subseteq \{1, ..., n\} \backslash \{i, j\}} \weight \cdot \frac{c(M, \bpi, j)}{\sum_{k \in S \cup \{ i \}} c(M, \bpi, k) + 1} \cdot \inef_{S \cup \{ i, j \}} =\\
        =& \sum_{S \subseteq \{1, ..., n\} \backslash \{j\}} \weight \cdot \frac{c(M, \bpi, j)}{\sum_{k \in S} c(M, \bpi, k) + 1} \cdot \inef_{S \cup \{ j \}} = \blamej,
    \end{align*}
    and hence property $\prop_{S}$ (symmetry) is satisfied.
    \item \underline{$\prop_{I}$ (invariance):} Consider $M$, $\bpi$, and agent $i$
    such that $\inef_{S \cup \{i \}} = \inef_{S}$ for all $S$. Given the definitions of contribution function $c$ and $\blame = \blamefunc_{AP}(M, \bpi)$, this implies that $\blame_i = 0$. Hence, property $\prop_I$ (invariance) is satisfied.
    \item \underline{$\prop_{cParM}$ (c-participation monotonicity):} 
    Consider $M^1$, $\bpi^1$ and $M^2$, $\bpi^2$ such that $c(M^1, \bpi^1, i) = c(M^2, \bpi^2, i)$ for every $i$. Consider also agent $i$ such that $\inef^1_{S \cup \{i\}} \ge \inef^2_{S \cup \{i\}}$ for all $S$. By using the definitions of $\blame^1 = \blamefunc_{AP}(M^1, \bpi^1)$ and $\blame^2 = \blamefunc_{AP}(M^2, \bpi^2)$, this implies:
    \begin{align*}
        \blame^1_i =& \sum_{S \subseteq \{1, ..., n\} \backslash \{i\}} \weight \cdot \frac{c(M^1, \bpi^1, i)}{\sum_{j \in S} c(M^1, \bpi^1, j) + 1} \cdot \inef^1_{S \cup \{ i \}} =\\
        =& \sum_{S \subseteq \{1, ..., n\} \backslash \{i\}} \weight \cdot \frac{c(M^2, \bpi^2, i)}{\sum_{j \in S} c(M^2, \bpi^2, j) + 1} \cdot \inef^1_{S \cup \{ i \}} \ge \\
        \ge& \sum_{S \subseteq \{1, ..., n\} \backslash \{i\}} \weight \cdot \frac{c(M^2, \bpi^2, i)}{\sum_{j \in S} c(M^2, \bpi^2, j) + 1} \cdot \inef^2_{S \cup \{ i \}} = \blame^2_i,
    \end{align*}
    and hence property $\prop_{cParM}$ (c-participation monotonicity) is satisfied.
    \item \underline{$\prop_{RcParM}$ (relative c-participation monotonicity):} 
    Consider $M^1$, $\bpi^1$ and $M^2$, $\bpi^2$ such that $c(M^1, \bpi^1, i) = c(M^2, \bpi^2, i)$ for every $i$. Consider also agents $j$ and $k$ such that $c(M^1, \bpi^1, j) = c(M^1, \bpi^1, k)$ and $\inef^1_{S \cup \{j \}} - \inef^2_{S \cup \{j \}}\geq \inef^1_{S \cup \{k \}} - \inef^2_{S \cup \{k \}}$ for all $S \subseteq \{1, ..., n\} \backslash \{j, k\}$. By using the definitions of $\blame^1 = \blamefunc_{AP}(M^1, \bpi^1)$ and $\blame^2 = \blamefunc_{AP}(M^2, \bpi^2)$, this implies:
    \begin{align*}
        \blamej^1 - \blamej^2 =& \sum_{S \subseteq \{1, ..., n\} \backslash \{j\}} \weight \cdot \frac{c(M^1, \bpi^1, j)}{\sum_{i \in S} c(M^1, \bpi^1, i) + 1} \cdot \left [ \inef^1_{S \cup \{ j \}} - \inef^2_{S \cup \{ j \}}\right] =\\
        =& \sum_{S \subseteq \{1, ..., n\} \backslash \{j,k\}} \weight \cdot \frac{c(M^1, \bpi^1, j)}{\sum_{i \in S} c(M^1, \bpi^1, i) + 1} \cdot \left [ \inef^1_{S \cup \{ j \}} - \inef^2_{S \cup \{ j \}}\right] +\\
        &+ \sum_{S \subseteq \{1, ..., n\} \backslash \{j,k\}} \weight \cdot \frac{c(M^1, \bpi^1, j)}{\sum_{i \in S \cup \{ k \}} c(M^1, \bpi^1, i) + 1} \cdot \left [ \inef^1_{S \cup \{ j,k \}} - \inef^2_{S \cup \{ j,k \}}\right] \ge\\
        \ge& \sum_{S \subseteq \{1, ..., n\} \backslash \{j,k\}} \weight \cdot \frac{c(M^1, \bpi^1, k)}{\sum_{i \in S} c(M^1, \bpi^1, i) + 1} \cdot \left [ \inef^1_{S \cup \{ k \}} - \inef^2_{S \cup \{ k \}}\right] +\\
        &+ \sum_{S \subseteq \{1, ..., n\} \backslash \{j,k\}} \weight \cdot \frac{c(M^1, \bpi^1, k)}{\sum_{i \in S \cup \{ j \}} c(M^1, \bpi^1, i) + 1} \cdot \left [ \inef^1_{S \cup \{ j,k \}} - \inef^2_{S \cup \{ j,k \}}\right] =\\
        =& \sum_{S \subseteq \{1, ..., n\} \backslash \{k\}} \weight \cdot \frac{c(M^1, \bpi^1, k)}{\sum_{i \in S} c(M^1, \bpi^1, i) + 1} \cdot \left [ \inef^1_{S \cup \{ k \}} - \inef^2_{S \cup \{ k \}}\right] = \blamek^1 - \blamek^2,
    \end{align*}
    and hence property $\prop_{RcParM}$ (relative c-participation monotonicity) is satisfied.
    \item \underline{$\prop_{cPerM}$ (c-performance monotonicity):} 
    Consider $M$, $\bpi_{-i}$, $\pi_i$ and $\pi_i'$ such that $\return(\pi_i,\bpi_{-i}) \le \return(\pi_i',\bpi_{-i})$ and $c(M, (\pi_i,\bpi_{-i}), j) = c(M, (\pi_i',\bpi_{-i}), j)$ for every $j$. This implies that:
    \begin{align*}
        &\return(\pi_i,\bpi_{-i}) \le \return(\pi_i',\bpi_{-i}) \Rightarrow\\
        \Rightarrow& \return(\optpirb_{S \cup \{ i \}}, \bpi_{-S\cup\{i\}}) - \return(\pi_i,\bpi_{-i}) \ge \return(\optpirb_{S \cup \{ i \}}, \bpi_{-S\cup\{i\}}) - \return(\pi'_i,\bpi_{-i}) \Rightarrow\\
        \Rightarrow& \inef_{S\cup\{i\}} \ge \inef'_{S\cup\{i\}}
    \end{align*}
    for every $S \subseteq \{1, ..., n\} \backslash \{i\}$. Given the definitions of $\blame = \blamefunc_{AP}(M, (\pi_i,\bpi_{-i}))$ and $\blame' = \blamefunc_{AP}(M, (\pi_i',\bpi_{-i}))$, this implies that $\blamei \ge \blamei'$. Hence,  property $\prop_{cPerM}$ (c-performance monotonicity) is satisfied.
    \item \underline{$\prop_{V}$ (validity):} Consider $M$ , $\bpi$. Notice that $\sum_{S \subseteq \{1, ..., n\}} \frac{1}{2^n - 1} \cdot \inef_S \le \inef$. Given that $\blamefunc_{AP}$ satisfies property $\prop_{AE}$ (average efficiency), we have that $\sum_{i = 1}^n \blamei = \sum_{S \subseteq \{1, ..., n\}} \frac{1}{2^n - 1} \cdot \inef_S$, and thus $\sum_{i = 1}^n \blamei \le \inef$, where $\blame = \blamefunc_{AP}(M, \bpi)$. Hence property $\prop_V$ (validity) is satisfied.
\end{itemize}

\subsubsubsection{\textbf{Second Part}}

We begin by introducing some additional notation. Consider $M$ , $\bpi$. We define the sets of agents $C_0 = \big\{i \in \{1, \dots, n\} : c(M, \bpi, i) = 0 \big\}$ and $C_1 = \big\{i \in \{1, \dots, n\} : c(M, \bpi, i) = 1 \big\}$. Consider $M^\epsilon$, $\bpi^\epsilon$ such that:
\begin{align}\label{eq.inef_eps}
    \inef^\epsilon_{S} = \begin{cases}
            \inef_{S} + \epsilon  \quad &\mbox{ if } S \cap C_1 \neq \emptyset\\
            \inef_{S} \quad &\mbox{ otherwise,}
    \end{cases}
\end{align}
where $\epsilon > 0$. Note that for every subset $S$ such that $S \cap C_1 = \emptyset$ it holds that $\inef^\epsilon_{S} = 0$, but we use $\inef^\epsilon_{S} = \inef_S$ for notational simplicity. Moreover, notice that Eq. \eqref{eq.inef_eps} satisfies the conditions of Lemma \ref{lem.inef_existence}, and hence $M^\epsilon$, $\bpi^\epsilon$ exist.

We prove that $\blamefunc_{AP}$ uniquely satisfies the properties mentioned in Theorem \ref{thm.average_participation.unique} through two intermediate lemmas. Lemma \ref{lem.ap_unq_ifeps} states that if $\blamefunc(M^\epsilon, \bpi^\epsilon) = \blamefunc_{AP}(M^\epsilon, \bpi^\epsilon)$ then $\blamefunc(M, \bpi) = \blamefunc_{AP}(M, \bpi)$, and Lemma \ref{lem.ap_unq_eps} states that $\blamefunc(M^\epsilon, \bpi^\epsilon) = \blamefunc_{AP}(M^\epsilon, \bpi^\epsilon)$.

\begin{lemma}\label{lem.ap_unq_ifeps}
Consider $M$, $\bpi$ and $M^\epsilon$, $\bpi^\epsilon$, where $\inef^\epsilon_S$ is defined by Eq. \eqref{eq.inef_eps}. If $\blamefunc$ satisfies properties $\prop_{AE}$, $\prop_{S}$, $\prop_{I}$, $\prop_{cParM}$ and $\prop_{RcParM}$ and $\blamefunc(M^\epsilon, \bpi^\epsilon) = \blamefunc_{AP}(M^\epsilon, \bpi^\epsilon)$, then $\blamefunc(M, \bpi) = \blamefunc_{AP}(M, \bpi)$.
\end{lemma}

\begin{proof}
We state three claims that we prove after the end of the proof of Theorem \ref{thm.average_participation.unique}:

\begin{claim}\label{clm.c_eps}
$c(M, \bpi, i) = c(M^\epsilon, \bpi^\epsilon, i)$ for every $i$.
\end{claim}

\begin{claim}\label{clm.bl_c0_eps}
$\blamei = 0$ and $\blamei^\epsilon = 0$ for every $i \in C_0$, where $\blame = \blamefunc(M, \bpi)$ and $\blame^\epsilon = \blamefunc(M^\epsilon, \bpi^\epsilon)$.
\end{claim}

\begin{claim}\label{clm.bl_c1_eps}
$\blamei^\epsilon - \blamei = r$ for every $i \in C_1$, where $\blame = \blamefunc(M, \bpi)$ and $\blame^\epsilon = \blamefunc(M^\epsilon, \bpi^\epsilon)$, and $r = \frac{1}{|C_1|} \cdot \sum_{S \subseteq \{1, ...,n\}} \weight \cdot \left[\inef^\epsilon_S - \inef_S \right]$.
\end{claim}

Given Claim \ref{clm.bl_c1_eps}, the assumption $\blamefunc(M^\epsilon, \bpi^\epsilon) = \blamefunc_{AP}(M^\epsilon, \bpi^\epsilon)$ implies that for every $i \in C_1$:
\begin{align}\label{eq.ap_c1}
    \blamei =& \sum_{S \subseteq \{1, ..., n\} \backslash \{i\}} \weight \cdot \frac{1}{\sum_{j \in S} c(M^\epsilon, \bpi^\epsilon, j) + 1} \cdot \inef^\epsilon_{S \cup \{i\}} - \frac{1}{|C_1|} \cdot \sum_{S \subseteq \{1, ...,n\}} \weight \cdot \left[\inef^\epsilon_S - \inef_S \right].
\end{align}
Combining Claim \ref{clm.bl_c0_eps} and Eq. \eqref{eq.ap_c1} implies that:
\begin{align}\label{eq.eps}
    \blamei = c(M, \bpi, i) \cdot \left [ \sum_{S \subseteq \{1, ..., n\} \backslash \{i\}} \weight \cdot \frac{1}{\sum_{j \in S} c(M^\epsilon, \bpi^\epsilon, j) + 1} \cdot \inef^\epsilon_{S \cup \{i\}} - \frac{1}{|C_1|} \cdot \sum_{S \subseteq \{1, ...,n\}} \weight \cdot \left[\inef^\epsilon_S - \inef_S \right]\right].
\end{align}
Notice that $\blame = \blamefunc(M, \bpi)$ is uniquely defined by the properties of $\blamefunc$, Eq. \eqref{eq.eps}, and since $\blamefunc_{AP}$ satisfies all properties assumed for $\blamefunc$ (see Part 1), it must hold that $\blamefunc(M, \bpi) = \blamefunc_{AP}(M, \bpi)$. This concludes the proof of Lemma \ref{lem.ap_unq_ifeps}.
\end{proof}

\begin{lemma}\label{lem.ap_unq_eps}
Consider $M$, $\bpi$ and $M^\epsilon$, $\bpi^\epsilon$, where $\inef^\epsilon_S$ is defined by Eq. \eqref{eq.inef_eps}. If $\blamefunc$ satisfies properties $\prop_{AE}$, $\prop_{S}$, $\prop_{I}$, $\prop_{cParM}$ and $\prop_{RcParM}$, then $\blamefunc(M^\epsilon, \bpi^\epsilon) = \blamefunc_{AP}(M^\epsilon, \bpi^\epsilon)$.
\end{lemma}
\begin{proof}
Let $I = \{1, ..., 2^{|C_1|} - 1\}$ be an index set, and for each $\iota \in I$, let $S_\iota$ be a subset of $C_1$ other than $\emptyset$. We assume that the indexing of subsets $S \subseteq C_1$ satisfies the following condition: for every $\iota, \zeta \in I$, $\iota < \zeta$ whenever $|S_{\iota}| > |S_{\zeta}|$.

Consider $M$, $\bpi$ and $M^\epsilon$, $\bpi^\epsilon$, where $\inef^\epsilon_S$ is defined by Eq. \eqref{eq.inef_eps}. For each index number $\iota \in I$ consider $M^\iota$, $\bpi^\iota$ such that:
\begin{align}\label{eq.inef_iot}
    \inef^\iota_{S} = \begin{cases}
            \epsilon  \quad &\mbox{ if } S \cap C_1 = S_\zeta, \mbox{ where } \zeta > \iota\\
            \inef_S \quad &\mbox{ if } S \cap C_1 = \emptyset\\
            \inef_S + \epsilon \quad &\mbox{ otherwise,}
    \end{cases}
\end{align}
where $\epsilon > 0$. Note that for every subset $S$ such that $S \cap C_1 = \emptyset$ it holds that $\inef^\iota_{S} = 0$, but we use $\inef^\iota_{S} = \inef_S$ for notational simplicity. Moreover, notice that for every $\iota \in I$ Eq. \eqref{eq.inef_iot} satisfies the conditions of Lemma \ref{lem.inef_existence}, and hence $M^\iota$, $\bpi^\iota$ exist. Notice also that $\inef_S^{2^{|C_1|} - 1} = \inef_S^\epsilon$ for every $S$.

We state four claims that we prove after the end of the proof of Theorem \ref{thm.average_participation.unique}:

\begin{claim}\label{clm.c_iot}
For each $\iota \in I$, $c(M, \bpi, i) = c(M^\iota, \bpi^\iota, i)$ for every $i$.
\end{claim}

\begin{claim}\label{clm.bl_c0_iot}
For each $\iota \in I$, $\blamei^\iota = 0$ for every $i \in C_0$, where $\blame^\iota = \blamefunc(M^\iota, \bpi^\iota)$.
\end{claim}

\begin{claim}\label{clm.bl_c1not_iot}
For each $\iota \in I \backslash \{2^{|C_1|} - 1\}$, $\blamei^{\iota + 1} - \blamei^\iota = 0$ for every $i \in C_1 \backslash S_{\iota + 1}$, where $\blame^\iota = \blamefunc(M^\iota, \bpi^\iota)$ and $\blame^{\iota + 1} = \blamefunc(M^{\iota + 1}, \bpi^{\iota + 1})$.
\end{claim}

\begin{claim}\label{clm.bl_c1_iot}
For each $\iota \in I \backslash \{2^{|C_1|} - 1\}$, $\blamei^{\iota + 1} - \blamei^\iota = r$ for every $i \in S_{\iota + 1}$, where $\blame^\iota = \blamefunc(M^\iota, \bpi^\iota)$ and $\blame^{\iota + 1} = \blamefunc(M^{\iota + 1}, \bpi^{\iota + 1})$, and $r = \frac{1}{|S_{\iota + 1}|} \cdot \sum_{S \subseteq \{1, ...,n\}} \weight \cdot \left[\inef^{\iota + 1}_S - \inef^\iota_S \right]$.
\end{claim}

We prove that $\blamefunc(M^{2^{|C_1|} - 1}, \bpi^{2^{|C_1|} - 1}) = \blamefunc_{AP}(M^{2^{|C_1|} - 1}, \bpi^{2^{|C_1|} - 1})$, by using induction in the index number $\iota$. Note that because $\inef_S^{2^{|C_1|} - 1} = \inef_S^\epsilon$ for every $S$, showing $\blamefunc(M^{2^{|C_1|} - 1}, \bpi^{2^{|C_1|} - 1}) = \blamefunc_{AP}(M^{2^{|C_1|} - 1}, \bpi^{2^{|C_1|} - 1})$ is equivalent to showing that $\blamefunc(M^\epsilon, \bpi^\epsilon) = \blamefunc_{AP}(M^\epsilon, \bpi^\epsilon)$.

$\bm{\iota = 1}$: We show that $\blamefunc(M^1, \bpi^1) = \blamefunc_{AP}(M^1, \bpi^1)$. Because of the condition that the indexing of the subsets of $C_1$ has to satisfy, it follows that $S_1 = C_1$. Notice that for every two agents $i, j \in C_1$ it holds that $S \cup \{i\} \cap C_1 \ne C_1 = S_1$ and $S \cup \{j\} \cap C_1 \ne C_1 = S_1$, for every $S \subseteq \{1, ..., n\} \backslash \{i, j\}$. By using Eq. \eqref{eq.inef_iot}, this implies that $\inef^1_{S \cup \{i\}} = \epsilon = \inef^1_{S \cup \{j\}}$, for every $S \subseteq \{1, ..., n\} \backslash \{i, j\}$. Given that $\blamefunc$ is assumed to satisfy $\prop_S$ (symmetry), this implies that $\blamei^1 = \blamej^1$, where $\blame^1 = \blamefunc(M^1, \bpi^1)$. It follows that for every $i \in C_1$:
\begin{align*}
    \blamei^1 = \frac{1}{|C_1|} \cdot \sum_{j \in C_1} \blamej^1.
\end{align*}
By using Claim \ref{clm.bl_c0_iot}, we have that $\blamei^1 = \frac{1}{|C_1|} \cdot \sum_{j \in \{1, ..., n\}} \blamej^1$. Given that $\blamefunc$ satisfies $\prop_{AE}$ (average efficiency), this implies that for every $i \in C_1$:
\begin{align}
    \blamei^1 &= \frac{1}{|C_1|} \cdot \sum_{j \in \{1, ..., n\}} \blamej^1 = \frac{1}{|C_1|} \cdot \sum_{S \subseteq \{1, ..., n\}} \weight \cdot \inef^1_S. \label{eq.c1_iot=1}
\end{align}
Combining Claim \ref{clm.bl_c0_iot} and Eq. \eqref{eq.c1_iot=1} implies that:
\begin{align}\label{eq.iot=1}
    \blamei^1 = c(M, \bpi, i) \cdot \frac{1}{|C_1|} \cdot \sum_{S \subseteq \{1, ..., n\}} \weight \cdot \inef^1_S.
\end{align}
Notice that $\blame^1 = \blamefunc(M^1, \bpi^1)$ is uniquely defined by the properties of $\blamefunc$, Eq. \eqref{eq.iot=1}, and since $\blamefunc_{AP}$ satisfies all properties assumed for $\blamefunc$ (see Part 1), it must hold that $\blamefunc_{AP}(M^1, \bpi^1) = \blame^1$, and hence $\blamefunc(M^1, \bpi^1) = \blamefunc_{AP}(M^1, \bpi^1)$.

$\bm{\iota \in I \backslash \{2^{|C_1|} - 1\}}$: Given that $\blamefunc(M^\iota, \bpi^\iota) = \blamefunc_{AP}(M^\iota, \bpi^\iota)$, we show that $\blamefunc(M^{\iota + 1}, \bpi^{\iota + 1}) = \blamefunc_{AP}(M^{\iota + 1}, \bpi^{\iota + 1})$. 

By using the definition of $\blamefunc_{AP}(M^\iota, \bpi^\iota)$ and Claim \ref{clm.bl_c1not_iot}, the assumption $\blamefunc(M^\iota, \bpi^\iota) = \blamefunc_{AP}(M^\iota, \bpi^\iota)$ implies that for every $i \in C_1 \backslash S_{\iota + 1}$:
\begin{align}\label{eq.iota+1_not}
    \blamei^{\iota + 1} = \blamei^\iota = \sum_{S \subseteq \{1, ..., n\} \backslash \{i\}} \weight \cdot \frac{1}{\sum_{j \in S} c(M^\iota, \bpi^\iota, j) + 1} \cdot \inef^\iota_{S \cup \{i\}}.
\end{align}
By using the definition of $\blamefunc_{AP}(M^\iota, \bpi^\iota)$ and Claim \ref{clm.bl_c1_iot}, the assumption $\blamefunc(M^\iota, \bpi^\iota) = \blamefunc_{AP}(M^\iota, \bpi^\iota)$ implies that for every $i \in S_{\iota + 1}$:
\begin{align}\label{eq.iota+1}
    \blamei^{\iota + 1} =& \blamei^\iota + \frac{1}{|S_{\iota + 1}|} \cdot \sum_{S \subseteq \{1, ...,n\}} \weight \cdot \left[\inef^{\iota + 1}_S - \inef^\iota_S \right] =\nonumber\\
    =& \sum_{S \subseteq \{1, ..., n\} \backslash \{i\}} \weight \cdot \frac{1}{\sum_{j \in S} c(M^\iota, \bpi^\iota, j) + 1} \cdot \inef^\iota_{S \cup \{i\}} + \frac{1}{|S_{\iota + 1}|} \cdot \sum_{S \subseteq \{1, ...,n\}} \weight \cdot \left[\inef^{\iota + 1}_S - \inef^\iota_S \right].
\end{align}
Notice that $\blame^{\iota + 1} = \blamefunc(M^{\iota + 1}, \bpi^{\iota + 1})$ is uniquely defined by properties of $\blamefunc$, Claim \ref{clm.bl_c0_iot}, Eq. \eqref{eq.iota+1_not} and Eq. \eqref{eq.iota+1}, and since $\blamefunc_{AP}$ satisfies all the properties assumed for $\blamefunc$ (see Part 1), it must hold that $\blamefunc(M^{\iota + 1}, \bpi^{\iota + 1}) = \blame^{\iota + 1}$, and hence $\blamefunc(M^{\iota + 1}, \bpi^{\iota + 1}) = \blamefunc_{AP}(M^{\iota + 1}, \bpi^{\iota + 1})$. This concludes the induction step and the proof of Lemma \ref{lem.ap_unq_eps}.
\end{proof}
The second part of the proof is hence concluded.
\end{proof}

\addtocounter{claim}{-7}

\subsubsubsection{\textbf{Proofs of the Claims \ref{clm.c_eps}, \ref{clm.bl_c0_eps} and \ref{clm.bl_c1_eps}}}

\begin{claim}
$c(M, \bpi, i) = c(M^\epsilon, \bpi^\epsilon, i)$ for every $i$.
\end{claim}

\begin{proof}
Consider agent $i$ such that $i \in C_1$. Given Eq. \eqref{eq.inef_eps}, this implies that $\inef^\epsilon_i = \inef_i + \epsilon > 0 = \inef^\epsilon_\emptyset$, and thus $c(M^\epsilon, \bpi^\epsilon, i) = 1$. Hence, $c(M, \bpi, i) = c(M^\epsilon, \bpi^\epsilon, i)$.

Consider agent $i$ such that $i \in C_0$. Given the definition of contribution function $c$, we have that $\inef_{S \cup \{i\}} = \inef_S$ for every $S$. By using Eq. \eqref{eq.inef_eps}, this implies that $\inef^\epsilon_{S \cup \{i\}} = \inef^\epsilon_S$ for every $S$ such that $S \cap C_1 = \emptyset$ and $\inef^\epsilon_{S \cup \{i\}} = \inef_{S \cup \{i\}} + \epsilon = \inef_S + \epsilon = \inef^\epsilon_S$ for every $S$ such that $S \cap C_1 \ne \emptyset$, and thus $c(M^\epsilon, \bpi^\epsilon, i) = 0$. Hence, $c(M, \bpi, i) = c(M^\epsilon, \bpi^\epsilon, i)$.
\end{proof}

\begin{claim}
$\blamei = 0$ and $\blamei^\epsilon = 0$ for every $i \in C_0$, where $\blame = \blamefunc(M, \bpi)$ and $\blame^\epsilon = \blamefunc(M^\epsilon, \bpi^\epsilon)$.
\end{claim}

\begin{proof}
Given the definition of contribution function $c$, the lemma follows from Claim \ref{clm.c_eps} and the assumption that $\blamefunc$ satisfies property $\prop_I$ (invariance).
\end{proof}
\begin{claim}
$\blamei^\epsilon - \blamei = r$ for every $i \in C_1$, where $\blame = \blamefunc(M, \bpi)$ and $\blame^\epsilon = \blamefunc(M^\epsilon, \bpi^\epsilon)$, and $r = \frac{1}{|C_1|} \cdot \sum_{S \subseteq \{1, ...,n\}} \weight \cdot \left[\inef^\epsilon_S - \inef_S \right]$.
\end{claim}
\begin{proof}
Notice that for every two agents $j, k \in C_1$ it holds that $c(M, \bpi, j) = c(M, \bpi, k)$ and that $\inef^\epsilon_{S \cup \{j\}} - \inef_{S \cup \{j\}} = \inef^\epsilon_{S \cup \{k\}} - \inef_{S \cup \{k\}} = \epsilon$ for every $S \subseteq \{1, ..., n\} \backslash \{j, k\}$. Furthermore, from Claim \ref{clm.c_eps} we have that $c(M, \bpi, i) = c(M^\epsilon, \bpi^\epsilon, i)$ for every $i$, and thus Lemma \ref{lem.rcparm} applies, $\blamej^\epsilon - \blamej = \blamek^\epsilon - \blamek = r$, where $r$ is some constant. Notice that:
\begin{align*}
    r = \frac{1}{|C_1|} \cdot \sum_{i \in C_1} \blamei^\epsilon - \blamei.
\end{align*}
By using Claim \ref{clm.bl_c0_eps}, we have that $r = \frac{1}{|C_1|} \cdot \sum_{i \in \{1, ..., n\}} \blamei^\epsilon - \blamei$. Given that $\blamefunc$ is assumed to satisfy $\prop_{AE}$ (average efficiency), this implies that:
\begin{align*}
    r &= \frac{1}{|C_1|} \cdot \sum_{i \in \{1, ..., n\}} \blamei^\epsilon - \blamei = \frac{1}{|C_1|} \cdot \sum_{S \subseteq \{1, ...,n\}} \weight \cdot \left[\inef^\epsilon_S - \inef_S \right],
\end{align*}
and hence $\blamei^\epsilon - \blamei =\frac{1}{|C_1|} \cdot \sum_{S \subseteq \{1, ...,n\}} \weight \cdot \left[\inef^\epsilon_S - \inef_S \right]$ for every $i \in C_1$.
\end{proof}

\subsubsubsection{\textbf{Proofs of the Claims \ref{clm.c_iot}, \ref{clm.bl_c0_iot}, \ref{clm.bl_c1not_iot} and \ref{clm.bl_c1_iot}}}

\begin{claim}
For each $\iota \in I$, $c(M, \bpi, i) = c(M^\iota, \bpi^\iota, i)$ for every $i$.
\end{claim}

\begin{proof}
Consider agent $i$ such that $i \in C_1$. Given Eq. \eqref{eq.inef_iot}, this implies that $\inef^\iota_i \ge \epsilon > 0 = \inef^\iota_\emptyset$, and thus $c(M^\iota, \bpi^\iota, i) = 1$. Hence, $c(M, \bpi, i) = c(M^\iota, \bpi^\iota, i)$.

Consider agent $i$ such that $i \in C_0$. Given the definition of contribution function $c$, we have that $\inef_{S \cup \{i\}} = \inef_S$ for every $S$. By using Eq. \eqref{eq.inef_iot}, this implies that $\inef^\iota_{S \cup \{i\}} = \epsilon = \inef^\iota_S$ for every $S$ such that $S \cap C_1 = S_\zeta$, where $\zeta > \iota$, $\inef^\iota_{S \cup \{i\}} = \inef^\iota_S$ for every $S$ such that $S \cap C_1 = \emptyset$ and $\inef^\iota_{S \cup \{i\}} = \inef_{S \cup \{i\}} + \epsilon = \inef_S + \epsilon = \inef^\iota_S$ for every other $S$, and thus $c(M^\iota, \bpi^\iota, i) = 0$. Hence, $c(M, \bpi, i) = c(M^\iota, \bpi^\iota, i)$.
\end{proof}

\begin{claim}
For each $\iota \in I$, $\blamei^\iota = 0$ for every $i \in C_0$, where $\blame^\iota = \blamefunc(M^\iota, \bpi^\iota)$.
\end{claim}

\begin{proof}
Given the definition of contribution function $c$, the lemma follows from Claim \ref{clm.c_iot} and the assumption that $\blamefunc$ satisfies property $\prop_I$ (invariance).
\end{proof}

Based on the next observation we prove the rest of the claims:

\begin{observation}\label{obs}
Observe that for each $\iota \in I \backslash \{2^{|C_1|} - 1\}$, $\inef^{\iota + 1}_S = \inef^\iota_S$ for every $S$ such that $S \cap C_1 \ne S_{\iota + 1}$.\footnote{Although it is not needed for the proofs of Claims \ref{clm.bl_c1not_iot} and \ref{clm.bl_c1_iot}, we mention that $\inef^{\iota + 1}_S = \inef^\iota_S + \inef_{S_{\iota + 1}}$ for every $S$ such that $S \cap C_1 = S_{\iota + 1}$.}
\end{observation}

\begin{claim}
For each $\iota \in I \backslash \{2^{|C_1|} - 1\}$, $\blamei^{\iota + 1} - \blamei^\iota = 0$ for every $i \in C_1 \backslash S_{\iota + 1}$, where $\blame^\iota = \blamefunc(M^\iota, \bpi^\iota)$ and $\blame^{\iota + 1} = \blamefunc(M^{\iota + 1}, \bpi^{\iota + 1})$.
\end{claim}

\begin{proof}
Notice that for every agent $i \in C_1 \backslash S_{\iota + 1}$ it holds that $S \cup \{i\} \cap C_1 \ne S_{\iota + 1}$ for every $S$. Given Observation \ref{obs} this implies that $\inef^{\iota + 1}_{S \cup \{i\}} = \inef^\iota_{S \cup \{i\}}$ for every $S$. Furthermore, from Claim \ref{clm.c_iot} we have that $c(M, \bpi, i) = c(M^\iota, \bpi^\iota, i) = c(M^{\iota + 1}, \bpi^{\iota + 1}, i)$ for every $i$, and thus Lemma \ref{lem.cparm} applies, and for every $i \in C_1 \backslash S_{\iota + 1}$ we have that $\blamei^{\iota + 1} = \blamei^\iota$.
\end{proof}

\begin{claim}
For each $\iota \in I \backslash \{2^{|C_1|} - 1\}$, $\blamei^{\iota + 1} - \blamei^\iota = r$ for every $i \in S_{\iota + 1}$, where $\blame^\iota = \blamefunc(M^\iota, \bpi^\iota)$ and $\blame^{\iota + 1} = \blamefunc(M^{\iota + 1}, \bpi^{\iota + 1})$, and $r = \frac{1}{|S_{\iota + 1}|} \cdot \sum_{S \subseteq \{1, ...,n\}} \weight \cdot \left[\inef^{\iota + 1}_S - \inef^\iota_S \right]$.
\end{claim}

\begin{proof}
Notice that for every two agents $j, k \in S_{\iota + 1}$ it holds that $c(M, \bpi, j) = c(M, \bpi, k)$. Given Claim \ref{clm.c_iot}, this implies that $c(M^\iota, \bpi^\iota, j) = c(M^\iota, \bpi^\iota, k)$. Notice also that $S \cup \{j\} \cap C_1 \ne S_{\iota + 1}$ and $S \cup \{k\} \cap C_1 \ne S_{\iota + 1}$ for every $S \subseteq \{1, ..., n\} \backslash \{j, k\}$.
Given Observation \ref{obs}, this implies that $\inef^{\iota + 1}_{S \cup \{j\}} - \inef^\iota_{S \cup \{j\}} = \inef^{\iota + 1}_{S \cup \{k\}} - \inef^\iota_{S \cup \{k\}} = 0$ for every $S \subseteq \{1, ..., n\} \backslash \{j, k\}$. Furthermore, from Claim \ref{clm.c_iot} we have that $c(M, \bpi, i) = c(M^\iota, \bpi^\iota, i) = c(M^{\iota + 1}, \bpi^{\iota + 1}, i)$ for every $i$, and thus Lemma \ref{lem.rcparm} applies, and for every $j, k \in S_{\iota + 1}$ we have that $\blamej^{\iota + 1} - \blamej^\iota = \blamek^{\iota + 1} - \blamek^\iota = r$, where $r$ is some constant. Notice that:
\begin{align*}
    r = \frac{1}{|S_{\iota + 1}|} \cdot \sum_{i \in S_{\iota + 1}} \blamei^{\iota + 1} - \blamei^\iota.
\end{align*}
By using Claim \ref{clm.bl_c0_iot} and Claim \ref{clm.bl_c1not_iot}, we have that $r = \frac{1}{|S_{\iota + 1}|} \cdot \sum_{i \in \{1, ..., n\}} \blamei^{\iota + 1} - \blamei^\iota$. Given that $\blamefunc$ is assumed to satisfy $\prop_{AE}$ (average efficiency), this implies that:
\begin{align*}
    r &= \frac{1}{|S_{\iota + 1}|} \cdot \sum_{i \in \{1, ..., n\}} \blamei^{\iota + 1} - \blamei^\iota = \frac{1}{|S_{\iota + 1}|} \cdot \sum_{S \subseteq \{1, ...,n\}} \weight \cdot \left[\inef^{\iota + 1}_S - \inef^\iota_S \right],
\end{align*}
and hence $\blamei^{\iota + 1} - \blamei^\iota = r = \frac{1}{|S_{\iota + 1}|} \cdot \sum_{S \subseteq \{1, ...,n\}} \weight \cdot \left[\inef^{\iota + 1}_S - \inef^\iota_S \right]$ for every $i \in S_{\iota + 1}$.
\end{proof}

\section{Proofs of the Results from Section \ref{sec.blame_under_uncertainty}}\label{sec.proof_uncertainty}

This section of the appendix contains the proofs of the results from Section \ref{sec.approaches_to_blame}, in particular: Proposition \ref{prop.uncertainty_validity}, Proposition \ref{prop.uncertainty_consistency}, and Theorem \ref{thm.approx_satisfability}. 

\subsection{Proof of Proposition \ref{prop.uncertainty_validity}}

\begin{proposition}
Let $\bpie$ be a solution to the optimization problem $\max_{\pi \in \mathcal P(\bpi)} \return (\pi)$. Then
$\widehat{\blamefunc}_{SV, V}(M, \mathcal P(\bpi)) = \blamefunc_{SV}(M, \bpie)$ satisfies $\prop_V$ (validity).
\end{proposition}
\begin{proof}
In the setting of interest, $P(\bpi)$ is consistent with $\bpi$, that is $\bpi \in P(\bpi)$, and hence $\return(\bpie) \geq \return(\bpi)$. By Theorem \ref{thm.shapley_value.unique}, the blame attribution method $\blamefunc_{SV}$ satisfies property $\prop_E$ (efficiency), which implies that $\sum_{i = 1}^n \widehat{\blamei} = \return(\optpi) - \return(\bpie)$, where $\widehat{\blame} = \blamefunc_{SV}(M, \bpie)$. This implies:
\begin{align*}
    & \return(\bpie) \geq \return(\bpi)\Rightarrow\\
    \Rightarrow& \return(\optpi) - \return(\bpie) \leq \return(\optpi) - \return(\bpi)\Rightarrow\\
    \Rightarrow& \return(\optpi) - \return(\bpie)\leq \inef \Rightarrow\\
    \Rightarrow& \sum_{i = 1}^n \widehat{\blamei} \leq \inef.
\end{align*}
Therefore, $\widehat{\blamefunc}_{SV, V}$ satisfies property $\prop_V$ (validity).
\end{proof}

\subsection{Proof of Proposition \ref{prop.uncertainty_consistency}}

\begin{proposition}
Let $\blame^i_i$ be the minimum value of the objective in \eqref{prob.robust_sv}.
Then $\widehat{\blamefunc}_{SV, BC}(M, \mathcal P(\bpi)) = (\blame^1_1, ..., \blame^n_n)$ satisfies $\prop_V$ (validity) and $\prop_{BC}(\blamefunc_{SV})$ (Blackstone consistency  $w.r.t.$ $\blamefunc_{SV}(M, \bpi)$).
\end{proposition}
\begin{proof}
Let $\blame = \blamefunc_{SV}(M, \bpi)$. Given Eq. \eqref{eq.def_sv}, $\blame^i_i$ being the minimum value of the objective in \eqref{prob.robust_sv} implies that $\blame^i_i = \min_{\pi \in \mathcal P(\bpi)} \blamei^\pi$ s.t. $\blame^\pi = \blamefunc_{SV}(M, \pi)$. In the setting of interest, $P(\bpi)$ is consistent with $\bpi$, that is $\bpi \in P(\bpi)$, which implies that $\blame^i_i = \min_{\pi \in \mathcal P(\bpi)} \blamei^\pi \le \blamei$. Therefore, $\widehat{\blamefunc}_{SV, BC}(M, \mathcal P(\bpi))$ satisfies $\prop_{BC}(\blamefunc_{SV})$ (Blackstone consistency  $w.r.t.$ $\blamefunc_{SV}(M, \bpi)$).
Furthermore, by applying the same reasoning to all agents, we obtain $\sum_{i \in \{1, ..., n\}}\blamei^i \leq \sum_{i \in \{1, ..., n\}}\blamei$. Given Theorem \ref{thm.shapley_value.unique}, this implies $\sum_{i \in \{1, ..., n\}}\blamei^i \leq \inef$, and hence $\widehat{\blamefunc}_{SV, BC}(M, \mathcal P(\bpi))$ also satisfies $\prop_V$ (validity).
\end{proof}

\subsection{Proof of Theorem \ref{thm.approx_satisfability}}

\begin{theorem}
Consider $\widehat \blamefunc$ and $\blamefunc$ s.t. $\norm{\widehat \blamefunc(M, \mathcal P(\bpi)) -\blamefunc(M, \bpi)}_{1} \le \epsilon$ for any $M$, $\bpi$, and  $\mathcal P(\bpi)$. Then if $\blamefunc$ satisfies a property $\prop \in \{ \prop_V, \prop_E, \prop_R, \prop_S, \prop_I,  \prop_{AE}\}$,  $\widehat \blamefunc$ satisfies $\epsilon$-$\prop$. Moreover, if $\blamefunc$ satisfies a property $\prop \in \{ \prop_{CM} ,\prop_{PerM}, \prop_{cPerM}, \prop_{cParM}, \prop_{RcParM}\}$,  $\widehat \blamefunc$ satisfies $2\epsilon$-$\prop$.
\end{theorem}

\begin{proof}
We prove the implication for each property $\prop$:
\begin{itemize}
    \item \underline{$\prop_V$ (validity):} Let $\blame = \blamefunc(M, \bpi)$ and $\widehat \blame = \widehat \blamefunc(M, \mathcal P(\bpi))$. If $\blamefunc$ satisfies $\prop_V$,
    \begin{align*}
        & \norm{\widehat{\blame} - \blame}_1 \leq \epsilon
        \Rightarrow \sum_{i = 1}^n |\widehat{\blamei} - \blamei| \leq \epsilon
        \Rightarrow \bigg|\sum_{i = 1}^n \widehat{\blamei} - \blamei\bigg| \leq \epsilon \Rightarrow\\
        \Rightarrow& \sum_{i = 1}^n \widehat{\blamei} \leq \sum_{i = 1}^n \blamei + \epsilon
        \Rightarrow \sum_{i = 1}^n \widehat{\blamei} \leq \inef + \epsilon,
    \end{align*}
    and hence $\widehat \blamefunc$ satisfies $\epsilon$-$\prop_V$.
    \item \underline{$\prop_E$ (efficiency):} Let $\blame = \blamefunc(M, \bpi)$ and $\widehat \blame = \widehat \blamefunc(M, \mathcal P(\bpi))$. If $\blamefunc$ satisfies $\prop_E$,
    \begin{align*}
        & \norm{\widehat{\blame} - \blame}_1 \leq \epsilon
        \Rightarrow \sum_{i = 1}^n |\widehat{\blamei} - \blamei| \leq \epsilon
        \Rightarrow \bigg|\sum_{i = 1}^n \widehat{\blamei} - \blamei\bigg| \leq \epsilon \Rightarrow\\
        \Rightarrow& \bigg|\sum_{i = 1}^n \widehat{\blamei} - \inef\bigg| \leq \epsilon,
    \end{align*}
    and hence $\widehat \blamefunc$ satisfies $\epsilon$-$\prop_E$.
    \item \underline{$\prop_R$ (rationality):} Let $\blame = \blamefunc(M, \bpi)$ and $\widehat \blame = \widehat \blamefunc(M, \mathcal P(\bpi))$. If $\blamefunc$ satisfies $\prop_R$,
    \begin{align*}
        & \norm{\widehat{\blame} - \blame}_1 \leq \epsilon
        \Rightarrow \sum_{i = 1}^n |\widehat{\blamei} - \blamei| \leq \epsilon
        \Rightarrow \sum_{i \in S} |\widehat{\blamei} - \blamei| \leq \epsilon \Rightarrow\\
        \Rightarrow& \bigg|\sum_{i \in S} \widehat{\blamei} - \blamei\bigg| \leq \epsilon 
        \Rightarrow \sum_{i \in S} \widehat{\blamei} \leq \sum_{i \in S} \blamei + \epsilon
        \Rightarrow \sum_{i \in S} \widehat{\blamei} \leq \inef_S + \epsilon,
    \end{align*}
    and hence $\widehat \blamefunc$ satisfies $\epsilon$-$\prop_R$.
    \item \underline{$\prop_S$ (symmetry):} Let $\blame = \blamefunc(M, \bpi)$ and $\widehat \blame = \widehat \blamefunc(M, \mathcal P(\bpi))$. If $\blamefunc$ satisfies $\prop_S$,
    \begin{align*}\label{ineq.r1}
        & \norm{\widehat{\blame} - \blame}_1 \leq \epsilon
        \Rightarrow \sum_{i = 1}^n |\widehat{\blamei} - \blamei| \leq \epsilon
        \Rightarrow |\widehat{\blamei} - \blamei| + |\widehat{\blamej} - \blamej| \leq \epsilon \Rightarrow\\
        \Rightarrow& \widehat{\blamei} - \blamei - \widehat{\blamej} + \blamej \leq \epsilon
        \Rightarrow \widehat{\blamei} - \widehat{\blamej} \leq \epsilon \tag{r1}
    \end{align*}
    and
    \begin{align*}\label{ineq.r2}
        & \norm{\widehat{\blame} - \blame}_1 \leq \epsilon
        \Rightarrow \sum_{i = 1}^n |\widehat{\blamei} - \blamei| \leq \epsilon
        \Rightarrow |\widehat{\blamei} - \blamei| + |\widehat{\blamej} - \blamej| \leq \epsilon \Rightarrow\\
        \Rightarrow& - \widehat{\blamei} + \blamei + \widehat{\blamej} - \blamej \leq \epsilon
        \Rightarrow - \widehat{\blamei} + \widehat{\blamej} \leq \epsilon \tag{r2}.
    \end{align*}
    From \eqref{ineq.r1} and \eqref{ineq.r2}, we have $|\widehat{\blamei} - \widehat{\blamej}| \leq \epsilon$, and hence $\widehat \blamefunc$ satisfies $\epsilon$-$\prop_S$.
    \item \underline{$\prop_I$ (invariance):} Let $\blame = \blamefunc(M, \bpi)$ and $\widehat \blame = \widehat \blamefunc(M, \mathcal P(\bpi))$. If $\blamefunc$ satisfies $\prop_I$,
    \begin{align*}
        & \norm{\widehat{\blame} - \blame}_1 \leq \epsilon
        \Rightarrow \sum_{i = 1}^n |\widehat{\blamei} - \blamei| \leq \epsilon
        \Rightarrow |\widehat{\blamei} - \blamei| \leq \epsilon \Rightarrow\\
        \Rightarrow& |\widehat{\blamei}| \leq \epsilon
        \Rightarrow \widehat{\blamei} \leq \epsilon,
    \end{align*}
    and hence $\widehat \blamefunc$ satisfies $\epsilon$-$\prop_I$.
    \item \underline{$\prop_{AE}$ (average efficiency):} Let $\blame = \blamefunc(M, \bpi)$ and $\widehat \blame = \widehat \blamefunc(M, \mathcal P(\bpi))$. If $\blamefunc$ satisfies $\prop_{AE}$,
    \begin{align*}
        & ||\widehat \blame - \blame||_1 \leq \epsilon
        \Rightarrow \sum_{i = 1}^n |\widehat{\blamei} - \blamei| \leq \epsilon
        \Rightarrow \bigg|\sum_{i = 1}^n \widehat{\blamei} - \blamei\bigg| \leq \epsilon \Rightarrow\\
        \Rightarrow& \bigg|\sum_{i = 1}^n \widehat{\blamei} - \sum_{S \subseteq \{1, ..., n\}} \frac{1}{2^n - 1} \cdot \inef_S\bigg| \leq \epsilon,
    \end{align*}
    and hence $\widehat \blamefunc$ satisfies $\epsilon$-$\prop_{AE}$.
    \item \underline{$\prop_{CM}$ (contribution monotonicity) and $\prop_{cParM}$ (c-participation monotonicity):} Let $\blame^1 = \blamefunc(M^1, \bpi^1)$, $\blame^2 = \blamefunc(M^2, \bpi^2)$, $\widehat \blame^1 = \widehat \blamefunc(M^1, \mathcal P(\bpi^1))$ and $\widehat \blame^2 = \widehat \blamefunc(M^2, \mathcal P(\bpi^2))$. 
    To show that $\blamefunc$ satisfying $\prop_{CM}$ (resp. $\prop_{cParM}$) implies that $\widehat \blamefunc$ satisfies $2\epsilon$-$\prop_{CM}$ (resp. $2\epsilon$-$\prop_{cParM}$), it suffices to show that $\blamei^1 - \blamei^2 \ge 0$ implies $\widehat{\blamei}^1 \geq \widehat{\blamei}^2 - 2\epsilon$. Let $\blamei^1 - \blamei^2 \ge 0$. We have
    \begin{align*}\label{ineq.r3}
        & \norm{\widehat{\blame}^1 - \blame^1}_1 \leq \epsilon
        \Rightarrow \sum_{i = 1}^n |\widehat{\blamei}^1 - \blamei^1| \leq \epsilon
        \Rightarrow |\widehat{\blamei}^1 - \blamei^1| \leq \epsilon \Rightarrow\\
        \Rightarrow& \blamei^1 - \widehat{\blamei}^1 \leq \epsilon \tag{r3}
    \end{align*}
    and
    \begin{align*}\label{ineq.r4}
        & \norm{\widehat{\blame}^2 - \blame^2}_1 \leq \epsilon
        \Rightarrow \sum_{i = 1}^n |\widehat{\blamei}^2 - \blamei^2| \leq \epsilon
        \Rightarrow |\widehat{\blamei}^2 - \blamei^2| \leq \epsilon \Rightarrow\\
        \Rightarrow& \widehat{\blamei}^2 - \blamei^2 \leq \epsilon \tag{r4}.
    \end{align*}
    By adding \eqref{ineq.r3} and \eqref{ineq.r4}, we obtain
    \begin{align*}
        & \blamei^1 - \widehat{\blamei}^1 + \widehat{\blamei}^2 - \blamei^2 \leq 2\epsilon
        \Rightarrow \widehat{\blamei}^1 \geq \widehat{\blamei}^2 - 2\epsilon,
    \end{align*}
    and hence $\widehat \blamefunc$ satisfies $2\epsilon$-$\prop_{CM}$ (resp. $2\epsilon$-$\prop_{cParM}$).
    
    \item \underline{$\prop_{PerM}$ (performance monotonicity) and $\prop_{cPerM}$ (c-performance monotonicity):} Let $\blame = \blamefunc(M, (\pi_i,\bpi_{-i}))$, $\blame' = \blamefunc(M, (\pi_i', \bpi_{-i}))$, $\widehat \blame = \widehat \blamefunc(M, \mathcal P((\pi_i,\bpi_{-i})))$ and $\widehat \blame' = \widehat \blamefunc(M, \mathcal P((\pi_i', \bpi_{-i})))$.
     To show that $\blamefunc$ satisfying $\prop_{PerM}$ (resp. $\prop_{cPerM}$) implies that $\widehat \blamefunc$ satisfies $2\epsilon$-$\prop_{PerM}$ (resp. $2\epsilon$-$\prop_{cPerM}$), 
     it suffices to show that $\blamei \ge \blamei'$ 
     implies $\widehat{\blamei} \geq \widehat{\blamei}' - 2\epsilon$.
    Let $\blamei \ge \blamei'$. We have
    \begin{align*}\label{ineq.r5}
        & \norm{\widehat{\blame} - \blame}_1 \leq \epsilon
        \Rightarrow \sum_{i = 1}^n |\widehat{\blamei} - \blamei| \leq \epsilon
        \Rightarrow |\widehat{\blamei} - \blamei| \leq \epsilon \Rightarrow\\
        \Rightarrow& \blamei - \widehat{\blamei} \leq \epsilon \tag{r5}
    \end{align*}
    and
    \begin{align*}\label{ineq.r6}
        & \norm{\widehat{\blame}' - \blame'}_1 \leq \epsilon
        \Rightarrow \sum_{i = 1}^n |\widehat{\blamei}' - \blamei'| \leq \epsilon
        \Rightarrow |\widehat{\blamei}' - \blamei'| \leq \epsilon \Rightarrow\\
        \Rightarrow& \widehat{\blamei}' - \blamei' \leq \epsilon \tag{r6}.
    \end{align*}
    By adding \eqref{ineq.r5} and \eqref{ineq.r6}, we obtain
    \begin{align*}
        &\blamei - \widehat{\blamei} + \widehat{\blamei}' - \blamei' \leq 2\epsilon
        \Rightarrow \widehat{\blamei} \geq \widehat{\blamei}' - 2\epsilon,
    \end{align*}
    and hence $\widehat \blamefunc$ satisfies $2\epsilon$-$\prop_{PerM}$ (resp. $2\epsilon$-$\prop_{cPerM}$).
    
    \item \underline{$\prop_{RcParM}$ (relative c-participation monotonicity):} Let $\blame^1 = \blamefunc(M^1, \bpi^1)$, $\blame^2 = \blamefunc(M^2, \bpi^2)$, $\widehat \blame^1 = \widehat \blamefunc(M^1, \mathcal P(\bpi^1))$ and $\widehat \blame^2 = \widehat \blamefunc(M^2, \mathcal P(\bpi^2))$. 
    To show that $\blamefunc$ satisfying $\prop_{RcParM}$ implies that $\widehat \blamefunc$ satisfies $2\epsilon$-$\prop_{RcParM}$,
     it suffices to show that ${\blamej}^1 - {\blamej}^2 \ge {\blamek}^1 - {\blamek}^2$ 
     implies $\widehat{\blamej}^1 - \widehat{\blamej}^2 \geq \widehat{\blamek}^1 - \widehat{\blamek}^2 - 2\epsilon$. Let ${\blamej}^1 - {\blamej}^2 \ge {\blamek}^1 - {\blamek}^2$. We have
    \begin{align*}\label{ineq.r7}
        & \norm{\widehat{\blame}^1 - \blame^1}_1 \leq \epsilon
        \Rightarrow \sum_{i = 1}^n |\widehat{\blamei}^1 - \blamei^1| \leq \epsilon
        \Rightarrow |\widehat{\blamej}^1 - \blamej^1| + |\widehat{\blamek}^1 - \blamek^1| \leq \epsilon \Rightarrow\\
        \Rightarrow& \blamej^1 - \widehat{\blamej}^1 - \blamek^1 + \widehat{\blamek}^1 \leq \epsilon \tag{r7}
    \end{align*}
    and
    \begin{align*}\label{ineq.r8}
        & \norm{\widehat{\blame}^2 - \blame^2}_1 \leq \epsilon
        \Rightarrow \sum_{i = 1}^n |\widehat{\blamei}^2 - \blamei^2| \leq \epsilon
        \Rightarrow |\widehat{\blamej}^2 - \blamej^2| + |\widehat{\blamek}^2 - \blamek^2| \leq \epsilon \Rightarrow\\
        \Rightarrow& - \blamej^2 + \widehat{\blamej}^2 + \blamek^2 - \widehat{\blamek}^2 \leq \epsilon \tag{r8}.
    \end{align*}
    By adding \eqref{ineq.r7} and \eqref{ineq.r8}, we obtain
    \begin{align*}
        & \blamej^1 - \widehat{\blamej}^1 - \blamek^1 + \widehat{\blamek}^1 - \blamej^2 + \widehat{\blamej}^2 + \blamek^2 - \widehat{\blamek}^2 \leq 2\epsilon\Rightarrow\\
        \Rightarrow& \widehat{\blamej}^1 - \widehat{\blamej}^2 \geq \widehat{\blamek}^1 - \widehat{\blamek}^2 - 2\epsilon,
    \end{align*}
    and hence $\widehat \blamefunc$ satisfies $2\epsilon$-$\prop_{RcParM}$.
\end{itemize}
\end{proof}

%% file: main.bbl
\begin{thebibliography}{10}

\bibitem{propublica_story}
Julia Angwin, Jeff Larson, Surya Mattu, and Lauren Kirchner.
\newblock Machine bias: {There's} software used across the country to predict
  future criminals. and it's biased against blacks.
\newblock
  \url{https://www.propublica.org/article/machine-bias-risk-assessments-in-criminal-sentencing},
  2016.

\bibitem{khandani2010consumer}
Amir~E. Khandani, Adlar~J. Kim, and Andrew~W. Lo.
\newblock Consumer credit-risk models via machine-learning algorithms.
\newblock {\em Journal of Banking \& Finance}, 34(11):2767--2787, 2010.

\bibitem{esteva2017dermatologist}
Andre Esteva, Brett Kuprel, Roberto~A. Novoa, Justin Ko, Susan~M. Swetter,
  Helen~M. Blau, and Sebastian Thrun.
\newblock Dermatologist-level classification of skin cancer with deep neural
  networks.
\newblock {\em Nature}, 542(7639):115--118, 2017.

\bibitem{doshi2017accountability}
Finale Doshi-Velez, Mason Kortz, Ryan Budish, Chris Bavitz, Sam Gershman, David
  O'Brien, Kate Scott, Stuart Schieber, James Waldo, David Weinberger, et~al.
\newblock Accountability of {AI} under the law: The role of explanation.
\newblock {\em arXiv preprint arXiv:1711.01134}, 2017.

\bibitem{kroll2016accountable}
Joshua~A. Kroll, Solon Barocas, Edward~W. Felten, Joel~R. Reidenberg, David~G.
  Robinson, and Harlan Yu.
\newblock Accountable algorithms.
\newblock {\em University of Pennsylvania Law Review}, 165:633, 2016.

\bibitem{wieringa2020account}
Maranke Wieringa.
\newblock What to account for when accounting for algorithms: {A} systematic
  literature review on algorithmic accountability.
\newblock In {\em Conference on Fairness, Accountability, and Transparency},
  pages 1--18, 2020.

\bibitem{ethicsEU}
{European Commission}.
\newblock {Ethics Guidelines for Trustworthy Artificial Intelligence}.
\newblock URL:
  \url{https://ec.europa.eu/digital-single-market/en/news/ethics-guidelines-trustworthy-ai},
  2019.
\newblock [Online; accessed 15-January-2021].

\bibitem{bovens2007analysing}
Mark Bovens.
\newblock Analysing and assessing accountability: {A} conceptual framework.
\newblock {\em European Law Journal}, 13(4):447--468, 2007.

\bibitem{ribeiro2016model}
Marco~Tulio Ribeiro, Sameer Singh, and Carlos Guestrin.
\newblock Model-agnostic interpretability of machine learning.
\newblock {\em arXiv preprint arXiv:1606.05386}, 2016.

\bibitem{datta2016algorithmic}
Anupam Datta, Shayak Sen, and Yair Zick.
\newblock Algorithmic transparency via quantitative input influence: {Theory}
  and experiments with learning systems.
\newblock In {\em IEEE Symposium on Security and Privacy}, pages 598--617,
  2016.

\bibitem{lundberg2017unified}
Scott~M. Lundberg and Su-In Lee.
\newblock A unified approach to interpreting model predictions.
\newblock In {\em International Conference on Neural Information Processing
  Systems}, pages 4768--4777, 2017.

\bibitem{rudin2019stop}
Cynthia Rudin.
\newblock Stop explaining black box machine learning models for high stakes
  decisions and use interpretable models instead.
\newblock {\em Nature Machine Intelligence}, 1(5):206--215, 2019.

\bibitem{doshi2017towards}
Finale Doshi-Velez and Been Kim.
\newblock Towards a rigorous science of interpretable machine learning.
\newblock {\em arXiv preprint arXiv:1702.08608}, 2017.

\bibitem{dann2019policy}
Christoph Dann, Lihong Li, Wei Wei, and Emma Brunskill.
\newblock Policy certificates: {Towards} accountable reinforcement learning.
\newblock In {\em International Conference on Machine Learning}, pages
  1507--1516, 2019.

\bibitem{chockler2004responsibility}
Hana Chockler and Joseph~Y. Halpern.
\newblock Responsibility and blame: {A} structural-model approach.
\newblock {\em Journal of Artificial Intelligence Research}, 22:93--115, 2004.

\bibitem{halpern2016actual}
Joseph~Y. Halpern.
\newblock {\em Actual causality}.
\newblock MiT Press, 2016.

\bibitem{halpern2018towards}
Joseph~Y. Halpern and Max Kleiman-Weiner.
\newblock Towards formal definitions of blameworthiness, intention, and moral
  responsibility.
\newblock In {\em AAAI Conference on Artificial Intelligence}, page
  1853–1860, 2018.

\bibitem{friedenberg2019blameworthiness}
Meir Friedenberg and Joseph~Y. Halpern.
\newblock Blameworthiness in multi-agent settings.
\newblock In {\em AAAI Conference on Artificial Intelligence}, pages 525--532,
  2019.

\bibitem{boutilier1996planning}
Craig Boutilier.
\newblock Planning, learning and coordination in multiagent decision processes.
\newblock In {\em Conference on Theoretical Aspects of Rationality and
  Knowledge}, pages 195--210, 1996.

\bibitem{von2007theory}
John Von~Neumann and Oskar Morgenstern.
\newblock {\em Theory of games and economic behavior (commemorative edition)}.
\newblock Princeton University Press, 2007.

\bibitem{jain2007cost}
Kamal Jain and Mohammad Mahdian.
\newblock Cost sharing.
\newblock {\em Algorithmic Game Theory}, 15:385--410, 2007.

\bibitem{balcan2015learning}
Maria-Florina Balcan, Ariel~D. Procaccia, and Yair Zick.
\newblock Learning cooperative games.
\newblock In {\em International Conference on Artificial Intelligence}, pages
  475--481, 2015.

\bibitem{balkanski2017statistical}
Eric Balkanski, Umar Syed, and Sergei Vassilvitskii.
\newblock Statistical cost sharing.
\newblock In {\em International Conference on Neural Information Processing
  Systems}, pages 6222--6231, 2017.

\bibitem{jia2019towards}
Ruoxi Jia, David Dao, Boxin Wang, Frances~Ann Hubis, Nick Hynes, Nezihe~Merve
  G{\"u}rel, Bo~Li, Ce~Zhang, Dawn Song, and Costas~J. Spanos.
\newblock Towards efficient data valuation based on the {Shapley} value.
\newblock In {\em International Conference on Artificial Intelligence and
  Statistics}, pages 1167--1176, 2019.

\bibitem{agarwal2019marketplace}
Anish Agarwal, Munther Dahleh, and Tuhin Sarkar.
\newblock A marketplace for data: {An} algorithmic solution.
\newblock In {\em ACM Conference on Economics and Computation}, pages 701--726,
  2019.

\bibitem{shoham2008multiagent}
Yoav Shoham and Kevin Leyton-Brown.
\newblock {\em Multiagent systems: Algorithmic, game-theoretic, and logical
  foundations}.
\newblock Cambridge University Press, 2008.

\bibitem{chalkiadakis2004bayesian}
Georgios Chalkiadakis and Craig Boutilier.
\newblock Bayesian reinforcement learning for coalition formation under
  uncertainty.
\newblock In {\em International Joint Conference on Autonomous Agents and
  Multiagent Systems}, pages 1090--1097, 2004.

\bibitem{gillies1959solutions}
Donald~B. Gillies.
\newblock Solutions to general non-zero-sum games.
\newblock {\em Contributions to the Theory of Games}, 4:47--85, 1959.

\bibitem{shapley201617}
Lloyd~S. Shapley.
\newblock {\em 17. A value for n-person games}.
\newblock Princeton University Press, 2016.

\bibitem{shapley1954method}
Lloyd~S. Shapley and Martin Shubik.
\newblock A method for evaluating the distribution of power in a committee
  system.
\newblock {\em The American Political Science Review}, 48(3):787--792, 1954.

\bibitem{banzhaf1964weighted}
John~F. Banzhaf~III.
\newblock Weighted voting doesn't work: {A} mathematical analysis.
\newblock {\em Rutgers Law Review}, 19:317, 1964.

\bibitem{banzhaf1968one}
John~F. Banzhaf~III.
\newblock One man, 3.312 votes: a mathematical analysis of the electoral
  college.
\newblock {\em Villanova Law Review}, 13:304, 1968.

\bibitem{blackstone1893commentaries}
William Blackstone and George Sharswood.
\newblock {\em Commentaries on the Laws of England. {In} Four Books}.
\newblock JB Lippincott, 1893.

\bibitem{scanlon2009moral}
Thomas~M. Scanlon.
\newblock {\em Moral dimensions}.
\newblock Harvard University Press, 2009.

\bibitem{shoemaker2011attributability}
David Shoemaker.
\newblock Attributability, answerability, and accountability: {Toward} a wider
  theory of moral responsibility.
\newblock {\em Ethics}, 121(3):602--632, 2011.

\bibitem{van2015moral}
Ibo Van~de Poel, Lamb{\`e}r Royakkers, and Sjoerd~D. Zwart.
\newblock {\em Moral responsibility and the problem of many hands}.
\newblock Routledge, 2015.

\bibitem{coeckelbergh2020artificial}
Mark Coeckelbergh.
\newblock Artificial intelligence, responsibility attribution, and a relational
  justification of explainability.
\newblock {\em Science and Engineering Ethics}, 26(4):2051--2068, 2020.

\bibitem{torrance2008ethics}
Steve Torrance.
\newblock Ethics and consciousness in artificial agents.
\newblock {\em AI \& Society}, 22(4):495--521, 2008.

\bibitem{asaro2007robots}
Peter~M. Asaro.
\newblock Robots and responsibility from a legal perspective.
\newblock {\em IEEE}, 4(14):20--24, 2007.

\bibitem{lima2021human}
Gabriel Lima, Nina Grgi{\'c}-Hla{\v{c}}a, and Meeyoung Cha.
\newblock Human perceptions on moral responsibility of {AI}: A case study in
  {AI}-assisted bail decision-making.
\newblock In {\em CHI Conference on Human Factors in Computing Systems}, pages
  1--17, 2021.

\bibitem{ijcai2021-244}
Christel Baier, Florian Funke, and Rupak Majumdar.
\newblock A game-theoretic account of responsibility allocation.
\newblock In {\em Proceedings of the Thirtieth International Joint Conference
  on Artificial Intelligence}, pages 1773--1779, 2021.

\bibitem{minsky1961steps}
Marvin Minsky.
\newblock Steps toward artificial intelligence.
\newblock {\em Institute of Radio Engineers}, 49(1):8--30, 1961.

\bibitem{sutton2018reinforcement}
Richard~S. Sutton and Andrew~G. Barto.
\newblock {\em Reinforcement learning: {An} introduction}.
\newblock MIT Press, 2018.

\bibitem{tumer2007distributed}
Kagan Tumer and Adrian Agogino.
\newblock Distributed agent-based air traffic flow management.
\newblock In {\em International Joint Conference on Autonomous Agents and
  Multi-Agent Systems}, pages 1--8, 2007.

\bibitem{foerster2018counterfactual}
Jakob Foerster, Gregory Farquhar, Triantafyllos Afouras, Nantas Nardelli, and
  Shimon Whiteson.
\newblock Counterfactual multi-agent policy gradients.
\newblock In {\em AAAI Conference on Artificial Intelligence}, pages
  2974--2982, 2018.

\bibitem{wang2020shapley}
Jianhong Wang, Yuan Zhang, Tae-Kyun Kim, and Yunjie Gu.
\newblock Shapley {Q-value}: {A} local reward approach to solve global reward
  games.
\newblock In {\em AAAI Conference on Artificial Intelligence}, pages
  7285--7292, 2020.

\bibitem{iyengar2005robust}
Garud~N. Iyengar.
\newblock Robust dynamic programming.
\newblock {\em Mathematics of Operations Research}, 30(2):257--280, 2005.

\bibitem{nilim2005robust}
Arnab Nilim and Laurent El~Ghaoui.
\newblock Robust control of {Markov} decision processes with uncertain
  transition matrices.
\newblock {\em Operations Research}, 53(5):780--798, 2005.

\bibitem{tamar2014scaling}
Aviv Tamar, Shie Mannor, and Huan Xu.
\newblock Scaling up robust {MDPs} using function approximation.
\newblock In {\em International Conference on Machine Learning}, pages
  181--189, 2014.

\bibitem{brandt2016handbook}
Felix Brandt, Vincent Conitzer, Ulle Endriss, J{\'e}r{\^o}me Lang, and Ariel~D.
  Procaccia.
\newblock {\em Handbook of computational social choice}.
\newblock Cambridge University Press, 2016.

\bibitem{young1985monotonic}
Peyton~H. Young.
\newblock Monotonic solutions of cooperative games.
\newblock {\em International Journal of Game Theory}, 14(2):65--72, 1985.

\bibitem{malawski2002equal}
Marcin Malawski.
\newblock Equal treatment, symmetry and {Banzhaf} value axiomatizations.
\newblock {\em International Journal of Game Theory}, 31(1):47--67, 2002.

\bibitem{voloshin2019empirical}
Cameron Voloshin, Hoang~M. Le, Nan Jiang, and Yisong Yue.
\newblock Empirical study of off-policy policy evaluation for reinforcement
  learning.
\newblock {\em arXiv preprint arXiv:1911.06854}, 2019.

\bibitem{datta2015program}
Anupam Datta, Deepak Garg, Dilsun Kaynar, Divya Sharma, and Arunesh Sinha.
\newblock Program actions as actual causes: A building block for
  accountability.
\newblock In {\em IEEE Computer Security Foundations Symposium}, pages
  261--275, 2015.

\bibitem{blackburn2005oxford}
Simon Blackburn.
\newblock {\em The Oxford dictionary of philosophy}.
\newblock Oxford University Press, 2005.

\bibitem{scheffler1988consequentialism}
Samuel Scheffler et~al.
\newblock {\em Consequentialism and its Critics}.
\newblock Oxford University Press, 1988.

\bibitem{kant2020groundwork}
Immanuel Kant.
\newblock {\em Groundwork of the Metaphysic of Morals}.
\newblock Routledge, 2020.

\bibitem{anscombe1958modern}
Gertrude Elizabeth~Margaret Anscombe.
\newblock Modern moral philosophy.
\newblock {\em Philosophy}, 33(124):1--19, 1958.

\bibitem{moore2010placing}
Michael~S. Moore.
\newblock {\em Placing blame: {A} theory of the criminal law}.
\newblock Oxford University Press, 2010.

\bibitem{pinter2015young}
Mikl{\'o}s Pint{\'e}r.
\newblock Young’s axiomatization of the {Shapley} value: a new proof.
\newblock {\em Annals of Operations Research}, 235(1):665--673, 2015.

\end{thebibliography}
